\documentclass[twoside,11pt]{article} 

\usepackage{jmlr2e} 
\usepackage{comment,enumitem,bm}

\usepackage{kz_style}

\newcommand{\kzedit}[1]{{\color{black}#1}}

\usepackage{lastpage}
\jmlrheading{24}{2023}{1-\pageref{LastPage}}{10/20; Revised
7/23}{8/23}{20-1131}{Kaiqing Zhang, Sham M. Kakade, Tamer Ba\c{s}ar, and Lin F. Yang. Part of the work has been accepted to NeurIPS 2020 as a Spotlight Presentation.}

\ShortHeadings{Model-Based MARL in Zero-Sum Markov Games with Near-Optimal  Sample Complexity}{Zhang, Kakade, Ba\c{s}ar, and Yang}
 
\firstpageno{1}

\begin{document}

\title{Model-Based Multi-Agent RL in Zero-Sum Markov Games with Near-Optimal Sample Complexity}

\author{\name Kaiqing Zhang
\email kaiqing@umd.edu \\
     \addr  University of Maryland, College Park\\
       College Park, MD 20740, USA
       \AND
       \name Sham M. Kakade \email sham@seas.harvard.edu \\
      \addr Harvard University\\
	Cambridge, MA 02138, USA
       \AND
       \name Tamer Ba\c{s}ar \email basar1@illinois.edu \\
        \addr University of Illinois at Urbana-Champaign\\
       Urbana, IL 61801, USA
       \AND
       \name Lin F. Yang \email linyang@ee.ucla.edu \\
        \addr University of California, Los Angeles\\
       Los Angeles, CA 90095, USA
       }

\editor{Joelle Pineau}

\maketitle

\begin{abstract}
	Model-based reinforcement learning (RL), which finds an optimal policy after establishing an empirical model, has long been recognized  as one of the cornerstones of RL. 
	It is especially suitable for multi-agent RL (MARL), as it naturally decouples the \emph{learning} and the \emph{planning} phases, and avoids the 
	\emph{non-stationarity} problem when all agents are improving their policies simultaneously. Though intuitive and widely-used, the sample complexity of model-based MARL algorithms has not been fully investigated. In this paper, we aim   to address the fundamental question about its sample complexity. 
	We study arguably the most basic MARL setting: two-player discounted zero-sum  Markov games, given only access to a generative model. We  show that 
	model-based MARL achieves a sample complexity of $\tilde \cO(|\cS||\cA||\cB|(1-\gamma)^{-3}\epsilon^{-2})$ for finding the Nash equilibrium (NE) \emph{value} up to some $\epsilon$ error, and the $\epsilon$-NE \emph{policies} with a smooth planning oracle, where $\gamma$ is the discount factor, and $\cS,\cA,\cB$ denote the state space, and the action spaces for the two agents. 
	We further show that such a sample bound is \emph{minimax-optimal} (up to  logarithmic  factors) if the algorithm is \emph{reward-agnostic}, where the algorithm queries state transition samples without reward knowledge, by establishing a matching lower bound. This is in contrast to the usual \emph{reward-aware} setting, where  the sample complexity lower bound is $\tilde\Omega(|\cS|(|\cA|+|\cB|)(1-\gamma)^{-3}\epsilon^{-2})$, and this model-based approach is \emph{near-optimal} with only a gap on the $|\cA|,|\cB|$ dependence. 
	Our results  not only illustrate the sample-efficiency of this basic  model-based MARL approach, but also elaborate on the fundamental tradeoff between its power (easily handling the reward-agnostic case) and limitation (less adaptive and suboptimal in $|\cA|,|\cB|$), which particularly arises in the multi-agent context.     
\end{abstract}  


\section{Introduction}\label{sec:intro}

Recent years have witnessed numerous successes of reinforcement learning (RL) in many applications, e.g., playing strategy games \citep{OpenAI_dota,alphastarblog}, playing the game of Go \citep{silver2016mastering, silver2017mastering}, autonomous driving \citep{shalev2016safe}, and  security \citep{nguyen2019deep,zhang2019non}. Most of these successful applications involve more than one decision-maker, giving birth to the surging interests and efforts in studying multi-agent RL (MARL) recently, especially on the theoretical side \citep{wei2017online,zhang2018fully,sidford2019solving,zhang2019policy,xie2020learning,shah2020reinforcement,bai2020provable,bai2020near}. See also comprehensive surveys on MARL in \cite{busoniu2008comprehensive,zhang2019multi,nguyen2020deep}.

In general MARL, all agents affect both the state  transition and the rewards of each other,  while each agent may possess different, sometimes even totally conflicting objectives. Without knowledge of  the model, the agents have to resort to data to either estimate the model, improve their own policy, and/or infer other agents' policies. 
One fundamental challenge in MARL is the  emergence of \emph{non-stationarity} during the learning  process \citep{busoniu2008comprehensive,zhang2019multi}: when multiple agents improve their policies concurrently and directly using samples, the environment  becomes non-stationary from each agent's perspective. This has  posed great challenge to development of  effective MARL algorithms based on single-agent ones, especially \emph{model-free} ones, as the condition for guaranteeing convergence in the latter fails to hold in MARL.   \kzedit{One tempting remedy for this non-stationarity issue} is the simple while intuitive method --- model-based\footnote{\kzedit{Note that we here follow the convention of {\it model-based} approach in the generative model setting \citep{azar2013minimax,yang2019optimality,li2020breaking}, which separates these two stages explicitly. In general, model-based RL approaches  do not have to separate the two stages, see e.g., Bayesian RL \citep{poupart2006analytic,ghavamzadeh2015bayesian}, and model-based RL in online exploration settings  \citep{azar2017minimax,bai2020provable}.}} MARL: one first estimates an   empirical model  using data, and then finds the optimal, more specifically, equilibrium policies in this empirical model, via planning. Model-based MARL naturally decouples the \emph{learning} and \emph{planning} phases, and can be incorporated with \emph{any} black-box planning algorithm  that is efficient, e.g., value iteration \citep{shapley1953stochastic} and  (generalized) policy iteration \citep{patek1997stochastic,perolat2015approx}. More importantly, after estimating the model, this approach can potentially handle \emph{more than one} MARL tasks with different reward functions but a  common transition model, without re-sampling the data. Being able to handle this \emph{reward-agnostic} case greatly expands the power of such a model-based approach.

Though intuitive and widely-used, rigorous theoretical justifications for these model-based MARL methods are relatively rare. In this work, our goal is to answer the following standing question: how good is the performance of this na\"{i}ve ``plug-in'' method in terms of non-asymptotic sample complexity? To this end, we focus on arguably the most basic MARL setting  since \cite{littman1994markov}: two-player discounted zero-sum Markov games (MGs) with simultaneous-move agents,  given only access to a generative model. This generative model allows agents to sample the MG, and query the next state from the transition process, given any state-action pair as input. The generative model setting has been  a benchmark in RL when studying the sample efficiency of algorithms \citep{kearns1999finite,kakade2003sample,azar2013minimax,sidford2018near,yang2019optimality}. 
Indeed, this model allows for the study of sample-based multi-agent planning over a long horizon, and helps develop better understanding of the statistical properties of the algorithms, decoupled from the exploration complexity.

Motivated by recent minimax optimal complexity results for single-agent model-based RL \citep{yang2019optimality}, we address  the question above with a positive answer: the model-based MARL approach can achieve near-minimax optimal sample complexity --- in terms of dependencies on the size of the state space, the horizon, and the desired accuracy --- for finding both the Nash equilibrium (NE) value and the NE policies. We also provide a separation in the achievable sample complexity, unique to the multi-agent setting, where, with regards to the dependencies on the number of actions, the na\"ive model-based approach is sub-optimal.
A detailed description is provided next.

\paragraph{Contributions.} We establish the sample complexities of model-based MARL  in zero-sum discounted Markov games, when a generative model is available. 
First, observing that the sampling process in this setting is agnostic to the reward function, we distinguish between   two algorithmic frameworks: \emph{reward-aware} and \emph{reward-agnostic} cases, depending on whether the reward is revealed \emph{before} or \emph{after} the sampling. The model-based approach can inherently handle both cases, especially the latter case with multiple reward functions, without re-sampling the data. 
Second, by establishing lower bounds for both cases, we show that there is indeed a separation in sample complexity, which is unique in the multi-agent setting. 
Third, we show that up to some logarithmic  factors, the model-based approach is indeed minimax optimal in all parameters  in the more challenging reward-agnostic case, and has only a gap on the $|\cA|,|\cB|$ (both agents' action space size) dependence in the reward-aware case.
This separation and the (near-)minimax results have not only justified the sample efficiency of this simple approach, but also highlighted both its power (easily handling multiple reward functions known in hindsight) and its limitation (less adaptive and can hardly  achieve optimal complexity with reward knowledge), particularly arising in the multi-agent RL context.  These results are first-of-their-kind in model-based MARL, and among the first (near-)minimax results in general MARL,  to the best of our knowledge. We also believe that this separation may shed some light on the choice of model-free and model-based approaches in various MARL scenarios in practice, and provide new understandings for algorithm-design in  other MARL settings, e.g., with no generative model, and going beyond two-player zero-sum MGs.

\paragraph{Related Work.}
Stemming from the formative work \cite{littman1994markov}, MARL has been mostly studied under the framework of   Markov games \citep{shapley1953stochastic}. There has been no     shortage of provably convergent MARL algorithms ever since then \citep{littman2001friend,hu2003nash,greenwald2003correlated}.  However, most of these early results are Q-learning-based (thus model-free) and asymptotic, with no sample complexity guarantees. To establish \emph{non-asymptotic} results, \cite{perolat2015approx,perolat2016use,perolat2016learning,yang2019theoretical,zhang2018finite} have studied the sample complexity of \emph{batch} model-free MARL methods. There are also increasing interests in policy-based (thus also model-free) methods for solving special MGs with non-asymptotic convergence guarantees \citep{perolat2018actor,srinivasan2018actor,zhang2019policy}.  No result on the (near-)minimax optimality of these  complexities has been established prior to the present work. 

Specific to the two-player zero-sum setting, \cite{jia2019feature} and \cite{sidford2019solving} have considered \emph{turn-based  MGs}, a special case of the simultaneous-move MGs considered  here, with a generative model. Specifically,  \cite{sidford2019solving} established near-optimal sample complexity of $\tilde \cO((1-\gamma)^{-3}\epsilon^{-2})$  for a variant of Q-learning for this setting. More recently, \cite{bai2020provable,xie2020learning} have established both regret and sample complexity guarantees for episodic zero-sum MGs, without a generative model, with focus on efficient exploration.    The work in \cite{shah2020reinforcement} also focused on the turn-based setting, and combined  Monte-Carlo Tree Search and supervised learning to find the NE values. 
In contrast, model-based MARL theory has relatively limited literature.      \cite{brafman2002r} proposed the R-MAX  algorithm for average-reward MGs, with polynomial sample complexity. \cite{wei2017online}  developed a model-based upper confidence algorithm with polynomial sample complexities for the same setting.    
These methods differ from ours, as they are either specific model-free approaches, or  not clear yet if they are (near-)minimax optimal in the corresponding setups.  
Concurrent to our work, \cite{bai2020near} developed \emph{model-free} algorithms with near-optimal sample complexities in episodic settings without a generative model. The results are optimal in $|\cS|,|\cA|,|\cB|$ dependence, but not in the horizon $H$.  \kzedit{Finally, we note that MARL in Markov games  is not restricted to the competitive setting of two-player zero-sum, and the studies in (multi-player)  cooperative/potential settings  \citep{leonardos2021global,zhang2021gradient,ding2022independent,sayin2022fictitious} and general-sum settings  \citep{hu2003nash,liu2021sharp,jin2021v,mao2022improving,mao2023provably}  also exist, and is not the focus of the present paper.}

In the single-agent regime, there has been extensive literature on  non-asymptotic efficiency of RL in MDPs;  see \cite{kearns1999finite,kakade2003sample,strehl2009reinforcement,jaksch2010near,azar2013minimax,osband2014model,dann2015sample,azar2017minimax,wang2017randomized,sidford2018near,jin2018q,li2020breaking}. Amongst them, we highlight the minimax optimal ones:  \cite{azar2013minimax} and \cite{azar2017minimax} have provided minimax optimal results for sample complexity and regret in the settings with and without a generative model, respectively. Specifically,  \cite{azar2013minimax} has shown  that to achieve the $\epsilon$-optimal  \emph{value} in Markov decision processes (MDPs),  at least $\tilde\Omega(|\cS||\cA|(1-\gamma)^{-3}\epsilon^{-2})$ samples are needed, for $\epsilon\in(0,1]$. They also showed that to find an $\epsilon$-optimal \emph{policy}, the same minimax complexity order in $1-\gamma$ and $\epsilon$ can be attained, if $\epsilon\in (0,(1-\gamma)^{-1/2}|\cS|^{-1/2}]$ and the total sample complexity is $\tilde\cO(|\cS|^2|\cA|)$, which is in fact \emph{linear} in the model size. Later, \cite{sidford2018near} has proposed a Q-learning based approach to attain this lower bound and remove the extra dependence on $|\cS|$, for $\epsilon\in(0,1]$. More recently,   \cite{yang2019optimality} developed new techniques based on \emph{absorbing MDPs}, to show that model-based RL also   achieves the lower bound for finding an $\epsilon$-optimal \emph{policy}, with a larger $\epsilon$ range of $(0,(1-\gamma)^{-1/2}]$\footnote{While preparing the present work, \cite{li2020breaking} has further improved the minimax optimal results in \cite{yang2019optimality}, in that they cover the entire range of sample sizes. We believe the improvement can also be incorporated in the MARL setting here, which is left as our future work.}. Finally, our separation of the reward-agnostic case is motivated by the recent novel framework of reward-free RL in \cite{jin2020reward}.

\section{Preliminaries}\label{sec:prelim}

\paragraph{Zero-Sum  Markov Games.}

Consider a zero-sum MG\footnote{We will hereafter  refer to this model simply  as a \emph{MG}.} $\cG$ characterized by $(\cS,\cA,\cB,P,r,\gamma)$, where $\cS$ is the state space; $\cA,\cB$ are the action spaces of agents $1$ and $2$, respectively; $P:\cS\times\cA\times\cB\to \Delta(\cS)$  denotes the transition probability of states; $r: \cS\times\cA\times \cB\to[0,1]$ denotes the reward function\footnote{Our results can be generalized to other ranges of reward function by a 
 standard reduction, see e.g., \cite{sidford2018near}, and randomized reward functions.} of agent $1$ (thus $-r$ is the bounded reward function of agent $2$); and $\gamma\in[0,1)$ is the discount  factor. The goal of agent $1$ (agent $2$) is to maximize (minimize) the long-term accumulative discounted reward. In MARL, the agents aim to achieve this goal using data samples collected from the model.   

At each time $t$, agent $1$ (agent $2$) has a stationary (not necessarily deterministic) policy $\mu:\cS\to\Delta(\cA)$ ($\nu:\cS\to\Delta(\cB)$), where $\Delta(\cX)$ denotes the space of all probability measures over $\cX$,  
so that $a_t\sim\mu(\cdot\given s_t)$ ($b_t\sim\nu(\cdot\given s_t)$). The state makes a transition   from $s_t$ to $s_{t+1}$ following the probability distribution  $P(\cdot\given s_t,a_t,b_t)$, given 
$(a_t,b_t)$. 
As in the MDP model, one can define the  \emph{state-value function} under a pair of joint policies $(\mu,\nu)$ as 
\$
V^{\mu,\nu}(s):=\EE_{a_t\sim \mu(\cdot\given s_t),b_t\sim\nu(\cdot\given s_t)}\bigg[\sum_{t\geq 0}\gamma^tr(s_t,a_t,b_t)\bigggiven s_0=s\bigg]. 
\$
Note that $V^{\mu,\nu}(s)\in[0,1/(1-\gamma)]$ for any $s\in\cS$ as $r\in[0,1]$, and the expectation is taken over the random trajectory produced by the joint policy $(\mu,\nu)$. 
Also, the \emph{state-action/Q-value function} 
under $(\mu,\nu)$ is  defined by
\small  
\$
Q^{\mu,\nu}(s,a,b)&:=\EE_{a_t\sim \mu(\cdot\given s_t),b_t\sim\nu(\cdot\given s_t)}\bigg[\sum_{t\geq 0}\gamma^tr(s_t,a_t,b_t)\bigggiven s_0=s,a_0=a,b_0=b\bigg]. 
\$
\normalsize
The solution concept  considered  is the (approximate) \emph{Nash equilibrium}, as defined below. 

\begin{definition}[($\epsilon$-)Nash Equilibrium]\label{def:NE}
	For a zero-sum MG $(\cS,\cA,\cB,P,r,\gamma)$, a \emph{Nash equilibrium policy} pair $(\mu^*,\nu^*)$ satisfies the following pair of inequalities\footnote{In game theory, this pair is commonly referred to as \emph{saddle-point inequalities}.}  for any $s\in\cS$, $\mu\in\Delta(\cA)^{|\cS|}$, and $\nu\in\Delta(\cB)^{|\cS|}$ 
	\#\label{equ:def_NE}
	V^{\mu,\nu^*}(s)\leq V^{\mu^*,\nu^*}(s)\leq V^{\mu^*,\nu}(s). 
	\#
	If \eqref{equ:def_NE} holds with some $\epsilon>0$ relaxation, i.e., for some policy $(\mu',\nu')$, such that 
	\#\label{equ:def_approx_NE}
	V^{\mu,\nu'}(s)-\epsilon\leq V^{\mu',\nu'}(s)\leq V^{\mu',\nu}(s)+\epsilon,
	\#
	then $(\mu',\nu')$ is an \emph{$\epsilon$-Nash equilibrium policy}  pair. 
\end{definition}

By \cite{shapley1953stochastic,patek1997stochastic},  there  exists a Nash equilibrium policy pair $(\mu^*,\nu^*)\in\Delta(\cA)^{|\cS|}\times \Delta(\cB)^{|\cS|}$ for two-player  discounted zero-sum MGs. The state-value $V^*:=V^{\mu^*,\nu^*}$ is referred to as the \emph{value of the game}. The corresponding Q-value function is denoted by $Q^*$. 
The objective of the two agents is to find the NE policy of the MG, namely, to solve the saddle-point problem
\#\label{equ:def_saddle_point} 
\max_{\mu}\min_{\nu}\quad V^{\mu,\nu}(s),
\#
for every $s\in\cS$,  where the order of $\max$ and $\min$ can be interchanged \citep{von1953theory,shapley1953stochastic}. 
For notational convenience, for any policy $(\mu,\nu)$, we define
\#\label{equ:def_best_response}
V^{\mu,*}=\min_{\nu}V^{\mu,\nu},\qquad\qquad\qquad V^{*,\nu}=\max_{\mu}V^{\mu,\nu},
\#
and denote the corresponding optimizers by $\nu(\mu)$ and $\mu(\nu)$, respectively. We refer to these values and optimizers as the \emph{best-response} values and policies, given $\mu$ and $\nu$, respectively.

\paragraph{Reward-Aware v.s. Reward-Agnostic.}
We first differentiate between two algorithmic mechanisms in the generative-model setting. 
In the reward-aware case, the reward function is either known to the agents  \citep{azar2013minimax,sidford2018near,yang2019optimality,sidford2019solving,li2020breaking}, or can at least be estimated from data. The reward knowledge  can thus be used to potentially guide the sampling process, making the algorithm  
 \emph{adaptive}. In the reward-agnostic case, reward knowledge is not used to guide sampling, and is possibly only revealed  after the sampling. 
This especially  fits in the scenario when there is more than one  reward function of interest, or when the reward function is engineered  iteratively, since it can now handle a class of reward functions that are not pre-specified, without re-sampling the data for each of them. {Existing works in single-agent settings have no such a separation \citep{azar2013minimax,sidford2018near, agarwal2019optimality,li2020breaking}, as the sample complexity of estimating the reward function is typically of lower order, and the reward function is thus assumed to be known. In particular, the model-based approaches in \cite{azar2013minimax, agarwal2019optimality,li2020breaking} are reward-agnostic, while the model-free approaches in \cite{sidford2018near,sidford2019solving} are reward-aware. Interestingly, in two-agent Markov games, whether the reward is known beforehand or not may lead to different sample complexity lower-bounds, as we will see  in \S\ref{sec:lower_bound}. We thus point out this separation here for clarity. }

\begin{remark}[Reward-Agnostic \& Reward-Free]
	The reward-agnostic case we advocate here is closely related to the recent novel algorithmic framework of \emph{reward-free} RL   \citep{jin2020reward}, where there are also two phases, \emph{exploration} and \emph{planning}, while trajectories are only collected in the exploration phase, without any reward knowledge, and various reward functions are fed to the algorithm for evaluation in the planning phase. One key difference is that the reward-free setting aims to be effective for \emph{all} reward function in the planning phase \emph{simultaneously}, while the reward-agnostic setting only focuses on handling the underlying single-reward (or a few, e.g., polynomial number of,   reward functions) that is not pre-specified.  
	Being less general than the pure reward-free setting, the sample complexity bounds are thus possibly better, as we will show   in \S\ref{sec:res}. 
\end{remark}




\paragraph{Model-Based Approach with Generative Model.}
As a standard setting, suppose that we have access to a \emph{generative model/sampler}, which can provide us with samples $s'\sim P(\cdot\given s,a,b)$ for any $(s,a,b)$. The model-based MARL algorithm simply calls the sampler $N$ times at each state-joint-action pair $(s,a,b)$, and constructs an empirical estimate of the transition model $P$, denoted by  $\hat{P}$, following
\$
\hat{P}(s'\given s,a,b)=\frac{\text{count}(s',s,a,b)}{N}. 
\$  
Here $\text{count}(s',s,a,b)$ is the number of times the state-action pair $(s,a,b)$ forces a transition  to state $s'$. 
Note that the reward function is not estimated, in either the reward-aware or reward-agnostic cases, as for the former, the sample complexity of estimating $r$ is only a lower order term, and $r$ is thus typically   assumed to be known \citep{azar2013minimax,sidford2018near,yang2019optimality,li2020breaking}; while for the latter, no reward information is even available in the sampling processes. This model-based approach via estimating $P$ inherently handles both cases. Such a model-estimation can be implemented by both agents independently. 

\paragraph{Planning Oracle.}

The reward function, together with the empirical transition model $\hat{P}$ and the components $(\cS,\cA,\cB,\gamma)$ in the true model $\cG$, constitutes an empirical game model $\hat{\cG}$. As in \cite{azar2013minimax,yang2019optimality,jin2020reward,li2020breaking} for single-agent RL, we assume  that an efficient planning oracle is available, which takes $\hat{\cG}$ as input, and outputs a policy pair $(\hat{\mu},\hat{\nu})$. This oracle decouples the statistical and computational aspects of the empirical  model $\hat{\cG}$. The output policy pair, referred to as being \emph{near-equilibrium}, is assumed to satisfy certain $\epsilon_{opt}$-order of equilibrium, in terms of value functions,    
and we  evaluate the performance of $(\hat{\mu},\hat{\nu})$ on the original MG $\cG$.  
Common planning algorithms include value iteration \citep{shapley1953stochastic} and  (generalized) policy iteration \citep{patek1997stochastic,perolat2015approx}, which are efficient in finding the ($\epsilon$-)NE of $\hat{\cG}$. In addition, it is not hard to have an oracle that is smooth in generating policies, i.e., the change of the approximate NE policies can be bounded by the changes of the NE value. See our Definition \ref{def:smooth_oracle} later for a formal statement. \kzedit{Finally, we  note that our definition of {\it model-based}  approach in the generative-model-setting 
follows from that in   \citep{azar2013minimax,yang2019optimality,li2020breaking}, which separates these two stages explicitly. In general, model-based RL approaches  do not have to separate the two stages, see e.g., Bayesian RL \citep{ghavamzadeh2015bayesian}, and model-based RL in online exploration settings \citep{azar2017minimax,bai2020provable}.}

\section{Main Results}\label{sec:res}

We now introduce the main results of this paper. For notational convenience, we use $\hat{V}^{\mu,\nu}$, $\hat{V}^{\mu,*}$, $\hat{V}^{*,\nu}$, and $\hat{V}^{*}$ to denote the value under $(\mu,\nu)$, the best-response value under $\mu$ and $\nu$, and the NE value, under the empirical game model $\hat{\cG}$, respectively. A similar convention is also used for Q-functions.

\subsection{Lower Bounds}\label{sec:lower_bound}

We first establish lower bounds on both approximating the NE value function and learning the $\epsilon$-NE policy pair, in both reward-aware and reward-agnostic cases.

\begin{lemma}[Lower Bound for Reward-Aware Case]\label{thm:lb_informal}
	Let $\cG$ be an unknown zero-sum MG, and the agents learn in a reward-aware  case, i.e., the reward knowledge is available during sampling. Then,  there exist $\epsilon_0,\delta_0>0$,  such that for all $\epsilon\in(0,\epsilon_0)$, $\delta\in(0, \delta_0)$, the sample complexity of learning an $\epsilon$-NE policy pair, or an $\epsilon$-approximate NE value, i.e., finding a $\hat{Q}$ such that $\|\hat{Q}-Q^*\|_\infty\leq \epsilon$ for $\cG$,  with a generative model with probability at least $1-\delta$, is $\tilde\Omega\big({|\cS|(|\cA|+|\cB|)}{(1-\gamma)^{-3}\epsilon^{-2}}\log({1}/{\delta})\big).$  
\end{lemma}

The proof of Lemma  \ref{thm:lb_informal}, via a straightforward adaptation from the lower bounds for MDPs \citep{azar2013minimax,feng2019does}, is provided in \S\ref{sec:append_reward_aware_proof}.  \kzedit{In particular, one can design a two-player zero-sum Markov game such that one of the players is dummy -- she has no control on the reward nor the transition dynamics. Then, the existing lower bound in \cite{azar2013minimax,feng2019does} for MDPs leads to the desired result.} Note that as in \cite{azar2013minimax,sidford2018near,sidford2019solving,yang2019optimality,li2020breaking}, the reward function is \emph{known} in this case. 
As we will show momentarily, our  sample complexity is tight in $1-\gamma$ and $|\cS|$, while has a gap in $|\cA|,|\cB|$ dependence ($\tilde\cO(|\cA||\cB|)$ versus $\tilde\Omega(|\cA|+|\cB|)$). 
{In \S\ref{sec:append_reward_aware_proof}, we discuss that the $\tilde\Omega(|\cA|+|\cB|)$ lower bound may not be improved in this reward-aware case, and might be attainable by \emph{model-free} algorithms instead (as $\tilde\cO(|\cA||\cB|)$ is inherent in model-based approaches due to estimating $P$). Interestingly, in the concurrent work \cite{bai2020near}, under a different MARL setting, such an $\tilde\Omega(|\cA|+|\cB|)$ complexity is indeed shown to be attainable by a model-free algorithm with \emph{online} updates.}

On the other hand, note that our model-based approach can inherently also handle the more challenging reward-agnostic case. Indeed, estimating the transition model $P$ seems a bit of an overkill for the reward-aware case, in terms of sample complexity. A natural two-part question then becomes: what is the sample complexity lower bound in this more challenging  reward-agnostic case, and can the model-based approach attain it? We formally answer the first part of the question in the following theorem, whose proof is deferred to \S\ref{sec:append_reward_agnostic_proof}, and answer the second part in \S\ref{sec:res_approx_NE_value} and \S\ref{sec:res_approx_NE_policy}.

 \begin{theorem}[Lower Bound for Reward-Agnostic Case]\label{thm:lb_reward_agnostic}
 	Let $\cG$ be an unknown zero-sum MG, and the agents learn in a reward-agnostic case, i.e., they first call  the generative model for sampling, without reward knowledge, and then are fed with the reward function $r$ in $\cG$,  for finding either an $\epsilon$-NE policy pair, or an $\epsilon$-approximate NE value for $\cG$.  
 	Then,  there exist $\epsilon_0,\delta_0>0$,  such that for all $\epsilon\in(0,\epsilon_0)$, $\delta\in(0, \delta_0)$, the sample complexity of achieving either goal  with probability at least $1-\delta$, is 
 	\$
 	\tilde\Omega\bigg(\frac{|\cS||\cA||\cB|}{(1-\gamma)^{3}\epsilon^{2}}\log\Big(\frac{1}{\delta}\Big)\bigg).
 	\$ 
 \end{theorem}
 
 Compared to Lemma \ref{thm:lb_informal}, the dependence on $|\cA|, |\cB|$ is increased   from $\tilde\Omega(|\cA|+|\cB|)$ to $\tilde\Omega(|\cA||\cB|)$. Several remarks are now in order. First, this suggests that without guidance from the reward,  {the reward-agnostic case can be more challenging to tackle.} 
 \kzedit{The intuition is that,  when the reward is only given   in hindsight, which might be chosen adversarially, costs the algorithm to at least sample at all $|\cA||\cB|$  elements in the Q-value $Q(s,\cdot,\cdot)$ at each state $s$ often enough.}   Second, when reduced to the single-agent setting (e.g., with $|\cB|=1$), such a separation disappears,  showing its unique emergence in the multi-agent setting, and explaining  why these two cases were not differentiated explicitly in the single-agent literature. Third,  this lower bound is also related to the reward-free setting \citep{jin2020reward} with a single unknown reward (not infinitely many as in \cite{jin2020reward}). 
 
 \kzedit{
 The basic intuition regarding the separation between the lower bounds in reward-aware and reward-agnostic cases, when compared to the single-agent setting (where there is no such a separation), is the {\it insensitivity} of Nash equilibrium (NE) to the changes in payoff matrices in two-player zero-sum  games \citep{jansen1981regularity}. In particular,  NE in general depends on the {\it joint} behavior and preferences of both agents.  Specifically, to construct the lower bound (even in the single-agent case, see e.g., \cite{azar2013minimax,feng2019does}), we needed to carefully  {\it perturb} the Q-value function at each state-action pair of some {\it null hypothesis} instance, so that the solution (which is the  {\it maximum} in the single-agent case, and {\it Nash equilibrium} in the multi-agent case) is also changed in the perturbed alternative hypothesis cases. Hence, for each alternative hypothesis case, we need to change the NE by only changing $\mathcal{O}(1)$ elements in the payoff matrix, i.e., the Q-value table. In the reward-aware setting,  since the reward is known (or can be estimated accurately with negligible sample complexity), we can {\it only perturb}  the {\it transition matrix} to perturb the Q-value table, which share the same size (i.e., the degree-of-freedom). Due to the insensitivity, we can hardly construct $\Theta(|\mathcal{A}||\mathcal{B}|)$ different hard cases (as needed to construct a $\Omega(|\mathcal{A}||\mathcal{B}|)$ lower bound) while by only perturbing $\mathcal{O}(1)$ elements in the transition matrix in each case. Note that such a perturbation can be effective in the single-agent MDP setting, as by only perturbing one element in the transition matrix, the maximum of the Q-value can be changed, see e.g., \cite{azar2013minimax,feng2019does}. 
		
		In contrast, in the reward-agnostic setting, the reward information is given {\it after} the sampling phase and the estimation of the model. 
		This way, more freedom is allowed to  construct $\Theta(|\mathcal{A}||\mathcal{B}|)$ different hard cases, by {\it adversarially choosing} the reward function afterwards. In particular, the Q-value will be affected by both the transition matrix and the reward, and with polynomial  number of reward functions, we were able to construct Q-value tables in $\Theta(|\mathcal{A}||\mathcal{B}|)$ different hard cases. Note that taking a union bound over the polynomial number of reward   functions do not affect the total sample complexity, as it is still dominated by the sample complexity of estimating the transition matrix. In other words, the freedom of constructing and perturbing the reward functions adversarially  {\it afterwards} forces the algorithm to estimate {\it all} the elements in the transition matrix well, in order to handle the reward-agnostic setting. This has been inherently done by  our model-based approach. We defer more details about the lower bounds comparison in Appendix \ref{sec:append_proof_lb}. 
		}

\subsection{Near-Optimality in Finding $\epsilon$-Approximate NE Value}\label{sec:res_approx_NE_value}

We now establish the near-minimax optimal sample complexities of model-based MARL. 
Note that theses results apply to both reward-aware and reward-agnostic cases, as the implementation of our model-based approach does not rely on the reward function. 
We start by showing the sample complexity to achieve an $\epsilon$-approximate NE value.

\begin{theorem}[Finding $\epsilon$-Approximate NE Value]\label{thm:main_res}
	Suppose that the policy pair $(\hat{\mu},\hat{\nu})$ is obtained from the \textbf{Planning Oracle} using the empirical model $\hat{\cG}$, which satisfies  
	\$
	\|\hat{V}^{\hat{\mu},\hat{\nu}}-\hat{V}^*\|_\infty\leq \epsilon_{opt}. 
	\$
	Then, for any $\delta\in(0,1]$ and {$\epsilon\in(0,1/(1-\gamma)^{1/2}]$}, if 
	\$
	N\geq \frac{c\gamma \log\big[c|\cS||\cA||\cB|(1-\gamma)^{-2}\delta^{-1}\big]}{(1-\gamma)^3\epsilon^2}
	\$
	for some absolute constant $c$, it holds that with probability at least $1-\delta$, 
	\$
	\big\|Q^{\hat{\mu},\hat{\nu}}-Q^*\big\|_\infty\leq \frac{2\epsilon}{3}+\frac{5\gamma\epsilon_{opt}}{1-\gamma},\qquad \qquad \big\|\hat{Q}^{\hat{\mu},\hat{\nu}}-Q^*\big\|_\infty\leq \epsilon+\frac{9\gamma\epsilon_{opt}}{1-\gamma}. 
	\$
\end{theorem}

Theorem \ref{thm:main_res} shows that if the planning error $\epsilon_{opt}$ is made small, e.g., with the order of $\cO((1-\gamma)\epsilon)$, then the Nash equilibrium Q-value can be estimated with a sample complexity of  $\tilde\cO(|\cS||\cA||\cB|(1-\gamma)^{-3}\epsilon^{-2})$, as $N$ queries are made for each $(s,a,b)$ pair.  
This planning error can be achieved by performing any efficient black-box optimization technique over the empirical model $\hat{\cG}$. Examples of such  oracles include value iteration \citep{shapley1953stochastic} and  (generalized) policy iteration \citep{patek1997stochastic,perolat2015approx}. 
Moreover, note that, in contrast to the single-agent  setting, where only  a  $\max$ operator is used, a $\min\max$ (or $\max\min$) operator is used in these algorithms, which involves solving  a \emph{matrix game} at each state. This can be   solved as a linear program \citep{osborne1994course}, with at  best  polynomial runtime complexity \citep{grotschel1981ellipsoid,karmarkar1984new}. This in total leads to an efficient polynomial runtime complexity  algorithm.   

As per Lemma \ref{thm:lb_informal}, our $\tilde\cO(|\cS||\cA||\cB|(1-\gamma)^{-3}\epsilon^{-2})$ complexity is near-minimax optimal for the reward-aware case, in that it is tight in the dependence of $1-\gamma$ and $|\cS|$, and sublinear in the model-size (which is $|\cS|^2|\cA||\cB|$). However, there is a gap on the $|\cA|,~|\cB|$ dependence ($\tilde\cO(|\cA||\cB|)$ versus $\tilde\cO(|\cA|+|\cB|)$). Unfortunately, without further assumption on the MG, e.g., being turn-based, the model-based algorithm can hardly avoid the $\tilde\cO(|\cS||\cA||\cB|)$ dependence, as it is required to estimate each $\hat P(\cdot\given s,a,b)$ accurately to perform the planning. It is only minimax-optimal if the action-space size of one agent dominates the other's (e.g., $|\cA|\gg |\cB|$).

In the reward-agnostic case, as per Theorem \ref{thm:lb_reward_agnostic},  $\tilde\cO(|\cS||\cA||\cB|(1-\gamma)^{-3}\epsilon^{-2})$ is indeed minimax-optimal, and is tight in all $|\cS|, |\cA|, |\cB|$ and $1-\gamma$ dependence. 
More significantly, in this case, more than one reward functions can be handled simultaneously, as long as the transition model is estimated accurately enough. Specifically, with $M$ reward functions, by letting $\delta=\delta/M$ in Theorem \ref{thm:main_res} and using union bounds, the sample complexity of finding $\epsilon$-approximate NE value corresponding to all $M$ reward functions becomes $\tilde\cO(\log(M)|\cS||\cA||\cB|(1-\gamma)^{-3}\epsilon^{-2})$, which, with $M$ being polynomial in $|\cS|,~|\cA|,~|\cB|$, is of the same order as  that in Theorem \ref{thm:main_res}. 

However, this (near-)optimal result does not necessarily lead to  near-optimal sample complexity for obtaining the  $\epsilon$-NE \emph{policies}. We first use a direct translation 
to obtain such  an {$\epsilon$-NE} policy pair based on Theorem \ref{thm:main_res}, for \emph{any} \textbf{Planning Oracle}.

\begin{corollary}[Finding $\epsilon$-NE Policy]\label{coro:main_res_coro}
	Let $(\hat{\mu},\hat{\nu})$ and $N$ satisfy the conditions in Theorem \ref{thm:main_res}. Let 
	\$
	\tilde \epsilon:=\frac{2}{1-\gamma}\cdot\Big(\epsilon+\frac{9\gamma\epsilon_{opt}}{1-\gamma}\Big), 
	\$ 
	and $(\tilde{\mu},\tilde{\nu})$ be the one-step Nash equilibrium of $\hat{Q}^{\hat{\mu},\hat{\nu}}$,  namely, for any $s\in\cS$
	\$
	\big(\tilde{\mu}(\cdot\given s),\tilde{\nu}(\cdot\given s)\big)\in\argmax_{u\in\Delta(\cA)}\min_{\vartheta\in\Delta(\cB)}\EE_{a\sim u,b\sim \vartheta}\big[\hat{Q}^{\hat{\mu},\hat{\nu}}(s,a,b)\big]. 
	\$
	Then, with  probability at least $1-\delta$, 
	\#\label{equ:coro_eps_NE_1}
	V^{*,\tilde{\nu}}-2\tilde \epsilon\leq V^{\tilde{\mu},\tilde{\nu}}\leq V^{\tilde{\mu},*}+2\tilde \epsilon,  
	\#
	namely, $(\tilde{\mu},\tilde{\nu})$ constitutes a $2\tilde \epsilon$-Nash equilibrium policy pair. 
\end{corollary}

Corollary \ref{coro:main_res_coro} is equivalently to saying that the sample complexity of achieving an $\epsilon$-NE policy pair is $\tilde\cO((1-\gamma)^{-5}\epsilon^{-2})$. 
This is worse than the model-based single-agent setting \citep{yang2019optimality}, and also worse than both the  model-free single-agent  \citep{sidford2018near} and turn-based two-agent \citep{sidford2019solving}  settings, where  $\tilde\cO((1-\gamma)^{-3}\epsilon^{-2})$ can be achieved for learning the optimal policy. 
This also has a gap from the lower bound  in both Lemma \ref{thm:lb_informal} and Theorem \ref{thm:lb_reward_agnostic}.  
Note that the above sample complexity still matches 
that of the Empirical QVI in \cite{azar2013minimax} if $\epsilon\in(0,1]$ for single-agent RL,  but with a larger choice of $\epsilon$ of $(0,(1-\gamma)^{-1/2}]$. As the Markov game setting is more challenging than MDPs, it is not clear yet if the lower bounds in Lemma \ref{thm:lb_informal} and Theorem \ref{thm:lb_reward_agnostic} in finding $\epsilon$-NE policies can be achieved, using a general \textbf{Planning Oracle}.  
In contrast, we show next that a stable \textbf{Planning Oracle} can indeed (almost) match the lower bounds. 


\subsection{Near-Optimality in Finding $\epsilon$-NE Policy}\label{sec:res_approx_NE_policy}

Admittedly,  Corollary \ref{coro:main_res_coro} does  not fully exploit the \emph{model-based} approach, since it finds the NE policy according to the Q-value estimate $\hat{Q}^{\hat{\mu},\hat{\nu}}$, instead of using the output policy pair $(\hat{\mu},\hat{\nu})$ directly. This loses a factor of $1-\gamma$. To improve the sample complexity of obtaining the NE {policies}, we first introduce the following definition of a smooth \textbf{Planning Oracle}. 
 

\begin{definition}
	[Smooth \textbf{Planning Oracle}]\label{def:smooth_oracle}
	A smooth \textbf{Planning Oracle} generates  policies that are smooth with respect to the NE Q-values of the empirical model. Specifically, for two empirical models $\hat{\cG}_1$ and $\hat{\cG}_2$, the generated near-equilibrium  policy pair  $(\hat{\mu}_1,\hat{\nu}_1)$ and $(\hat{\mu}_2,\hat{\nu}_2)$ satisfy that for each $s\in\cS$, $\|\hat{\mu}_1(\cdot\given s)-\hat{\mu}_2(\cdot\given s)\|_{TV}\leq C\cdot\|\hat{Q}_1^*-\hat{Q}_2^*\|_\infty$ and $\|\hat{\nu}_1(\cdot\given s)-\hat{\nu}_2(\cdot\given s)\|_{TV}\leq C\cdot\|\hat{Q}_1^*-\hat{Q}_2^*\|_\infty$ for some constant\footnote{We allow $C$ to  depend polynomially on $|\cA|,~|\cB|$, which, as we will show later, does not affect the sample complexity as it appears as $\log C$.} $C>0$, where $\hat{Q}_i^*$ is the NE Q-value of $\hat{\cG}_i$ for $i=1,2$, and $\|\cdot\|_{TV}$ is the total variation distance.   
\end{definition}

%
Such a smooth \textbf{Planning Oracle} can be 
 readily obtained in several ways. For example, one simple (but possibly  computationally expensive) approach is to output the average over the entire policy space, using a \emph{softmax}  randomization over best-response values induced by $\hat{Q}^*$. Specifically, for agent $1$, the output $\hat{\mu}$ is given by 
\$
\hat{\mu}(\cdot\given s)=\int_{\Delta(\cA)}\frac{\exp\big(\min\limits_{\vartheta\in\Delta(\cB)}\EE_{a\sim u,b\sim\vartheta}\big[\hat{Q}^*(s,a,b)\big]\big/\tau\big)}{\int_{\Delta(\cA)}\exp\big(\min\limits_{\vartheta\in\Delta(\cB)}\EE_{a\sim u',b\sim\vartheta}\big[\hat{Q}^*(s,a,b)\big]\big/\tau\big)du'}\cdot udu, 
\$
where $\tau>0$ is some temperature constant. The output of $\hat{\nu}$ is analogous. With a small enough $\tau$, $\hat{\mu}$ approximates the exact solution to $\argmax\limits_{u\in\Delta(\cA)}\min\limits_{\vartheta\in\Delta(\cB)}\EE_{a\sim u,b\sim\vartheta}[\hat{Q}^*(s,a,b)]$, the NE policy given $\hat{Q}^*$.  Moreover, notice that $\hat{\mu}$ satisfies the smoothness condition in Definition  \ref{def:smooth_oracle}. This is because for each $u\in\Delta(\cA)$ in the integral: i) the softmax function is Lipschitz continuous \kzedit{with respect to  $\min\limits_{\vartheta\in\Delta(\cB)}\EE_{a\sim u,b\sim\vartheta}\big[\hat{Q}^*(s,a,b)\big]$ with Lipschitz constant $1/\tau$} \citep{gao2017properties}; ii)  the best-response value $\min\limits_{\vartheta\in\Delta(\cB)}\EE_{a\sim u,b\sim\vartheta}\big[\hat{Q}^*(s,a,b)\big]$ is smooth with respect to $\hat{Q}^*$. Thus, such an oracle   is an instance  of  a smooth \textbf{Planning Oracle}.

Another more tractable way to obtain $(\hat{\mu},\hat{\nu})$ is by directly  solving a \emph{regularized} matrix game induced by $\hat{Q}^*$. Specifically, one solves 
\#\label{equ:solve_reg_minimax}
\big(\hat{\mu}(\cdot\given s),\hat{\nu}(\cdot\given s)\big)=\argmax_{u\in\Delta(\cA)}\min_{\vartheta\in\Delta(\cB)}~~\EE_{a\sim u,b\sim \vartheta}\big[\hat{Q}^*(s,a,b)\big]-\tau_1\Omega_{1}(u)+\tau_2\Omega_{2}(\vartheta),
\#
for each $s\in\cS$, 
where $\Omega_{i}$ is the regularizer for agent $i$'s policy, usually a strongly convex function, $\tau_i>0$ are the temperature parameters. 
This strongly-convex-strongly-concave saddle point problem admits a unique solution, and can be solved efficiently \citep{facchinei2007finite,cherukuri2017role,pmlr-v89-liang19b}. 
This regularization has been widely used  in both single-agent MDPs \citep{neu2017unified,pmlr-v80-haarnoja18b,chow2018path,pmlr-v97-geist19a}, 
and learning in games \citep{syrgkanis2015fast,mertikopoulos2016learning,grill2019planning}, to improve both the exploration and convergence.   

With small enough $\tau_i$ \kzedit{(with the order of $\cO(\epsilon)$, see \S\ref{append:planning})}, the solution to \eqref{equ:solve_reg_minimax} will be  $\epsilon$-close to that of the unregularized one \citep{pmlr-v97-geist19a}. 
More importantly, many commonly used regularizations, including negative entropy \citep{neu2017unified}, Tsallis entropy \citep{chow2018path} and R\'enyi entropy with certain parameters \citep{mertikopoulos2016learning},   naturally yield a smooth \textbf{Planning Oracle}; 
 see Lemma \ref{lemma:smooth_reg} in \S\ref{append:planning} for a formal statement. 
Note that the smoothness 
of the oracle 
does not affect the sample complexity of our model-based MARL algorithm. 

Now we  present  another theorem, which gives the $\epsilon$-Nash equilibrium \emph{policy  pair} directly, with the (near-)minimax optimal sample complexity of $\tilde\cO(|\cS||\cA||\cB|(1-\gamma)^{-3}\epsilon^{-2})$. 

\begin{theorem}[Finding $\epsilon$-NE Policy with a Smooth \textbf{Planning Oracle}] \label{thm:main_res_2}
	Suppose that the policy pair $(\hat{\mu},\hat{\nu})$ is obtained from a smooth  \textbf{Planning Oracle} using the empirical model $\hat{\cG}$ (see Definition \ref{def:smooth_oracle}),  which satisfies  
	\$
	\|\hat{V}^{\hat{\mu},*}-\hat{V}^*\|_\infty\leq \epsilon_{opt},\qquad \|\hat{V}^{*,\hat{\nu}}-\hat{V}^*\|_\infty\leq \epsilon_{opt}. 
	\$  
	Then, 
	for any $\delta\in(0,1]$ and {$\epsilon\in(0,1/(1-\gamma)^{1/2}]$}, if 
	\$
	N\geq \frac{c\gamma \log\big[c(C+1)|\cS||\cA||\cB|(1-\gamma)^{-4}\delta^{-1}\big]}{(1-\gamma)^3\epsilon^2}
	\$
	for some absolute constant $c$, then, letting  $\tilde \epsilon:=\epsilon+{4\epsilon_{opt}}/({1-\gamma})$, with probability at least $1-\delta$, 
	\$
	V^{*,\hat{\nu}}-2\tilde \epsilon\leq V^{\hat{\mu},\hat{\nu}}\leq V^{\hat{\mu},*}+2\tilde \epsilon, 
	\$
	namely, $(\hat{\mu},\hat{\nu})$ constitutes a $2\tilde \epsilon$-Nash equilibrium policy pair. 
\end{theorem}

Theorem \ref{thm:main_res_2}   shows that 
the sample complexity of achieving an $\epsilon$-NE policy can be near-minimax optimal for the reward-aware case, and minimax-optimal for the reward-agnostic case, if a smooth \textbf{Planning Oracle} is used. The dependence on $|\cS|$ and $1-\gamma$ also  matches the only known near-optimal complexity in MGs in \cite{sidford2019solving}, with a turn-based setting and a model-free algorithm. 
Inherited from \cite{yang2019optimality},  this improves the second result in \cite{azar2013minimax} that also has $\tilde\cO((1-\gamma)^{-3}\epsilon^{-2})$ in finding an $\epsilon$-optimal policy, by removing the dependence on $|\cS|^{-1/2}$ and enlarging the choice of $\epsilon$ from $(0,(1-\gamma)^{-1/2}|\cS|^{-1/2}]$ to $(0,(1-\gamma)^{-1/2}]$, and removing a  factor of $|\cS|$ in the total sample complexity for any fixed $\epsilon$.  
In addition, Theorem \ref{thm:main_res_2} also applies to the multi-reward setting, as Theorem \ref{thm:main_res}, by   taking a union bound argument over all reward functions in the reward-agnostic case. If the number of reward functions $M$ is of order $\text{poly}(|\cS|,|\cA|,|\cB|)$, the sample complexity of handling multiple reward functions has  the same order as that in Theorem \ref{thm:main_res_2}. 

Theorems \ref{thm:main_res} and \ref{thm:main_res_2} together   justify that, this simple model-based MARL algorithm is indeed sample-efficient, in approximating  both the Nash equilibrium  values and policies. Moreover, our separation of the reward-aware and reward-agnostic cases   highlights both the power (easily handling multiple reward functions), and the limitation (less adaptive and can hardly  achieve $\tilde\cO(|\cA|+|\cB|)$)  of the model-based approach, particularly arising in the multi-agent RL  context.

\section{Proofs}\label{sec:proof}

We first introduce some additional notation for convenience. 

\paragraph{Notation.} For a matrix $X\in\RR^{m\times n}$, $X\geq c$ for some scalar $c\in\RR$ means that each element of $X$ is no-less than $c$. For a vector $x$, we use $(x^2),~\sqrt{x},~|x|$ to denote  the component-wise square, square-root, and absolute value of $x$. We use $P_{(s,a,b),s'}$ to denote the transition probability $P(s'\given s,a,b)$, and $P_{s,a,b}$ to denote the vector $P(\cdot\given s,a,b)$.   We also use $P^{\mu,\nu}$ to denote the transition probability of state-action pairs induced by the policy pair $(\mu,\nu)$, which is defined as 
\$
P^{\mu,\nu}_{(s,a,b),(s',a',b')}=\mu(a'\given s')\nu(b'\given s')P(s'\given s,a,b).
\$
Hence, the Q-value function can be written as 
\$
Q^{\mu,\nu}=r+\gamma P^{\mu,\nu}Q^{\mu,\nu}=(I-\gamma P^{\mu,\nu})^{-1}r.
\$ 
Also, for any $V\in\RR^{|\cS|}$, we define the vector $\Var_P(V)\in\RR^{|\cS|\times|\cA|\times|\cB|}$ as
\$
\Var_P(V)(s,a,b):=\Var_{P(\cdot\given s,a,b)}(V)=P(V)^2-(PV)^2.
\$ 
Then, we define $\Sigma_{\cG}^{\mu,\nu}$ to be the variance of the discounted reward under the MG $\cG$, i.e.,
\$
\Sigma_{\cG}^{\mu,\nu}(s,a,b):=\EE\Big[\Big(\sum_{t=0}^\infty \gamma^t r(s_t,a_t,b_t)-Q_\cG^{\mu,\nu}(s,a,b)\Big)^2\biggiven s_0=s,a_0=a,b_0=b\Big]. 
\$
It can be shown (see an almost identical formula for MDPs in {\cite[Lemma 6]{azar2013minimax}}) that $\Sigma_{\cG}^{\mu,\nu}$ satisfies some \emph{Bellman-type} equation for any policy pair $(\mu,\nu)$: 
\#\label{equ:bellman_variance}
\Sigma_{\cG}^{\mu,\nu}=\gamma^2\Var_P(V_\cG^{\mu,\nu})+\gamma^2 P^{\mu,\nu} \Sigma_{\cG}^{\mu,\nu}. 
\#
It can also be verified that $\|\Sigma_{\cG}^{\mu,\nu}\|_\infty\leq \gamma^2/(1-\gamma)^2$ \citep{azar2013minimax,yang2019optimality}. 
Before proceeding further, we provide a roadmap for the proof. 
 
\paragraph{Proof Roadmap.} Our proof mainly consists of the following  steps:
\begin{enumerate}[leftmargin=1cm]
	\item \textbf{Helper lemmas and a crude bound.} We first establish several important lemmas, including the component-wise error bounds for the final Q-value errors, the variance error bound, and a crude error bound that directly uses Hoeffding's inequality. \kzedit{Some of the results are adapted from the single-agent setting  to zero-sum MGs, see \cite{yang2019optimality}.} See \S\ref{subsec:helper_lemma}.
	\item \textbf{Establishing an auxiliary Markov game.} To improve the crude bound, we build up an \emph{absorbing Markov game}, in order to handle the statistical dependence between $\hat{P}$ and some value function  generated by $\hat{P}$, which occurs as a product in the component-wise bound above. By carefully designing the auxiliary game, we establish a Bernstein-like concentration inequality, despite this dependency. See \S\ref{subsec:aux_game}, more precisely, Lemmas  \ref{lemma:absorb_main_res_2} and \ref{lemma:V_hat_mu_nu_error}. 
	\item \textbf{Final bound for $\epsilon$-approximate NE value.} Lemma \ref{lemma:absorb_main_res_2} in Step \textbf{2} allows us to exploit the variance bound, see Lemma \ref{lemma:bellman_variance}, to obtain an $\tilde\cO(\sqrt{1/[(1-\gamma)^3]N})$ order bound on the Q-value error, leading to a $\tilde \cO((1-\gamma)^{-3}\epsilon^{-2})$  near-minimax optimal sample complexity for achieving the $\epsilon$-approximate NE value. See \S\ref{sec:proof_main}.  
	\item \textbf{Final bounds for $\epsilon$-NE policy.} Based on the final bound in Step \textbf{3}, we then establish a $\tilde \cO((1-\gamma)^{-5}\epsilon^{-2})$ sample complexity for obtaining an  $\epsilon$-NE policy pair, by solving an additional matrix game over the output Q-value $\hat{Q}^{\hat{\mu},\hat{\nu}}$. See \S\ref{sec:proof_main_coro}. In addition, given a smooth \textbf{Planning Oracle}, by Lemma \ref{lemma:V_hat_mu_nu_error} in Step \textbf{2}, and more careful self-bounding techniques, we establish a $\tilde \cO((1-\gamma)^{-3}\epsilon^{-2})$ sample  complexity for achieving such an $\epsilon$-NE policy, directly using the output policies $(\hat{\mu},\hat{\nu})$.  
	See \S\ref{sec:proof_main_2}. 
\end{enumerate}

\subsection{Important Lemmas}\label{subsec:helper_lemma}

We start with the component-wise error bounds. 

\begin{lemma}[Component-Wise Bounds]\label{lemma:comp_wise_bnd}
	For any policy pair $(\mu,\nu)$, it follows that
	\$
	&\qquad\qquad\qquad\qquad Q^{\mu,\nu}-\hat{Q}^{\mu,\nu}=\gamma(I-\gamma P^{\mu,\nu})^{-1}(P-\hat{P})\hat{V}^{\mu,\nu}, 
	\\
	&
	{\gamma (I-\gamma \hat{P}^{\mu,{\nu(\mu)}})^{-1}(P-\hat {P}){V}^{\mu,*}\leq Q^{\mu,*}-\hat{Q}^{\mu,*}\leq \gamma (I-\gamma P^{\mu,\hat{\nu(\mu)}})^{-1}(P-\hat {P})\hat{V}^{\mu,*},}\\
	&\gamma (I-\gamma P^{\hat{\mu(\nu)},\nu})^{-1}(P-\hat {P})\hat{V}^{*,\nu}\leq Q^{*,\nu}-\hat{Q}^{*,\nu}\leq \gamma (I-\gamma \hat{P}^{\mu(\nu),\nu})^{-1}(P-\hat {P}){V}^{*,\nu},
	\$ 
	where we recall that $\nu(\mu)$ and $\mu(\nu)$ denote the best-response policy given $\mu$ and $\nu$, respectively (see \eqref{equ:def_best_response}). 
	Moreover, we have
	\#
	&Q^{\mu,\nu}\geq Q^*-\|Q^{\mu,\nu}-\hat{Q}^{\mu,\nu}\|_{\infty}-\|\hat{Q}^{\mu,\nu}-\hat{Q}^*\|_\infty-\|\hat{Q}^{\mu^*,*}-Q^*\|_\infty\label{equ:error_decomp_1}\\
	& Q^{\mu,\nu}\leq Q^*+\|Q^{\mu,\nu}-\hat{Q}^{\mu,\nu}\|_{\infty}+\|\hat{Q}^{\mu,\nu}-\hat{Q}^*\|_\infty+\|\hat{Q}^{*,\nu^*}-Q^*\|_\infty\label{equ:error_decomp_2}
	\\
	&V^{\mu,*}\geq V^*-\|Q^{\mu,*}-\hat{Q}^{\mu,*}\|_{\infty}-\|\hat{V}^{\mu,*}-\hat{V}^*\|_\infty-\|\hat{Q}^{\mu^*,*}-Q^*\|_\infty\label{equ:error_decomp_5}\\
	& V^{*,\nu}\leq V^*+\|Q^{*,\nu}-\hat{Q}^{*,\nu}\|_{\infty}+\|\hat{V}^{*,\nu}-\hat{V}^*\|_\infty+\|\hat{Q}^{*,\nu^*}-Q^*\|_\infty.\label{equ:error_decomp_6}
	\#
\end{lemma}
\begin{proof} 
First, note that
\small
\$
&Q^{\mu,\nu}-\hat{Q}^{\mu,\nu}=(I-\gamma P^{\mu,\nu})^{-1}r-(I-\gamma\hat{P}^{\mu,\nu})^{-1}r=(I-\gamma P^{\mu,\nu})^{-1}[(I-\gamma \hat{P}^{\mu,\nu})-(I-\gamma {P}^{\mu,\nu})]\hat{Q}^{\mu,\nu}\\
&\quad =\gamma (I-\gamma P^{\mu,\nu})^{-1}(P^{\mu,\nu}-\hat {P}^{\mu,\nu})\hat{Q}^{\mu,\nu}=\gamma (I-\gamma P^{\mu,\nu})^{-1}(P-\hat {P})\hat{V}^{\mu,\nu},
\$
\normalsize
proving the first equation. 
Also, 
\small
\$
&Q^{\mu,*}-\hat{Q}^{\mu,*}\leq Q^{\mu,\hat{\nu(\mu)}}-\hat{Q}^{\mu,*}
=Q^{\mu,\hat{\nu(\mu)}}-\hat{Q}^{\mu,\hat{\nu(\mu)}}\\
&\quad=
 \big(I-\gamma P^{\mu,\hat{\nu(\mu)}}\big)^{-1}r-\big(I-\gamma\hat{P}^{\mu,\hat{\nu(\mu)}}\big)^{-1}r\\
&\quad=\big(I-\gamma P^{\mu,\hat{\nu(\mu)}}\big)^{-1}\big[(I-\gamma \hat{P}^{\mu,\hat{\nu(\mu)}})-(I-\gamma {P}^{\mu,\hat{\nu(\mu)}})\big]\hat{Q}^{\mu,\hat{\nu(\mu)}}
=\gamma (I-\gamma P^{\mu,\hat{\nu(\mu)}})^{-1}(P-\hat {P})\hat{V}^{\mu,\hat{\nu(\mu)}},  
\$
\normalsize
where we recall that $\hat{\nu(\mu)}(\cdot\given s)\in\argmin \hat{V}^{\mu,\nu}(s)$ for all  $s\in\cS$. By similar arguments, recalling that ${\nu(\mu)}(\cdot\given s)\in\argmin {V}^{\mu,\nu}(s)$ for all $s$, we have 
\small
\$
&Q^{\mu,*}-\hat{Q}^{\mu,*}\geq Q^{\mu,{\nu(\mu)}}-\hat{Q}^{\mu,\nu(\mu)}=(I-\gamma P^{\mu,\nu(\mu)})^{-1}r-(I-\gamma\hat{P}^{\mu,\nu(\mu)})^{-1}r\\
&\quad= (I-\gamma \hat{P}^{\mu,\nu(\mu)})^{-1}[(I-\gamma \hat{P}^{\mu,\nu(\mu)})-(I-\gamma {P}^{\mu,\nu(\mu)})]{Q}^{\mu,\nu(\mu)}=\gamma (I-\gamma \hat{P}^{\mu,{\nu(\mu)}})^{-1}(P-\hat {P}){V}^{\mu,*}. 
\$
\normalsize
Similar arguments yield the third inequality in the first argument. 

	For the second argument, we have
	\small
	\$
	Q^{\mu,\nu}-Q^*=Q^{\mu,\nu}-\hat{Q}^*+\hat{Q}^*-Q^*\geq Q^{\mu,\nu}-\hat{Q}^*+\hat{Q}^{\mu^*,*}-Q^*\geq -\|Q^{\mu,\nu}-\hat{Q}^*\|_\infty-\|\hat{Q}^{\mu^*,*}-Q^*\|_\infty,
	\$
	\normalsize
	which, combined with triangle inequality, yields the first inequality. Similarly, we have
	\small
	\$
	Q^{\mu,\nu}-Q^*=Q^{\mu,\nu}-\hat{Q}^*+\hat{Q}^*-Q^*\leq Q^{\mu,\nu}-\hat{Q}^*+\hat{Q}^{*,\nu^*}-Q^*\leq \|Q^{\mu,\nu}-\hat{Q}^*\|_\infty+\|\hat{Q}^{*,\nu^*}-Q^*\|_\infty. 
	\$
	\normalsize
	Using triangle inequality proves the second inequality. 
	For \eqref{equ:error_decomp_5}-\eqref{equ:error_decomp_6}, we similarly have
	\small
	\#
	V^{\mu,*}-V^*&=V^{\mu,*}-\hat{V}^*+\hat{V}^*-V^*\geq V^{\mu,*}-\hat{V}^*+\hat{V}^{\mu^*,*}-V^*\geq -\|V^{\mu,*}-\hat{V}^*\|_\infty-\|\hat{V}^{\mu^*,*}-V^*\|_\infty,\label{equ:componentwise_bnd_trash_1}\\
	V^{*,\nu}-V^*&=V^{*,\nu}-\hat{V}^*+\hat{V}^*-V^*\leq V^{*,\nu}-\hat{V}^*+\hat{V}^{*,\nu}-V^*\leq \|V^{*,\nu}-\hat{V}^*\|_\infty+\|\hat{V}^{*,\nu}-V^*\|_\infty.\label{equ:componentwise_bnd_trash_2} 
	\#
	\normalsize
	Notice that for any $\mu\in\Delta(\cA)^{|\cS|}$ and $\nu\in\Delta(\cB)^{|\cS|}$,  
	\#
	&\|V^{\mu,*}-\hat{V}^{\mu,*}\|_{\infty}=\Big\|\min_{\vartheta\in\Delta(\cB)}\EE_{a\sim \mu(\cdot\given s),b\sim \vartheta}[Q^{\mu,*}(\cdot,a,b)]- \min_{\vartheta\in\Delta(\cB)}\EE_{a\sim \mu(\cdot\given s),b\sim \vartheta}[\hat{Q}^{\mu,*}(\cdot,a,b)]\Big\|_\infty\notag\\
	&\quad\leq \max_{\vartheta\in\Delta(\cB)}\big\|\EE_{a\sim \mu(\cdot\given s),b\sim \vartheta}[Q^{\mu,*}(\cdot,a,b)]-  \EE_{a\sim \mu(\cdot\given s),b\sim \vartheta}[\hat{Q}^{\mu,*}(\cdot,a,b)]\big\|_\infty\leq \|Q^{\mu,*}-\hat{Q}^{\mu,*}\|_{\infty}\label{equ:componentwise_bnd_trash_3}\\
	&\|V^{*,\nu}-\hat{V}^{*,\nu}\|_{\infty}=\Big\|\max_{u\in\Delta(\cA)}\EE_{a\sim u,b\sim \nu(\cdot\given s)}[Q^{*,\nu}(\cdot,a,b)]- \max_{u\in\Delta(\cA)}\EE_{a\sim u,b\sim \nu(\cdot\given s)}[\hat{Q}^{*,\nu}(\cdot,a,b)]\Big\|_\infty\notag\\
	&\quad\leq \max_{u\in\Delta(\cA)}\big\|\EE_{a\sim u,b\sim \nu(\cdot\given s)}[Q^{*,\nu}(\cdot,a,b)]-  \EE_{a\sim u,b\sim \nu(\cdot\given s)}[\hat{Q}^{*,\nu}(\cdot,a,b)]\big\|_\infty\leq \|Q^{*,\nu}-\hat{Q}^{*,\nu}\|_{\infty}. \label{equ:componentwise_bnd_trash_4}
	\#
	Combining \eqref{equ:componentwise_bnd_trash_1}-\eqref{equ:componentwise_bnd_trash_2} and \eqref{equ:componentwise_bnd_trash_3}-\eqref{equ:componentwise_bnd_trash_4}, together with triangle inequality, we arrive at \eqref{equ:error_decomp_5}-\eqref{equ:error_decomp_6}, and complete the proof. 
\end{proof}

\kzedit{We establish the decomposition in  \eqref{equ:error_decomp_1}-\eqref{equ:error_decomp_2} for the following intuition and reasons. The error in \eqref{equ:error_decomp_1}-\eqref{equ:error_decomp_2} contains three terms:  the first and third terms $\|Q^{\mu,\nu}-\hat{Q}^{\mu,\nu}\|_{\infty}$ and $\hat{Q}^*\|_\infty-\|\hat{Q}^{\mu^*,*}-Q^*\|_\infty$ are the differences of the Q-value for some policy pairs in the true and estimated models, respectively, which will be handled later based on the {\it statistical error} of the model estimation; the second term    $\|\hat{Q}^{\mu,\nu}-\hat{Q}^*\|_\infty$ is the \emph{optimization error} we obtained from the algorithm that solves the empirical game, which will be controlled with an efficient \textbf{Planning Oracle}. To deal with the statistical errors, 
we first introduce the following lemma, which is adapted from Lemma $2$ in \cite{yang2019optimality}. }

\begin{lemma}\label{lemma:inf_norm_error}
	For any policy pair $(\mu,\nu)$ and vector $v\in\RR^{|\cS|\times|\cA|\times|\cB|}$,   $\|(I-\gamma P^{\mu,\nu})^{-1}v\|_\infty\leq \|v\|_\infty/(1-\gamma)$.
\end{lemma}
\begin{proof}
	The proof is straightforward. Letting  $w=(I-\gamma P^{\mu,\nu})^{-1}v$, we have 
	 $
	v=(I-\gamma P^{\mu,\nu})w
	$. Triangle inequality yields $\|v\|_\infty\geq \|w\|_\infty-\gamma\|P^{\mu,\nu} w\|_\infty\geq \|w\|_\infty-\gamma\|w\|_\infty$, which completes the proof. 
\end{proof}
 
Next we establish the Bellman property of a policy pair $(\mu,\nu)$'s variance and its accumulation. This has been observed for MDPs before in  \cite{munos1999variable, lattimore2012pac,azar2012sample,yang2019optimality}.  
We establish the counterpart for Markov games as follows. 

\begin{lemma}\label{lemma:bellman_variance}
	For any policy pair $(\mu,\nu)$ and MG $\cG$ with transition model $P$, we have
	\$
	\Big\|\big(I-\gamma P^{\mu,\nu}\big)^{-1}\sqrt{\Var_P\big(V^{\mu,\nu}_\cG\big)}\Big\|_\infty\leq \sqrt{\frac{2}{(1-\gamma)^3}}. 
	\$
\end{lemma} 
\begin{proof}
	The proof follows that of \cite[Lemma 3]{yang2019optimality}. For any positive vector $v$, by Jensen's inequality, we have
	\#\label{equ:lemma:bellman_variance_trash_1}
	\|(I-\gamma P^{\mu,\nu})^{-1}\sqrt{v}\|_\infty=\frac{1}{1-\gamma}\|(1-\gamma)(I-\gamma P^{\mu,\nu})^{-1}\sqrt{v}\|_\infty\leq \sqrt{\Big\|\frac{1}{1-\gamma}(I-\gamma P^{\mu,\nu})^{-1}{v}\Big\|_\infty}.
	\#
	Also, observe that
	\#
	&\|(I-\gamma P^{\mu,\nu})^{-1}{v}\|_\infty =\|(I-\gamma P^{\mu,\nu})^{-1}(I-\gamma^2 P^{\mu,\nu})(I-\gamma^2 P^{\mu,\nu})^{-1}{v}\|_\infty\notag\\
	&\quad = \big\|[(I-\gamma P^{\mu,\nu})^{-1}(1-\gamma + \gamma - \gamma^2P^{\mu,\nu})](I-\gamma^2 P^{\mu,\nu})^{-1}{v}\big\|_\infty\notag\\
	&\quad=\big\|[(1-\gamma)(I-\gamma P^{\mu,\nu})^{-1}+\gamma I](I-\gamma^2 P^{\mu,\nu})^{-1}{v}\big\|_\infty\notag\\
	&\quad\leq (1-\gamma)\big\|(I-\gamma P^{\mu,\nu})^{-1}(I-\gamma^2 P^{\mu,\nu})^{-1}{v}\big\|_\infty+\gamma \big\|(I-\gamma^2 P^{\mu,\nu})^{-1}{v}\big\|_\infty\notag\\
	&\quad \leq \frac{1-\gamma}{1-\gamma}\big\|(I-\gamma^2 P^{\mu,\nu})^{-1}{v}\big\|_\infty+\gamma \big\|(I-\gamma^2 P^{\mu,\nu})^{-1}{v}\big\|_\infty\leq 2\big\|(I-\gamma^2 P^{\mu,\nu})^{-1}{v}\big\|_\infty.\label{equ:lemma:bellman_variance_trash_2}
	\#
	Combining \eqref{equ:lemma:bellman_variance_trash_1} and \eqref{equ:lemma:bellman_variance_trash_2} yields 
	\#\label{equ:lemma:bellman_variance_trash_3}
	\|(I-\gamma P^{\mu,\nu})^{-1}\sqrt{v}\|_\infty\leq \sqrt{\Big\|\frac{2}{1-\gamma}(I-\gamma^2 P^{\mu,\nu})^{-1}{v}\Big\|_\infty}. 
	\#
	In addition, by \eqref{equ:bellman_variance},  we have $\Sigma_\cG^{\mu,\nu}=\gamma^2(I-\gamma^2 P^{\mu,\nu})^{-1}\Var_P(V_\cG^{\mu,\nu})$. Letting $v=\Var_P(V_\cG^{\mu,\nu})$ in \eqref{equ:lemma:bellman_variance_trash_3} and noticing that $\|\Sigma_\cG^{\mu,\nu}\|_\infty\leq \gamma^2/(1-\gamma)^2$ completes the proof.
\end{proof}

Finally, if we just apply Hoeffding's inequality, we obtain the following concentration argument, 
upon which we will improve to obtain our final results.

\begin{lemma}\label{lemma:crude_bnd}
	Let $(\mu^*,\nu^*)$ be the Nash equilibrium  policy pair under the actual model $\cG$. Then, for any $\delta\in(0,1]$,  with probability at least $1-\delta$, we have 
	\small
	\$
	\|Q^*-\hat{Q}^{\mu^*,\nu^*}\|_\infty\leq \Delta_{\delta,N},~~\|Q^*-\hat{Q}^{\mu^*,*}\|_\infty\leq \Delta_{\delta,N},~~ \|Q^*-\hat{Q}^{*,\nu^*}\|_\infty\leq \Delta_{\delta,N},~~ \|Q^*-\hat{Q}^*\|_\infty\leq \Delta_{\delta,N},
	\$
	\normalsize
	where 
	\$
	\Delta_{\delta,N}:=\frac{\gamma}{(1-\gamma)^2}\sqrt{\frac{2\log(2|\cS||\cA||\cB|/\delta)}{N}}. 
	\$
\end{lemma}
\begin{proof} 
	First note that $V^*$ is fixed and independent of the randomness in $\hat{P}$. Due to the boundedness of $V^*$ that $\|V^*\|_\infty\leq (1-\gamma)^{-1}$, and the union of  Hoeffding bounds over $\cS\times\cA\times\cB$, we have that with probability at least $1-\delta$
	\#\label{equ:cond_prob_event}
	\big\|(\hat{P}-P)V^*\big\|_\infty\leq \frac{1}{1-\gamma}\cdot\sqrt{\frac{2\log(2|\cS||\cA||\cB|/\delta)}{N}}.
	\#
  	On the other hand, let $\cT_{\mu,\nu}$ be the Bellman operator under the true transition model $P$, using any   joint policy $(\mu,\nu)$, i.e., for any $s\in\cS$ and $(s,a,b)\in\cS\times\cA\times\cB$, $V\in\RR^{|\cS|}$ and $Q\in\RR^{|\cS|\times|\cA|\times|\cB|}$:
  	\$  
  	\cT_{\mu,\nu}(V)(s)&=\EE_{a\sim \mu(\cdot\given s),b\sim \nu(\cdot\given s)}\big[r(s,a,b)+\gamma\cdot P(\cdot\given s,a,b)^\top V\big]\\
  	\cT_{\mu,\nu}(Q)(s,a,b)&=r(s,a,b)+\gamma \cdot\EE_{s'\sim P(\cdot\given s,a,b),a'\sim \mu(\cdot\given s'),b'\sim \nu(\cdot\given s')}\big[Q(s',a',b')\big].
  	\$
  	Similarly, let $\hat{\cT}_{\mu,\nu}$ be the corresponding operator defined under the estimated transition $\hat{P}$. 
  	Note that $\hat{Q}^{\mu,\nu}$ and $Q^*$ are the fixed points of $\hat{\cT}_{\mu,\nu}$ and $\cT_{\mu^*,\nu^*}$, respectively. 
  	We thus have
  	\#
  	&\|Q^*-\hat{Q}^{\mu,\nu}\|_\infty=\|\cT_{\mu^*,\nu^*}Q^*-\hat{\cT}_{\mu,\nu}\hat{Q}^{\mu,\nu}\|_\infty\notag\\
  	&\quad\leq \|\cT_{\mu^*,\nu^*}Q^*-r-\gamma\hat{P}^{\mu^*,\nu^*}Q^*\|_\infty+\|r+\gamma\hat{P}^{\mu^*,\nu^*}Q^*-\hat{\cT}_{\mu,\nu}\hat{Q}^{\mu,\nu}\|_\infty\notag\\
  	&\quad= \gamma\|{P}^{\mu^*,\nu^*}Q^*-\hat{P}^{\mu^*,\nu^*}Q^*\|_\infty+\gamma\|\hat{P}^{\mu^*,\nu^*}Q^*-\hat{P}^{\mu,\nu}\hat{Q}^{\mu,\nu}\|_\infty\notag\\
  	&\quad= \gamma\|{P}V^*-\hat{P}V^*\|_\infty+\gamma\|\hat{P}V^*-\hat{P}\hat{V}^{\mu,\nu}\|_\infty\leq \gamma\|({P}-\hat{P})V^*\|_\infty+\gamma\|V^*-\hat{V}^{\mu,\nu}\|_\infty\label{equ:crude_bnd_trash_1}.
  	\#
  	
  	To show the first argument, letting $\mu=\mu^*$ and $\nu=\nu^*$, we have 
  	\#\label{equ:crude_bnd_trash_2}
	\gamma\|V^*-\hat{V}^{\mu^*,\nu^*}\|_\infty &=\gamma\big\|\EE_{a\sim \mu^*(\cdot\given s),b\sim \nu^*(\cdot\given s)}[Q^*(\cdot,a,b)]-\EE_{a\sim \mu^*(\cdot\given s),b\sim \nu^*(\cdot\given s)}[\hat{Q}^{\mu^*,\nu^*}(\cdot,a,b)]\big\|_\infty\notag\\
	&\leq \gamma\|Q^*-\hat{Q}^{\mu^*,\nu^*}\|_\infty. 
  	\#
  	Using   \eqref{equ:crude_bnd_trash_2} to bound  the last term in  \eqref{equ:crude_bnd_trash_1}, and solving for $\|Q^*-\hat{Q}^{\mu^*,\nu^*}\|_\infty$ from \eqref{equ:crude_bnd_trash_1}, we obtain the first argument.

	For the second argument, letting $\mu=\mu^*$ and $\nu=\hat{\nu(\mu^*)}$ (note that $\hat{Q}^{\mu^*,*}=\hat{Q}^{\mu^*,\hat{\nu(\mu^*)}}$), we have
	\small 
	\#\label{equ:crude_bnd_trash_3}
	&\gamma\|V^*-\hat{V}^{\mu^*,*}\|_\infty =\gamma\cdot\big\|\min_{\vartheta\in\Delta(\cB)}\EE_{a\sim \mu^*(\cdot\given s),b\sim \vartheta}[Q^*(\cdot,a,b)]-\min_{\vartheta\in\Delta(\cB)}\EE_{a\sim \mu^*(\cdot\given s),b\sim \vartheta}[\hat{Q}^{\mu^*,*}(\cdot,a,b)]\big\|_\infty\notag\\
	&\quad\leq \gamma\cdot\max_{\vartheta\in\Delta(\cB)}\big\|\EE_{a\sim \mu^*(\cdot\given s),b\sim \vartheta}[Q^*(\cdot,a,b)]-\EE_{a\sim \mu^*(\cdot\given s),b\sim \vartheta}[\hat{Q}^{\mu^*,*}(\cdot,a,b)]\big\|_\infty \leq \gamma\|Q^*-\hat{Q}^{\mu^*,*}\|_\infty, 
  	\#
  	\normalsize
where the first inequality is due to the non-expansiveness of the $\min$ operator. 
  	Using   \eqref{equ:crude_bnd_trash_3} to bound  the last term in  \eqref{equ:crude_bnd_trash_1}, and solving for $\|Q^*-\hat{Q}^{\mu^*,*}\|_\infty$ from \eqref{equ:crude_bnd_trash_1}, we obtain the second argument. Similarly, we can obtain the third argument. 
	
	For the fourth argument, letting $\mu=\hat{\mu}^*$ and $\nu=\hat{\nu}^*$, the NE policy under $\hat{P}$ (note that $\hat{Q}^{\hat{\mu}^*,\hat{\nu}^*}=\hat{Q}^{*}$), we have
	\$
	&\gamma\|V^*-\hat{V}^{*}\|_\infty =\gamma\cdot\big\|\max_{u\in\Delta(\cA)}\min_{\vartheta\in\Delta(\cB)}\EE_{a\sim u,b\sim \vartheta}[Q^*(\cdot,a,b)]-\max_{u\in\Delta(\cA)}\min_{\vartheta\in\Delta(\cB)}\EE_{a\sim u,b\sim \vartheta}[\hat{Q}^{*}(\cdot,a,b)]\big\|_\infty\notag\\
	&\quad\leq \gamma\cdot\max_{u\in\Delta(\cA)}\big\|\min_{\vartheta\in\Delta(\cB)}\EE_{a\sim u,b\sim \vartheta}[Q^*(\cdot,a,b)]- \min_{\vartheta\in\Delta(\cB)}\EE_{a\sim u,b\sim \vartheta}[\hat{Q}^{*}(\cdot,a,b)]\big\|_\infty 
\leq \gamma\|Q^*-\hat{Q}^{*}\|_\infty, 
  	\$
  	where the inequalities are due to the non-expansivenesses of both the $\max$ and the $\min$ operators.  
	This, combined with \eqref{equ:crude_bnd_trash_1}, completes the proof. 
\end{proof}

\kzedit{The argument above will lead to a crude bound, with an additional $1/(1-\gamma)$ dependence compared to our main results  in Theorem \ref{thm:main_res} and Theorem \ref{thm:main_res_2}. The key reason is that we used some self-bounding of the error terms, e.g., $\|Q^*-\hat{Q}^{\mu^*,\nu^*}\|_\infty$, which appears on both sides of the inequality, with a $\gamma$ discounting on the right-hand side. This way, by subtracting the term on the right-hand side, we have an additional $1/(1-\gamma)$ order after dividing $(1-\gamma)$ on both sides. This was essentially due to the fact that the direct concentration argument can only deal with the concentration of $\big\|(\hat{P}-P)V^*\big\|_\infty$, where $\hat{P}$ and $V^*$ are {\it not} dependent as $V^*$ is a fixed vector. To obtain sharper rates, one has to directly deal with the quantities as $\big\|(\hat{P}-P)\hat{V}\big\|_\infty$, where $\hat{V}$ denotes some value function obtained from the empirical model, and is correlated with $\hat{P}$. Properly handling this interdependence will be the focus of our proof next.}

\subsection{An Auxiliary Markov Game}\label{subsec:aux_game}

Motivated by the \emph{absorbing MDP} technique in \cite{yang2019optimality}, we introduce an \emph{absorbing Markov game}, in order to handle the interdependence between $\hat{P}$ and $\hat{V}^{\mu,\nu}$, for any $\mu,\nu$ (which may also depend on $\hat{P}$), which will show up frequently in the analysis.  

We now define a new Markov game $\cG_{s,u}$ as follows (with $s\in\cS$ and $u\in\RR$ a constant): $\cG_{s,u}$ is identical to $\cG$, except that $P_{\cG_{s,u}}(s\given s,a,b)=1$ for all $(a,b)\in\cA\times\cB$, namely, state $s$ is an \emph{absorbing} state; and the instantaneous reward at $s$ is always $(1-\gamma)u$. The rest of the reward function and the transition model of $\cG_{s,u}$ are the same as those of $\cG$. For notational simplicity, we now use $X^{\mu,\nu}_{s,u}$ to denote  $X^{\mu,\nu}_{\cG_{s,u}}$, where $X$ can be either the value functions $Q$ and $V$, or the reward function $r$, under the model $\cG_{s,u}$. 
Obviously, for any policy pair $(\mu,\nu)$, $V^{\mu,\nu}_{s,u}(s)=u$ for the absorbing state $s$. 

In addition, we define $U_s$ for some state $s$ to choose $u$ from, which is a set of evenly spaced elements in the interval $[V^*(s)-\Delta,V^*(s)+\Delta]$ for some $\Delta>0$, i.e., $U_s\subset [V^*(s)-\Delta,V^*(s)+\Delta]$. An appropriately chosen  size of $|U_s|$ will be the key in the proof. We also use $\hat{P}_{\cG_{s,u}}$ to denote the  transition model of the absorbing MG for the empirical MG $\hat{\cG}$, denoted by $\hat{\cG}_{s,u}$.  Specifically, at all non-absorbing states, $\hat{P}_{\cG_{s,u}}$ is identical to $\hat{P}$; while at the absorbing state, $\hat{P}_{\cG_{s,u}}(s\given s,a,b)=1$ for any $(a,b)\in\cA\times\cB$. The corresponding value functions are for short denoted by $\hat{V}^{\mu,\nu}_{s,u}$ and $\hat{Q}^{\mu,\nu}_{s,u}$. 
Similar as in the original MG, we also use $\hat{V}^{*}_{s,u}$ to denote the NE value under the model $\hat{\cG}_{s,u}$, and use $\hat{V}^{\mu,*}_{s,u}$ and $\hat{V}^{*,\nu}_{s,u}$ to denote the best-response values of some given $\mu$ and $\nu$, under the model $\hat{\cG}_{s,u}$. 
Now we first have the following lemma based on Bernstein's inequality; see a similar argument in Lemma 5 in \cite{yang2019optimality}.  

\begin{lemma}\label{lemma:bernstein_inq}
	For fixed state $s$, action $(a,b)$, a finite set $U_s$, 
	and $\delta> 0$, it holds that for all $u\in U_s$, with probability greater than $1-\delta$, 
	\$
	&\big|(P_{s,a,b}-\hat{P}_{s,a,b})\cdot \hat{V}^*_{s,u}\big|\leq \sqrt{\frac{2\log(4|U_s|/\delta)\cdot \Var_{P_{s,a,b}}(\hat{V}^*_{s,u})}{N}}+\frac{2\log(4|U_s|/\delta)}{3(1-\gamma)N},\\
	&\big|(P_{s,a,b}-\hat{P}_{s,a,b})\cdot \hat{V}^{\mu^*,*}_{s,u}\big|\leq \sqrt{\frac{2\log(4|U_s|/\delta)\cdot \Var_{P_{s,a,b}}(\hat{V}^{\mu^*,*}_{s,u})}{N}}+\frac{2\log(4|U_s|/\delta)}{3(1-\gamma)N},
	\$
	\$
	&\big|(P_{s,a,b}-\hat{P}_{s,a,b})\cdot \hat{V}^{*,\nu^*}_{s,u}\big|\leq \sqrt{\frac{2\log(4|U_s|/\delta)\cdot \Var_{P_{s,a,b}}(\hat{V}^{*,\nu^*}_{s,u})}{N}}+\frac{2\log(4|U_s|/\delta)}{3(1-\gamma)N},\\
	&\big|(P_{s,a,b}-\hat{P}_{s,a,b})\cdot \hat{V}^{\mu^*,\nu^*}_{s,u}\big|\leq \sqrt{\frac{2\log(4|U_s|/\delta)\cdot \Var_{P_{s,a,b}}(\hat{V}^{\mu^*,\nu^*}_{s,u})}{N}}+\frac{2\log(4|U_s|/\delta)}{3(1-\gamma)N},\\
	&\big|(P_{s,a,b}-\hat{P}_{s,a,b})\cdot {V}^{\hat{\mu}_{s,u},*}\big|\leq \sqrt{\frac{2\log(4|U_s|/\delta)\cdot \Var_{P_{s,a,b}}({V}^{\hat{\mu}_{s,u},*})}{N}}+\frac{2\log(4|U_s|/\delta)}{3(1-\gamma)N},\\
	&\big|(P_{s,a,b}-\hat{P}_{s,a,b})\cdot {V}^{*,\hat{\nu}_{s,u}}\big|\leq \sqrt{\frac{2\log(4|U_s|/\delta)\cdot \Var_{P_{s,a,b}}({V}^{*,\hat{\nu}_{s,u}})}{N}}+\frac{2\log(4|U_s|/\delta)}{3(1-\gamma)N},
	\$
	where  $P_{s,a,b}$ and $\hat{P}_{s,a,b}$ are the transition models extracted from the original game $\cG$ and its empirical version $\hat{\cG}$, respectively (not related to either $\cG_{s,u}$ or $\hat{\cG}_{s,u}$), and $(\hat{\mu}_{s,a},\hat{\nu}_{s,a})$ is the output of the \textbf{Planning Oracle} using the auxiliary  empirical model $\hat{\cG}_{s,u}$
\end{lemma}
\begin{proof}
	The key observation is that  the random variables $\hat{P}_{s,a,b}$ and $\hat{V}^{*}_{s,u}$ are independent. Using Bernstein's inequality along with a union bound over all $u\in U_s$, we obtain the first inequality. The other inequalities   follow similarly, as  $\hat{P}_{s,a,b}$ is independent of  $\hat{V}^{\mu^*,*}_{s,u}$, $\hat{V}^{*,\nu^*}_{s,u}$,  $\hat{V}^{\mu^*,\nu^*}_{s,u}$,  ${V}^{\hat{\mu}_{s,u},*}$, and ${V}^{*,\hat{\nu}_{s,u}}$. This is because the latter terms are  all decided by the original game $\cG$, and/or  the auxiliary empirical game $\hat{\cG}_{s,u}$ (not the original empirical game $\hat{\cG}$). 
\end{proof}

Note that the arguments in Lemma \ref{lemma:bernstein_inq} do not hold, if we replace $\hat{V}^*_{s,u}$ by $\hat{V}^*$, or $\hat{V}^{\mu^*,*}_{s,u}$ by $\hat{V}^{\mu^*,*}$, or $\hat{V}^{*,\nu^*}_{s,u}$ by $\hat{V}^{*,\nu^*}$. \kzedit{It will neither hold if we replace $\hat{V}^{\mu^*,*}_{s,u}$  and ${V}^{\hat{\mu}_{s,u},*}$ by some $\hat{V}^{\mu,*}$ and ${V}^{\mu,*}$, for any $\mu$ that is dependent on $\hat{P}$, e.g., the NE policy $\hat{\mu}^*$ for the original empirical game $\hat{\cG}$. This is one of the key subtleties that are worth emphasizing.} 

Next we establish two helpful lemmas that help guide the choices of $U_s$, so that $\hat{V}^*_{s,u}$ (resp. $\hat{V}^{\mu^*,*}_{s,u}$, $\hat{V}^{*,\nu^*}_{s,u}$, and $\hat{V}^{\mu^*,\nu^*}_{s,u}$) will be a good approximate of $\hat{V}^*$ (resp. $\hat{V}^{\mu^*,*}$, $\hat{V}^{*,\nu^*}$, and $\hat{V}^{\mu^*,\nu^*}$). 

\begin{lemma}\label{lemma:helper_u_choose_1}
For the absorbing state $s$, and any joint policy $(\mu,\nu)$, suppose that $u^*=V^*_{\cG}(s)$, $u^{\mu,*}=V^{\mu,*}_{\cG}(s)$, $u^{*,\nu}=V^{*,\nu}_{\cG}(s)$, and $u^{\mu,\nu}=V^{\mu,\nu}_{\cG}(s)$. Then, 
\$
V^*_{\cG}=V^*_{s,u^*}\qquad V^{\mu,*}_{\cG}=V^{\mu,*}_{\cG_{s,u^{\mu,*}}} \qquad V^{*,\nu}_{\cG}=V^{*,\nu}_{\cG_{s,u^{*,\nu}}}\qquad V^{\mu,\nu}_{\cG}=V^{\mu,\nu}_{\cG_{s,u^{\mu,\nu}}}. 
\$
\end{lemma}
\begin{proof}
	For the first formula, we need to verify that $V^*_{\cG}$ satisfies the optimal (Nash equilibrium)  Bellman equation for the game $\cG_{s,u^*}$. To this end, note that if $s'=s$, then  $u^*=V^{*}_{\cG}(s)$ satisfies the Bellman equation trivially,  since $s$ is absorbing with the value $V^*_{s,u^*}(s)=u^*$.  
	
	On the other hand, 	for any $s'\neq s$, the outgoing transition model at $s'$ in $\cG_{s,u^*}$ is the same as that in $\cG$, and $V^*_{\cG}(s')$ per se satisfies the Bellman equation in $\cG$ (which are the same for $\cG_{s,u^*}$ at these states $s'\neq s$). Thus, $V^*_{\cG}$ satisfies the Bellman equation in $\cG_{s,u^*}$ for all states. This proves the first equation. The proofs for the remaining three equations are analogous. 
\end{proof}

Perfect choices of $u$ have been specified in Lemma \ref{lemma:helper_u_choose_1} above. \kzedit{Moreover, we need to quantify how the value changes if we deviate from these perfect choices, i.e., the robustness to misspecification of $u$ \citep{yang2019optimality}.} This result is formally established in the following lemma; see also Lemma 7 in \cite{yang2019optimality} for a similar result.  

\begin{lemma}\label{lemma:helper_u_choose_2}
	For any state $s$, $u,u'\in\RR$, and joint policy pair $(\mu,\nu)$, we have 
	\$
	&\big\|V^*_{s,u}-V^*_{s,u'}\big\|_\infty\leq |u-u'|,\qquad \big\|V^{\mu,*}_{s,u}-V^{\mu,*}_{s,u'}\big\|_\infty\leq |u-u'|,\\
	&\big\|V^{*,\nu}_{s,u}-V^{*,\nu}_{s,u'}\big\|_\infty\leq |u-u'|,\qquad \big\|V^{\mu,\nu}_{s,u}-V^{\mu,\nu}_{s,u'}\big\|_\infty\leq |u-u'|.
	\$
\end{lemma}
\begin{proof}
	Note that $\|r_{s,u}-r_{s,u'}\|_\infty=(1-\gamma)|u-u'|$, since the reward functions only differ at $s$, where $r_{s,u}(s,a,b)=(1-\gamma)u$ and $r_{s,u'}(s,a,b)=(1-\gamma)u'$. We denote the NE policy pair in $\cG_{s,u}$ by $(\mu^*_{s,u},\nu^*_{s,u})$.  Thus,
	\# 
	&Q^*_{s,u}-Q^*_{s,u'}=Q^{\mu^*_{s,u},\nu^*_{s,u}}_{s,u}-Q^{\mu^*_{s,u'},\nu^*_{s,u'}}_{s,u'}\leq Q^{\mu^*_{s,u},\nu^*_{s,u'}}_{s,u}-Q^{\mu^*_{s,u},\nu^*_{s,u'}}_{s,u'}\label{equ:robust_mis_1}\\
	&\quad=\big(I-\gamma P^{\mu^*_{s,u},\nu^*_{s,u'}}_{s,u}\big)^{-1}r_{s,u}-\big(I-\gamma P^{\mu^*_{s,u},\nu^*_{s,u'}}_{s,u'}\big)^{-1}r_{s,u'}\label{equ:robust_mis_2}\\
	&\quad = \big(I-\gamma P^{\mu^*_{s,u},\nu^*_{s,u'}}_{s,u}\big)^{-1}\big(r_{s,u}-r_{s,u'}\big)\label{equ:robust_mis_3}\\
	&\quad \leq \frac{\|r_{s,u}-r_{s,u'}\|_\infty}{1-\gamma}=|u-u'|,\label{equ:robust_mis_4}
	\#
	where \eqref{equ:robust_mis_1} uses the fact that at the NE, 
	\$
	V^{\mu^*_{s,u},\nu^*_{s,u}}_{s,u}=\min_{\nu} V^{\mu^*_{s,u},\nu}_{s,u}\leq V^{\mu^*_{s,u},\nu^*_{s,u'}}_{s,u},\quad V^{\mu^*_{s,u'},\nu^*_{s,u'}}_{s,u'}=\max_{\mu}V^{\mu,\nu^*_{s,u'}}_{s,u'}\geq V^{\mu^*_{s,u},\nu^*_{s,u'}}_{s,u'},
	\$
	implying the relationships of the corresponding Q-values; \eqref{equ:robust_mis_2} is by definition; \eqref{equ:robust_mis_3} uses the observation that $P^{\mu^*_{s,u},\nu^*_{s,u'}}_{s,u}$ is the same as $P^{\mu^*_{s,u},\nu^*_{s,u'}}_{s,u'}$ (transition is not affected by the value of $u$). Similarly, we can establish the lower bound that $Q^*_{s,u}-Q^*_{s,u'}\geq -|u-u'|$, which proves $\|Q^*_{s,u}-Q^*_{s,u'}\|_\infty\leq |u-u'|$. Moreover,  we have
	\$
	&\big\|V^{*}_{s,u}-V^{*}_{s,u'}\big\|_\infty=\Big\|\max_{u\in\Delta(\cA)}\min_{\vartheta\in\Delta(\cB)}\EE_{a\sim u,b\sim \vartheta}[Q^{*}_{s,u}(\cdot,a,b)]-\max_{u\in\Delta(\cA)}\min_{\vartheta\in\Delta(\cB)}\EE_{a\sim u,b\sim \vartheta}[Q^{*}_{s,u'}(\cdot,a,b)]\Big\|_\infty\\
	&\quad \leq \max_{u\in\Delta(\cA),\vartheta\in\Delta(\cB)}\Big\|\EE_{a\sim \mu(\cdot\given s),b\sim \vartheta}[Q^{*}_{s,u}(\cdot,a,b)]-\EE_{a\sim \mu(\cdot\given s),b\sim \vartheta}[Q^{*}_{s,u'}(\cdot,a,b)]\Big\|_\infty\\
	&\quad \leq\big\|Q^*_{s,u}-Q^*_{s,u'}\big\|_\infty\leq |u-u'|,
	\$
	which proves the first inequality.  
	
	For the second one, recalling that the best-response policy of $\mu$ under $\cG_{s,u}$ being $\nu_{s,u}(\mu)$, we have
	\#
	&Q^{\mu,*}_{s,u}-Q^{\mu,*}_{s,u'}=\min_{\nu}Q^{\mu,\nu}_{s,u}-Q^{\mu,*}_{s,u'}=\min_{\nu}\big(I-\gamma P^{\mu,\nu}_{s,u}\big)^{-1}r_{s,u}-Q^{\mu,*}_{s,u'}\label{equ:robust_mis2_1}\\
	&\quad\leq \big(I-\gamma P^{\mu,\nu_{s,u'}(\mu)}_{s,u}\big)^{-1}r_{s,u}-\big(I-\gamma P^{\mu,\nu_{s,u'}(\mu)}_{s,u'}\big)^{-1}r_{s,u'}\label{equ:robust_mis2_2}\\
	&\quad = \big(I-\gamma P^{\mu,\nu_{s,u'}(\mu)}_{s,u}\big)^{-1}\big(r_{s,u}-r_{s,u'}\big)\leq \frac{\|r_{s,u}-r_{s,u'}\|_\infty}{1-\gamma}=|u-u'|,\label{equ:robust_mis2_4}
	\#
	where \eqref{equ:robust_mis2_1} uses the definition of a best-response value, \eqref{equ:robust_mis2_2} plugs in the best-response policy $\nu_{s,u'}(\mu)$, and \eqref{equ:robust_mis2_4} also uses the fact that the transition does not depend on the value $u$. A lower bound can be established by noticing that $Q^{\mu,*}_{s,u'}=\min_{\nu}Q^{\mu,\nu}_{s,u'}\leq Q^{\mu,\nu_{s,u}(\mu)}_{s,u'}$. 
	This proves $\|Q^{\mu,*}_{s,u}-Q^{\mu,*}_{s,u'}\|_\infty\leq |u-u'|$.  Furthermore, notice that
	\$
	&\big\|V^{\mu,*}_{s,u}-V^{\mu,*}_{s,u'}\big\|_\infty=\Big\|\min_{\vartheta\in\Delta(\cB)}\EE_{a\sim \mu(\cdot\given s),b\sim \vartheta}[Q^{\mu,*}_{s,u}(\cdot,a,b)]-\min_{\vartheta\in\Delta(\cB)}\EE_{a\sim \mu(\cdot\given s),b\sim \vartheta}[Q^{\mu,*}_{s,u'}(\cdot,a,b)]\Big\|_\infty\\
	&\qquad \leq \max_{\vartheta\in\Delta(\cB)}\Big\|\EE_{a\sim \mu(\cdot\given s),b\sim \vartheta}[Q^{\mu,*}_{s,u}(\cdot,a,b)]-\EE_{a\sim \mu(\cdot\given s),b\sim \vartheta}[Q^{\mu,*}_{s,u'}(\cdot,a,b)]\Big\|_\infty\\
	&\qquad \leq \big\|Q^{\mu,*}_{s,u}-Q^{\mu,*}_{s,u'}\big\|_\infty\leq |u-u'|,
	\$
	which proves the second inequality. 
	Similar arguments can also be used to establish the third and the fourth inequalities.  
	This completes the proof. 
\end{proof}

We are now ready to show the main result in this section.

\begin{lemma}\label{lemma:absorb_main_res_1}
	For any state $s$, joint action pair $(a,b)$, and a finite set $U_s$, 
 	with probability greater than $1-\delta$, we have
	\small
	\$
	\big|(P_{s,a,b}-\hat{P}_{s,a,b}) \hat{V}^*\big|&\leq \sqrt{\frac{2\log(4|U_s|/\delta)\cdot\Var_{P_{s,a,b}}(\hat{V}^*)}{N}}+\frac{2\log(4|U_s|/\delta)}{3(1-\gamma)N} \\
	&\qquad\quad+\min_{u\in U_s}\big|\hat{V}^*(s)-u\big|\cdot\Bigg(2+\sqrt{\frac{2\log(4|U_s|/\delta)}{N}}\Bigg)\\
	\big|(P_{s,a,b}-\hat{P}_{s,a,b}) \hat{V}^{\mu^*,*}\big|&\leq \sqrt{\frac{2\log(4|U_s|/\delta)\cdot\Var_{P_{s,a,b}}(\hat{V}^{\mu^*,*})}{N}}+\frac{2\log(4|U_s|/\delta)}{3(1-\gamma)N} \\
	&\qquad\quad+\min_{u\in U_s}\big|\hat{V}^{\mu^*,*}(s)-u\big|\cdot\Bigg(2+\sqrt{\frac{2\log(4|U_s|/\delta)}{N}}\Bigg)\\
	\big|(P_{s,a,b}-\hat{P}_{s,a,b}) \hat{V}^{*,\nu^*}\big|&\leq \sqrt{\frac{2\log(4|U_s|/\delta)\cdot\Var_{P_{s,a,b}}(\hat{V}^{*,\nu^*})}{N}}+\frac{2\log(4|U_s|/\delta)}{3(1-\gamma)N} \\
	&\qquad\quad+\min_{u\in U_s}\big|\hat{V}^{*,\nu^*}(s)-u\big|\cdot\Bigg(2+\sqrt{\frac{2\log(4|U_s|/\delta)}{N}}\Bigg)\\
	\big|(P_{s,a,b}-\hat{P}_{s,a,b}) \hat{V}^{\mu^*,\nu^*}\big|&\leq \sqrt{\frac{2\log(4|U_s|/\delta)\cdot\Var_{P_{s,a,b}}(\hat{V}^{\mu^*,\nu^*})}{N}}+\frac{2\log(4|U_s|/\delta)}{3(1-\gamma)N} \\
	&\qquad\quad+\min_{u\in U_s}\big|\hat{V}^{\mu^*,\nu^*}(s)-u\big|\cdot\Bigg(2+\sqrt{\frac{2\log(4|U_s|/\delta)}{N}}\Bigg).
	\$ 
	\normalsize
	Moreover, recalling that $(\hat{\mu}_{s,u},\hat{\nu}_{s,u})$ is the output of the \textbf{Planning Oracle} using $\hat{\cG}_{s,u}$, we have
	\small
	\$
	\big|(P_{s,a,b}-\hat{P}_{s,a,b}) {V}^{\hat{\mu},*}\big|&\leq \sqrt{\frac{2\log(4|U_s|/\delta)\cdot\Var_{P_{s,a,b}}({V}^{\hat{\mu},*})}{N}}+\frac{2\log(4|U_s|/\delta)}{3(1-\gamma)N} 
\\
	&\qquad\quad
	+\min_{u\in U_s}\big\|{V}^{\hat{\mu},*}-{V}^{\hat{\mu}_{s,u},*}\big\|_\infty\Bigg(2+\sqrt{\frac{2\log(4|U_s|/\delta)}{N}}\Bigg),
	\$
	\$
	\big|(P_{s,a,b}-\hat{P}_{s,a,b}) {V}^{*,\hat{\nu}}\big|&\leq \sqrt{\frac{2\log(4|U_s|/\delta)\cdot\Var_{P_{s,a,b}}({V}^{*,\hat{\nu}})}{N}}+\frac{2\log(4|U_s|/\delta)}{3(1-\gamma)N} 
\\
	&\qquad\quad
	+\min_{u\in U_s}\big\|{V}^{*,\hat{\nu}}-{V}^{*,\hat{\nu}_{s,u}}\big\|_\infty\Bigg(2+\sqrt{\frac{2\log(4|U_s|/\delta)}{N}}\Bigg). 
	\$
	\normalsize
\end{lemma}  
\begin{proof}
	First, for all $u\in U_s$ and with probability greater than $1-\delta$,   we have
	\small
	\#
	&\big|(P_{s,a,b}-\hat{P}_{s,a,b}) \hat{V}^*\big|=\big|(P_{s,a,b}-\hat{P}_{s,a,b}) (\hat{V}^*-\hat{V}^*_{s,u}+\hat{V}^*_{s,u})\big|\notag\\
	&\quad\leq \big|(P_{s,a,b}-\hat{P}_{s,a,b}) (\hat{V}^*-\hat{V}^*_{s,u})\big|+\big|(P_{s,a,b}-\hat{P}_{s,a,b}) \hat{V}^*_{s,u}\big|\label{equ:absorb_main_res_1_1}\\
	&\quad\leq 2\cdot\big\|\hat{V}^*-\hat{V}^*_{s,u}\big\|_\infty+\big|(P_{s,a,b}-\hat{P}_{s,a,b}) \hat{V}^*_{s,u}\big|\label{equ:absorb_main_res_1_2}\\
	&\quad\leq 2\cdot\big\|\hat{V}^*-\hat{V}^*_{s,u}\big\|_\infty+\sqrt{\frac{2\log(4|U_s|/\delta)\cdot\Var_{P_{s,a,b}}(\hat{V}^*_{s,u})}{N}}+\frac{2\log(4|U_s|/\delta)}{3(1-\gamma)N} \label{equ:absorb_main_res_1_3}\\
	&\quad\leq \big\|\hat{V}^*-\hat{V}^*_{s,u}\big\|_\infty\Bigg(2+\sqrt{\frac{2\log(4|U_s|/\delta)}{N}}\Bigg)+\sqrt{\frac{2\log(4|U_s|/\delta)\cdot\Var_{P_{s,a,b}}(\hat{V}^*)}{N}}+\frac{2\log(4|U_s|/\delta)}{3(1-\gamma)N} \label{equ:absorb_main_res_1_4}
	\#
	\normalsize
	where \eqref{equ:absorb_main_res_1_1}-\eqref{equ:absorb_main_res_1_2}  use triangle inequality,  \eqref{equ:absorb_main_res_1_3} is due to Lemma \ref{lemma:bernstein_inq}, and \eqref{equ:absorb_main_res_1_4} uses the facts that 
	$\sqrt{\Var_{P_{s,a,b}}(X+Y)}\leq \sqrt{\Var_{P_{s,a,b}}(X)}+\sqrt{\Var_{P_{s,a,b}}(Y)}$, and $\sqrt{\Var_{P_{s,a,b}}(X)}\leq \|X\|_\infty$. Moreover, by Lemmas  \ref{lemma:helper_u_choose_1} and \ref{lemma:helper_u_choose_2}, we obtain that
	\$
	\big\|\hat{V}^*-\hat{V}^*_{s,u}\big\|_\infty=\big\|\hat{V}^*_{s,\hat{V}^*(s)}-\hat{V}^*_{s,u}\big\|_\infty\leq \big|\hat{V}^*(s)-u\big|,
	\$
	which, combined with \eqref{equ:absorb_main_res_1_4} and taken minimization over all $u\in U_s$, yields the first inequality. 
Proofs for the remaining inequalities are analogous, except that for the last two, the norms $\|{V}^{\hat{\mu},*}-{V}^{\hat{\mu}_{s,u},*}\|_\infty$ and $\|{V}^{*,\hat{\nu}}-{V}^{*,\hat{\nu}_{s,u}}\|_\infty$ are kept and not further bounded.  
\end{proof}

Next we establish the important result that characterizes the errors $|(P-\hat{P})\hat{V}^*|$, $|(P-\hat{P})\hat{V}^{\mu^*,*}|$,  $|(P-\hat{P})\hat{V}^{*,\nu^*}|$, and  $|(P-\hat{P})\hat{V}^{\mu^*,\nu^*}|$, which could not have  been  handled without the arguments above, due to the dependence between $\hat{P}$ and $\hat{V}^*$ (and also $\hat{V}^{\mu^*,*}$, $\hat{V}^{*,\nu^*}$, and $\hat{V}^{\mu^*,\nu^*}$).

\begin{lemma}\label{lemma:absorb_main_res_2}
	For any $\delta\in(0,1]$, with probability greater than $1-\delta$, it holds that 
	\$
	\big|(P-\hat{P})\hat{V}^*\big|&\leq \sqrt{\frac{2\log\big(16|\cS||\cA||\cB|/[(1-\gamma)^2\delta]\big)\cdot\Var_{P}(\hat{V}^*)}{N}}+\Delta_{\delta,N}'\\
	\big|(P-\hat{P})\hat{V}^{\mu^*,*}\big|&\leq \sqrt{\frac{2\log\big(16|\cS||\cA||\cB|/[(1-\gamma)^2\delta]\big)\cdot\Var_{P}(\hat{V}^{\mu^*,*})}{N}}+\Delta_{\delta,N}'\\
	\big|(P-\hat{P})\hat{V}^{*,\nu^*}\big|&\leq \sqrt{\frac{2\log\big(16|\cS||\cA||\cB|/[(1-\gamma)^2\delta]\big)\cdot\Var_{P}(\hat{V}^{*,\nu^*})}{N}}+\Delta_{\delta,N}'\\
	\big|(P-\hat{P})\hat{V}^{\mu^*,\nu^*}\big|&\leq \sqrt{\frac{2\log\big(16|\cS||\cA||\cB|/[(1-\gamma)^2\delta]\big)\cdot\Var_{P}(\hat{V}^{\mu^*,\nu^*})}{N}}+\Delta_{\delta,N}'
	\$
	where $\Delta_{\delta,N}'$ is defined as
	\$
	\Delta_{\delta,N}'=\sqrt{\frac{c\log\big(c|\cS||\cA||\cB|/[(1-\gamma)^2\delta]\big)}{N}}+\frac{c\log\big(c|\cS||\cA||\cB|/[(1-\gamma)^2\delta]\big)}{(1-\gamma)N},
	\$
	and $c$ is some absolute constant.  
\end{lemma}
\begin{proof}
	Let $U_s$ denote a set with evenly spaced elements in the interval $[V^*(s)-\Delta_{\delta/2,N},V^*(s)+\Delta_{\delta/2,N}]$, with $|U_s|=2/(1-\gamma)^2$, and $\Delta_{\delta,N}$ being defined in Lemma \ref{lemma:crude_bnd}. Lemma \ref{lemma:crude_bnd} shows that with probability greater than $1-\delta/2$, 
\#\label{equ:equ:lemma:absorb_main_res_2_trash_1}
	\hat{V}^*(s)\in\big[V^*(s)-\Delta_{\delta/2,N},~~V^*(s)+\Delta_{\delta/2,N}\big]
	\#  
	for all $s\in\cS$. Since each subinterval determined by $U_s$ is of  length $2\Delta_{\delta/2,N}/(|U_s|-1)$, and $\hat{V}^*(s)$ will fall into one of them, we know that
	\small
	\$
	\min_{u\in U_s}\big|\hat{V}^*(s)-u\big|\leq \frac{2\Delta_{\delta/2,N}}{|U_s|-1}=\frac{2\gamma}{(|U_s|-1)(1-\gamma)^2}\sqrt{\frac{2\log(4|\cS||\cA||\cB|/\delta)}{N}}\leq 2\gamma \sqrt{\frac{2\log(4|\cS||\cA||\cB|/\delta)}{N}},
	\$
	\normalsize
	where we have used the fact that $|U_s|\geq 1/(1-\gamma)^2+1$. 
 We then choose $\delta/2$ to be $\delta/(2|\cS||\cA||\cB|)$ in Lemma \ref{lemma:absorb_main_res_1}, 
  so that it holds for all states and joint actions with probability greater than $1-\delta/2$.  By substitution and noting that the two events in Lemmas \ref{lemma:crude_bnd} and \ref{lemma:absorb_main_res_1}  both  fail  with probability $\delta/2$, we obtain the first inequality by properly choosing the constant $c$. 
 Similarly, for the other two inequalities, note that Lemma \ref{lemma:crude_bnd} can be applied to show that $\hat{V}^{\mu^*,*}(s)$, $\hat{V}^{*,\nu^*}(s)$,  and $\hat{V}^{\mu^*,\nu^*}(s)$, all lie in the interval in \eqref{equ:equ:lemma:absorb_main_res_2_trash_1} (centered at $V^*(s)$). By similar arguments, the remaining three inequalities can be proved (note that Lemma \ref{lemma:absorb_main_res_1} can be applied to $\hat{V}^{\mu^*,*}(s)$, $\hat{V}^{*,\nu^*}(s)$, and $\hat{V}^{\mu^*,\nu^*}(s)$, as well). 
\end{proof}

Lastly, with a smooth \textbf{Planning Oracle}, see Definition \ref{def:smooth_oracle}, we can similarly establish the following error bounds on $|(P-\hat{P})V^{\hat{\mu},*}|$ and $|(P-\hat{P})V^{*,\hat{\nu}}|$, thanks to  Lemma \ref{lemma:absorb_main_res_1}.  

\begin{lemma}\label{lemma:V_hat_mu_nu_error}
	With a smooth \textbf{Planning Oracle} that has a smooth constant $C$ (see Definition \ref{def:smooth_oracle}),  for any $\delta\in(0,1]$, 
	 with probability greater than $1-\delta$, it holds that 
	\$
	\big|(P-\hat{P})V^{\hat{\mu},*}\big|&\leq \sqrt{\frac{2\log\big(8(C+1)|\cS||\cA||\cB|/[(1-\gamma)^4\delta]\big)\cdot\Var_{P}(V^{\hat{\mu},*})}{N}}+\Delta_{\delta,N}''\\
	\big|(P-\hat{P})V^{*,\hat{\nu}}\big|&\leq \sqrt{\frac{2\log\big(8(C+1)|\cS||\cA||\cB|/[(1-\gamma)^4\delta]\big)\cdot\Var_{P}(V^{*,\hat{\nu}})}{N}}+\Delta_{\delta,N}''
	\$
	where $\Delta_{\delta,N}''$ is defined as
	\$
	\Delta_{\delta,N}''=\sqrt{\frac{c\log\big(c(C+1)|\cS||\cA||\cB|/[(1-\gamma)^4\delta]\big)}{N}}+\frac{c\log\big(c(C+1)|\cS||\cA||\cB|/[(1-\gamma)^4\delta]\big)}{(1-\gamma)N},
	\$
	for some absolute constant $c$. 
\end{lemma}
\begin{proof}
Following the proof of Lemma \ref{lemma:absorb_main_res_2}, let $U_s$ denote a set with evenly spaced elements in the interval $[V^*(s)-\Delta_{\delta/2,N},V^*(s)+\Delta_{\delta/2,N}]$,  with $\Delta_{\delta,N}$ being defined in Lemma \ref{lemma:crude_bnd}. By Lemma \ref{lemma:crude_bnd}, we know that $\hat{V}^*(s)$ lies in this interval with probability greater than $1-\delta/2$, for all $s\in\cS$. Now we choose  $|U_s|=(C+1)/(1-\gamma)^4$, where $C$ is the smooth coefficient in Definition  \ref{def:smooth_oracle}. As $\hat{V}^*(s)$ will fall into one of the subintervals determined by $U_s$, we have
\#\label{proof_lemma_V_hat_mu_nu_error_trash0}
	\min_{u\in U_s}\big|\hat{V}^*(s)-u\big|\leq \frac{2\Delta_{\delta/2,N}}{|U_s|-1}
	\leq \frac{2\gamma(1-\gamma)^2}{C}\cdot  \sqrt{\frac{2\log(4|\cS||\cA||\cB|/\delta)}{N}},
\# 
which also uses the fact $|U_s|\geq C/(1-\gamma)^4+1$. 
Furthermore, by Definition  \ref{def:smooth_oracle} and the proof of Lemma \ref{lemma:helper_u_choose_2}, we have
\#\label{proof_lemma_V_hat_mu_nu_error_trash1}
&\big\|\hat{\mu}-\hat{\mu}_{s,u}\big\|_{TV}\leq C\cdot\|\hat{Q}^*-\hat{Q}^*_{s,u}\|_\infty
\leq C \cdot\big|\hat{V}^*(s)-u\big|. 
\#
On the other hand, we have
\#
&\big\|V^{\hat{\mu},*}-V^{\hat{\mu}_{s,u},*}\big\|_\infty\leq \max_{\vartheta\in\Delta(\cB)}\big\|\EE_{a\sim \hat{\mu}(\cdot\given s),b\sim \vartheta}[Q^{\hat{\mu},*}(\cdot,a,b)]-\EE_{a\sim \hat{\mu}_{s,u}(\cdot\given s),b\sim \vartheta}[{Q}^{\hat{\mu}_{s,u},*}(\cdot,a,b)]\big\|_\infty \notag\\
&\quad \leq \max_{\vartheta\in\Delta(\cB)}\big\|\EE_{a\sim \hat{\mu}(\cdot\given s),b\sim \vartheta}[Q^{\hat{\mu},*}(\cdot,a,b)]-\EE_{a\sim \hat{\mu}(\cdot\given s),b\sim \vartheta}[{Q}^{\hat{\mu}_{s,u},*}(\cdot,a,b)]\big\|_\infty\notag\\
&\qquad\qquad+\max_{\vartheta\in\Delta(\cB)}\big\|\EE_{a\sim \hat{\mu}(\cdot\given s),b\sim \vartheta}[{Q}^{\hat{\mu}_{s,u},*}(\cdot,a,b)]-\EE_{a\sim \hat{\mu}_{s,u}(\cdot\given s),b\sim \vartheta}[{Q}^{\hat{\mu}_{s,u},*}(\cdot,a,b)]\big\|_\infty\notag\\
&\quad \leq \big\|Q^{\hat{\mu},*}-{Q}^{\hat{\mu}_{s,u},*}\big\|_\infty+\big\|\hat{\mu}-\hat{\mu}_{s,u}\big\|_{TV}\cdot \big\|{Q}^{\hat{\mu}_{s,u},*}\big\|_\infty\label{proof_lemma_V_hat_mu_nu_error_trash2}\\
&\quad \leq \gamma \big\|V^{\hat{\mu},*}-{V}^{\hat{\mu}_{s,u},*}\big\|_\infty+\frac{C}{1-\gamma} \cdot\big|\hat{V}^*(s)-u\big|,\label{proof_lemma_V_hat_mu_nu_error_trash3}
\#
where \eqref{proof_lemma_V_hat_mu_nu_error_trash2} uses H\"{o}lder's inequality, and \eqref{proof_lemma_V_hat_mu_nu_error_trash3} follows by expanding the Q-value functions, using \eqref{proof_lemma_V_hat_mu_nu_error_trash1}, and noticing that $\|{Q}^{\hat{\mu}_{s,u},*}\|_\infty\leq 1/(1-\gamma)$. 
Combining \eqref{proof_lemma_V_hat_mu_nu_error_trash3} and \eqref{proof_lemma_V_hat_mu_nu_error_trash0}, and taking $\min$ over $u\in U_s$, we have
\$
\min_{u\in U_s}\big\|V^{\hat{\mu},*}-V^{\hat{\mu}_{s,u},*}\big\|_\infty\leq \frac{C}{(1-\gamma)^2} \cdot\min_{u\in U_s}\big|\hat{V}^*(s)-u\big|\leq 2\gamma \cdot  \sqrt{\frac{2\log(4|\cS||\cA||\cB|/\delta)}{N}}.
\$
The rest of the proof follows the arguments of Lemma \ref{lemma:absorb_main_res_2}, 
which combines the last two inequalities in Lemma \ref{lemma:absorb_main_res_1} to 
obtain the desired bound. Note that the absolute constant here might be  different from that in Lemma \ref{lemma:absorb_main_res_2}. The proof for the second inequality is analogous. 
\end{proof}

\kzedit{Note that compared to Lemma \ref{lemma:absorb_main_res_2}, Lemma \ref{lemma:V_hat_mu_nu_error} has to additionally deal with the interdependence between $\hat{P}$ and $V^{\hat{\mu},*}$ (as well as that between $\hat{P}$ and $V^{*,\hat{\nu}}$). What can be guaranteed before, in the absorbing MGs, is that the {\it value} function can be controlled to be close to that in the original MG (see Lemmas \ref{lemma:helper_u_choose_1} and \ref{lemma:helper_u_choose_2}, and the proof of Lemma \ref{lemma:absorb_main_res_1}). However, in general, it is unclear how much the NE {\it policy} changes, as well as how much the {\it best-response value} in the original true MG changes. This calls for some {\it stability} of the NE policy, and was made possible due to the smoothness of our \textbf{Planning Oracle} (see \eqref{proof_lemma_V_hat_mu_nu_error_trash1}-\eqref{proof_lemma_V_hat_mu_nu_error_trash3}). Lemma \ref{lemma:V_hat_mu_nu_error} will play an important role in obtaining the near-optimal sample complexity in Theorem \ref{thm:main_res_2} (see \S\ref{sec:proof_main_2}).}   

\subsection{Proof of Theorem \ref{thm:main_res}}\label{sec:proof_main}

We are now ready to prove   Theorem \ref{thm:main_res}.  To this end, we first establish the following lemma. 

\begin{lemma}\label{lemma:proof_main_res_Q_deviation}
	For any policy pair $(\hat{\mu},\hat{\nu})$ that satisfies the condition in Theorem \ref{thm:main_res}, there exists some absolute constant $c$ such that 
	\$
	\big\|Q^{\hat{\mu},\hat{\nu}}-\hat{Q}^{\hat{\mu},\hat{\nu}}\big\|_\infty &\leq \frac{\gamma}{1-\alpha_{\delta,N}}\Bigg(\sqrt{\frac{c\log(c|\cS||\cA||\cB|/[(1-\gamma)^2\delta])}{(1-\gamma)^3 N}}+\frac{c\log(c|\cS||\cA||\cB|/[(1-\gamma)^2\delta])}{(1-\gamma)^2N}\Bigg)\\
	&\qquad +\frac{1}{1-\alpha_{\delta,N}}\cdot \frac{\gamma\epsilon_{opt}}{(1-\gamma)}\Bigg(1+\sqrt{\frac{\log(c|\cS||\cA||\cB|/[(1-\gamma)^2\delta])}{N}}\Bigg)\\
	\big\|Q^{*}-\hat{Q}^{\mu^*,*}\big\|_\infty &\leq \frac{\gamma}{1-\alpha_{\delta,N}}\Bigg(\sqrt{\frac{c\log(c|\cS||\cA||\cB|/[(1-\gamma)^2\delta])}{(1-\gamma)^3 N}}+\frac{c\log(c|\cS||\cA||\cB|/[(1-\gamma)^2\delta])}{(1-\gamma)^2N}\Bigg)\\
	\\
	\big\|Q^{*}-\hat{Q}^{*,\nu^*}\big\|_\infty &\leq \frac{\gamma}{1-\alpha_{\delta,N}}\Bigg(\sqrt{\frac{c\log(c|\cS||\cA||\cB|/[(1-\gamma)^2\delta])}{(1-\gamma)^3 N}}+\frac{c\log(c|\cS||\cA||\cB|/[(1-\gamma)^2\delta])}{(1-\gamma)^2N}\Bigg),
	\$
	where 
	\$
	\alpha_{\delta,N}=\frac{\gamma}{1-\gamma}\sqrt{\frac{2\log(16|\cS||\cA||\cB|/[(1-\gamma)^2\delta])}{N}}. 
	\$
\end{lemma}
\begin{proof}
	Note that
	\#
	&\|Q^{\hat{\mu},\hat{\nu}}-\hat{Q}^{\hat{\mu},\hat{\nu}}\|_\infty =\gamma\big\|(I-\gamma P^{\hat{\mu},\hat{\nu}})^{-1}(P-\hat{P})\hat{V}^{\hat{\mu},\hat{\nu}}\big\|_\infty\label{equ:lemma:proof_main_res_Q_deviation_trash_1}\\
	&\quad\leq \gamma\big\|(I-\gamma P^{\hat{\mu},\hat{\nu}})^{-1}(P-\hat{P})\hat{V}^{*}\big\|_\infty+\gamma\big\|(I-\gamma P^{\hat{\mu},\hat{\nu}})^{-1}(P-\hat{P})(\hat{V}^{\hat{\mu},\hat{\nu}}-\hat{V}^{*})\big\|_\infty\label{equ:lemma:proof_main_res_Q_deviation_trash_2}\\
	&\quad\leq \gamma\big\|(I-\gamma P^{\hat{\mu},\hat{\nu}})^{-1}\big|(P-\hat{P})\hat{V}^{*}\big|\big\|_\infty+\frac{2\gamma\epsilon_{opt}}{1-\gamma},\label{equ:lemma:proof_main_res_Q_deviation_trash_4}
	\#
	where \eqref{equ:lemma:proof_main_res_Q_deviation_trash_1} is due to Lemma \ref{lemma:comp_wise_bnd};  \eqref{equ:lemma:proof_main_res_Q_deviation_trash_2} uses triangle inequality; 
	and  \eqref{equ:lemma:proof_main_res_Q_deviation_trash_4} is due to the non-negativeness  of the entries in $(I-\gamma P^{\hat{\mu},\hat{\nu}})^{-1}$, the sub-optimality of $(\hat{\mu},\hat{\nu})$, and Lemma \ref{lemma:inf_norm_error}. Since the first term in \eqref{equ:lemma:proof_main_res_Q_deviation_trash_4} can be bounded using Lemma \ref{lemma:absorb_main_res_2}, we have
	\small
	\#
	&\|Q^{\hat{\mu},\hat{\nu}}-\hat{Q}^{\hat{\mu},\hat{\nu}}\|_\infty \leq \gamma\sqrt{\frac{2\log\big(16|\cS||\cA||\cB|/[(1-\gamma)^2\delta]\big)}{N}}\Big\|(I-\gamma P^{\hat{\mu},\hat{\nu}})^{-1}\sqrt{\Var_{P}(\hat{V}^*)}\Big\|_\infty+\frac{\gamma\Delta_{\delta,N}'}{1-\gamma}+\frac{2\gamma\epsilon_{opt}}{1-\gamma}\notag\\
	&\leq \gamma\sqrt{\frac{2\log\big(16|\cS||\cA||\cB|/[(1-\gamma)^2\delta]\big)}{N}}\Big\|(I-\gamma P^{\hat{\mu},\hat{\nu}})^{-1}\Big(\sqrt{\Var_{P}({V}^{\hat{\mu},\hat{\nu}})}+\sqrt{\Var_{P}({V}^{\hat{\mu},\hat{\nu}}-\hat{V}^{\hat{\mu},\hat{\nu}})}\Big)\Big\|_\infty\notag\\
	&\qquad +\gamma\sqrt{\frac{2\log\big(16|\cS||\cA||\cB|/[(1-\gamma)^2\delta]\big)}{N}}\Big\|(I-\gamma P^{\hat{\mu},\hat{\nu}})^{-1}\Big(\sqrt{\Var_{P}(\hat{V}^{\hat{\mu},\hat{\nu}}-\hat{V}^{*})}\Big)\Big\|_\infty+\frac{\gamma\Delta_{\delta,N}'}{1-\gamma}+\frac{2\gamma\epsilon_{opt}}{1-\gamma}\label{equ:lemma:proof_main_res_Q_deviation_trash_2_2}\\
	&\leq \gamma\sqrt{\frac{2\log\big(16|\cS||\cA||\cB|/[(1-\gamma)^2\delta]\big)}{N}}\Bigg(\sqrt{\frac{2}{(1-\gamma)^3}}+\frac{{\|V^{\hat{\mu},\hat{\nu}}-\hat{V}^{\hat{\mu},\hat{\nu}}\|_\infty}}{1-\gamma}+\frac{\epsilon_{opt}}{1-\gamma}\Bigg)+\frac{\gamma\Delta_{\delta,N}'}{1-\gamma}+\frac{2\gamma\epsilon_{opt}}{1-\gamma}\label{equ:lemma:proof_main_res_Q_deviation_trash_3_2}
	\#
	\#
	&\quad\leq \gamma\sqrt{\frac{2\log\big(16|\cS||\cA||\cB|/[(1-\gamma)^2\delta]\big)}{N}}\Bigg(\sqrt{\frac{2}{(1-\gamma)^3}}+\frac{{\|Q^{\hat{\mu},\hat{\nu}}-\hat{Q}^{\hat{\mu},\hat{\nu}}\|_\infty}}{1-\gamma}+\frac{\epsilon_{opt}}{1-\gamma}\Bigg)+\frac{\gamma\Delta_{\delta,N}'}{1-\gamma}+\frac{2\gamma\epsilon_{opt}}{1-\gamma}\label{equ:lemma:proof_main_res_Q_deviation_trash_4_2}\\
	&\quad=\gamma\sqrt{\frac{2\log\big(16|\cS||\cA||\cB|/[(1-\gamma)^2\delta]\big)}{N}}\Bigg(\sqrt{\frac{2}{(1-\gamma)^3}}+\frac{{\|Q^{\hat{\mu},\hat{\nu}}-\hat{Q}^{\hat{\mu},\hat{\nu}}\|_\infty}}{1-\gamma}\Bigg)+\frac{\gamma\Delta_{\delta,N}'}{1-\gamma}\notag\\
	&\qquad\quad+\Bigg(2+\sqrt{\frac{2\log\big(16|\cS||\cA||\cB|/[(1-\gamma)^2\delta]\big)}{N}}\Bigg)\cdot\frac{\gamma\epsilon_{opt}}{1-\gamma},\label{equ:lemma:proof_main_res_Q_deviation_trash_5_2}
	\#
	\normalsize
	where \eqref{equ:lemma:proof_main_res_Q_deviation_trash_2_2} uses the fact that $\sqrt{\Var_{P}(X+Y)}\leq \sqrt{\Var_{P}(X)}+\sqrt{\Var_{P}(Y)}$;  \eqref{equ:lemma:proof_main_res_Q_deviation_trash_3_2} is due to Lemma  \ref{lemma:bellman_variance},  the fact that $\sqrt{\Var_{P}({V}^{\hat{\mu},\hat{\nu}}-\hat{V}^{\hat{\mu},\hat{\nu}})}\leq \|{V}^{\hat{\mu},\hat{\nu}}-\hat{V}^{\hat{\mu},\hat{\nu}}\|_\infty$, and $\|\hat{V}^{\hat{\mu},\hat{\nu}}-\hat{V}^{*}\|_\infty\leq \epsilon_{opt}$; \eqref{equ:lemma:proof_main_res_Q_deviation_trash_4_2} is due to $\|{V}^{\hat{\mu},\hat{\nu}}-\hat{V}^{\hat{\mu},\hat{\nu}}\|_\infty\leq \|{Q}^{\hat{\mu},\hat{\nu}}-\hat{Q}^{\hat{\mu},\hat{\nu}}\|_\infty$. Solving for $\|Q^{\hat{\mu},\hat{\nu}}-\hat{Q}^{\hat{\mu},\hat{\nu}}\|_\infty$ in \eqref{equ:lemma:proof_main_res_Q_deviation_trash_5_2} yields the desired inequality. 
	
	For the second inequality, by Lemma \ref{lemma:comp_wise_bnd}, we  first have
	\small
	\$
	\underbrace{\gamma (I-\gamma P^{\mu^*,\nu^*})^{-1}(P-\hat {P})\hat{V}^{\mu^*,\nu^*}}_{Q^{\mu^*,\nu^*}-\hat{Q}^{\mu^*,\nu^*}}\leq Q^{*}-\hat{Q}^{\mu^*,*}=Q^{\mu^*,\nu^*}-\hat{Q}^{\mu^*,*}\leq \underbrace{\gamma (I-\gamma P^{\mu^*,\hat{\nu(\mu^*)}})^{-1}(P-\hat {P})\hat{V}^{\mu^*,*}}_{Q^{\mu^*,\hat{\nu(\mu^*)}}-\hat{Q}^{\mu^*,*}}.
	\$
	\normalsize
	Thus, we obtain that 
	\#
	&\big\|Q^{*}-\hat{Q}^{\mu^*,*}\big\|_\infty \leq \max\big\{\big\|Q^{\mu^*,\nu^*}-\hat{Q}^{\mu^*,\nu^*}\big\|_\infty,~~\big\|Q^{\mu^*,\hat{\nu(\mu^*)}}-\hat{Q}^{\mu^*,*}\big\|_\infty\big\}\label{equ:lemma:proof_main_res_Q_deviation2_trash_1}\\
	&\quad= \max\Big\{\gamma\big\|(I-\gamma P^{\mu^*,\nu^*})^{-1}(P-\hat {P})\hat{V}^{\mu^*,\nu^*}\big\|_\infty,~~\gamma\big\|(I-\gamma P^{\mu^*,\hat{\nu(\mu^*)}})^{-1}(P-\hat {P})\hat{V}^{\mu^*,*}\big\|_\infty\Big\}. \notag
	\#
	For the first term in the $\max$ operator above, by similar arguments from \eqref{equ:lemma:proof_main_res_Q_deviation_trash_2_2}-\eqref{equ:lemma:proof_main_res_Q_deviation_trash_5_2}, we have
	\#
	&\big\|Q^{\mu^*,\nu^*}-\hat{Q}^{\mu^*,\nu^*}\big\|_\infty=\gamma\big\|(I-\gamma P^{\mu^*,\nu^*})^{-1}(P-\hat {P})\hat{V}^{\mu^*,\nu^*}\big\|_\infty \notag\\
	&\quad \leq \gamma\sqrt{\frac{2\log\big(16|\cS||\cA||\cB|/[(1-\gamma)^2\delta]\big)}{N}}\Big\|(I-\gamma P^{\mu^*,\nu^*})^{-1}\sqrt{\Var_{P}(\hat{V}^{\mu^*,\nu^*})}\Big\|_\infty+\frac{\gamma\Delta_{\delta,N}'}{1-\gamma}\label{equ:lemma:proof_main_res_Q_deviation2_trash_2}\\
	&\quad \leq\gamma\sqrt{\frac{2\log\big(16|\cS||\cA||\cB|/[(1-\gamma)^2\delta]\big)}{N}}\Big\|(I-\gamma P^{\mu^*,\nu^*})^{-1}\sqrt{\Var_{P}({V}^{\mu^*,\nu^*}-\hat{V}^{\mu^*,\nu^*})}\Big\|_\infty\notag\\
	&\qquad +\gamma\sqrt{\frac{2\log\big(16|\cS||\cA||\cB|/[(1-\gamma)^2\delta]\big)}{N}}\Big\|(I-\gamma P^{\mu^*,\nu^*})^{-1}\sqrt{\Var_{P}({V}^{\mu^*,\nu^*})}\Big\|_\infty+\frac{\gamma\Delta_{\delta,N}'}{1-\gamma}\label{equ:lemma:proof_main_res_Q_deviation2_trash_3}\\
	&\quad \leq\gamma\sqrt{\frac{2\log\big(16|\cS||\cA||\cB|/[(1-\gamma)^2\delta]\big)}{N}}\cdot\frac{\big\|{Q}^{\mu^*,\nu^*}-\hat{Q}^{\mu^*,\nu^*}\big\|_\infty}{1-\gamma}\notag\\
	&\qquad +\gamma\sqrt{\frac{2\log\big(16|\cS||\cA||\cB|/[(1-\gamma)^2\delta]\big)}{N}}\cdot\sqrt{\frac{2}{(1-\gamma)^3}}+\frac{\gamma\Delta_{\delta,N}'}{1-\gamma}\label{equ:lemma:proof_main_res_Q_deviation2_trash_4},
	\# 
	where \eqref{equ:lemma:proof_main_res_Q_deviation2_trash_2} is due to Lemma  \ref{lemma:absorb_main_res_2}, \eqref{equ:lemma:proof_main_res_Q_deviation2_trash_3} uses triangle inequality, and \eqref{equ:lemma:proof_main_res_Q_deviation2_trash_4} uses  Lemma \ref{lemma:bellman_variance}. Solving for $\big\|{Q}^{\mu^*,\nu^*}-\hat{Q}^{\mu^*,\nu^*}\big\|_\infty$ gives the bound for it. 
	
	Similarly, the second term in   the $\max$ operator in \eqref{equ:lemma:proof_main_res_Q_deviation2_trash_1} can be bounded by
	\#
	\big\|Q^{\mu^*,\hat{\nu(\mu^*)}}-\hat{Q}^{\mu^*,*}\big\|_\infty &\leq\gamma\sqrt{\frac{2\log\big(16|\cS||\cA||\cB|/[(1-\gamma)^2\delta]\big)}{N}}\cdot\frac{\big\|Q^{\mu^*,\hat{\nu(\mu^*)}}-\hat{Q}^{\mu^*,*}\big\|_\infty}{1-\gamma}\notag\\
	&\quad +\gamma\sqrt{\frac{2\log\big(16|\cS||\cA||\cB|/[(1-\gamma)^2\delta]\big)}{N}}\cdot\sqrt{\frac{2}{(1-\gamma)^3}}+\frac{\gamma\Delta_{\delta,N}'}{1-\gamma}\label{equ:lemma:proof_main_res_Q_deviation2_trash_4},
	\#
	which can be solved to obtain a bound for $\big\|Q^{\mu^*,\hat{\nu(\mu^*)}}-\hat{Q}^{\mu^*,*}\big\|_\infty$. Combining the two bounds and \eqref{equ:lemma:proof_main_res_Q_deviation2_trash_1}, we prove the second inequality in the lemma. The proof for the third inequality is analogous. 
\end{proof}

With Lemma \ref{lemma:proof_main_res_Q_deviation} in hand, we are ready to prove Theorem \ref{thm:main_res}. Note that the condition on $N$ in   Theorem \ref{thm:main_res} makes $\alpha_{\delta,N}<1/2$. Thus, by \eqref{equ:error_decomp_1}-\eqref{equ:error_decomp_2} in Lemma \ref{lemma:comp_wise_bnd} with $(\mu,\nu)$ being replaced by $(\hat{\mu},\hat{\nu})$, we have
\small
\$
-\|Q^{\hat{\mu},\hat{\nu}}-\hat{Q}^{\hat{\mu},\hat{\nu}}\|_{\infty}-\gamma\epsilon_{opt}-\|\hat{Q}^{\mu^*,*}-Q^*\|_\infty\leq Q^{\hat{\mu},\hat{\nu}}-Q^*\leq 
\|Q^{\hat{\mu},\hat{\nu}}-\hat{Q}^{\hat{\mu},\hat{\nu}}\|_{\infty}+\gamma\epsilon_{opt}+\|\hat{Q}^{*,\nu^*}-Q^*\|_\infty,
\$
\normalsize
where we use 
\$
\|\hat{Q}^{\hat{\mu},\hat{\nu}}-\hat{Q}^*\|_\infty=\gamma\|P\hat{V}^{\hat{\mu},\hat{\nu}}-P\hat{V}^*\|_\infty\leq \gamma \|\hat{V}^{\hat{\mu},\hat{\nu}}-\hat{V}^*\|_\infty\leq \gamma\epsilon_{opt}.
\$ 
Substituting in the bounds of $\|Q^{\hat{\mu},\hat{\nu}}-\hat{Q}^{\hat{\mu},\hat{\nu}}\|_\infty$, $\|Q^{*}-\hat{Q}^{\mu^*,*}\|_\infty$, and $\|Q^{*}-\hat{Q}^{*,\nu^*}\|_\infty$ in Lemma \ref{lemma:proof_main_res_Q_deviation}, we arrive at the final bound for $\|Q^{\hat{\mu},\hat{\nu}}-Q^*\|_\infty$:
\small
\$
\|Q^{\hat{\mu},\hat{\nu}}-Q^*\|_\infty\leq 4\gamma \Bigg(\sqrt{\frac{c\log(c|\cS||\cA||\cB|/[(1-\gamma)^2\delta])}{(1-\gamma)^3 N}}+\frac{c\log(c|\cS||\cA||\cB|/[(1-\gamma)^2\delta])}{(1-\gamma)^2N}\Bigg)+\frac{4\gamma\epsilon_{opt}}{1-\gamma}+\gamma\epsilon_{opt}.
\$
\normalsize
With a certain choice of $c$, we have $\|Q^{\hat{\mu},\hat{\nu}}-Q^*\|_\infty\leq 2\epsilon/3+5\gamma\epsilon_{opt}/(1-\gamma)$. 

For the last argument in Theorem \ref{thm:main_res}, by triangle inequality, with the same constant $c$ used above, we have
\$
\|\hat{Q}^{\hat{\mu},\hat{\nu}}-Q^*\|_\infty \leq \|{Q}^{\hat{\mu},\hat{\nu}}-Q^*\|_\infty+\|\hat{Q}^{\hat{\mu},\hat{\nu}}-{Q}^{\hat{\mu},\hat{\nu}}\|_\infty\leq \epsilon+\frac{9\gamma\epsilon_{opt}}{1-\gamma},
\$
which completes the proof. 
\hfill\QED

\subsection{Proof of Corollary \ref{coro:main_res_coro}}\label{sec:proof_main_coro}
We now prove Corollary \ref{coro:main_res_coro}, based on Theorem \ref{thm:main_res}. 
For any state $s$, we have
\#
&V^*(s)-V^{\tilde{\mu},*}(s)=\min_{\vartheta\in\Delta(\cB)}\EE_{a\sim{\mu}^*(\cdot\given s),b\sim \vartheta}\big[Q^*(s,a,b)\big]-\min_{\vartheta\in\Delta(\cB)}\EE_{a\sim\tilde{\mu}(\cdot\given s),b\sim \vartheta}\big[Q^{\tilde{\mu},*}(s,a,b)\big]\notag\\
&\quad =\min_{\vartheta\in\Delta(\cB)}\EE_{a\sim{\mu}^*(\cdot\given s),b\sim \vartheta}\big[Q^*(s,a,b)\big]-\min_{\vartheta\in\Delta(\cB)}\EE_{a\sim\tilde{\mu}(\cdot\given s),b\sim \vartheta}\big[Q^*(s,a,b)\big]\notag\\
&\quad\qquad +\min_{\vartheta\in\Delta(\cB)}\EE_{a\sim\tilde{\mu}(\cdot\given s),b\sim \vartheta}\big[Q^*(s,a,b)\big]-\min_{\vartheta\in\Delta(\cB)}\EE_{a\sim\tilde{\mu}(\cdot\given s),b\sim \vartheta}\big[Q^{\tilde{\mu},*}(s,a,b)\big]\notag\\
&\quad \leq \min_{\vartheta\in\Delta(\cB)}\EE_{a\sim{\mu}^*(\cdot\given s),b\sim \vartheta}\big[Q^*(s,a,b)\big]-\min_{\vartheta\in\Delta(\cB)}\EE_{a\sim\tilde{\mu}(\cdot\given s),b\sim \vartheta}\big[Q^*(s,a,b)\big]+\gamma \|V^*-V^{\tilde{\mu},*}\|_\infty \label{equ:coro_transition_trash_0}\\
&\quad \leq \min_{\vartheta\in\Delta(\cB)}\EE_{a\sim{\mu}^*(\cdot\given s),b\sim \vartheta}\big[Q^*(s,a,b)\big]-\min_{\vartheta\in\Delta(\cB)}\EE_{a\sim{\mu}^*(\cdot\given s),b\sim \vartheta}\big[\hat{Q}^{\hat{\mu},\hat{\nu}}(s,a,b)\big]\notag\\
&\qquad+\min_{\vartheta\in\Delta(\cB)}\EE_{a\sim\tilde{\mu}(\cdot\given s),b\sim \vartheta}\big[\hat{Q}^{\hat{\mu},\hat{\nu}}(s,a,b)\big]-\min_{\vartheta\in\Delta(\cB)}\EE_{a\sim\tilde{\mu}(\cdot\given s),b\sim \vartheta}\big[Q^*(s,a,b)\big]+\gamma \|V^*-V^{\tilde{\mu},*}\|_\infty \label{equ:coro_transition_trash_1}\\
&\quad \leq 2\big\|Q^*-\hat{Q}^{\hat{\mu},\hat{\nu}}\big\|_\infty+\gamma \|V^*-V^{\tilde{\mu},*}\|_\infty,\label{equ:coro_transition_trash_2}
\#
where \eqref{equ:coro_transition_trash_0} uses the fact that
\$
&\min_{\vartheta\in\Delta(\cB)}\EE_{a\sim\tilde{\mu}(\cdot\given s),b\sim \vartheta}\big[Q^*(s,a,b)\big]-\min_{\vartheta\in\Delta(\cB)}\EE_{a\sim\tilde{\mu}(\cdot\given s),b\sim \vartheta}\big[Q^{\tilde{\mu},*}(s,a,b)\big]\notag\\
&\quad\leq \max_{\vartheta\in\Delta(\cB)}\bigg|\EE_{a\sim\tilde{\mu}(\cdot\given s),b\sim \vartheta}\big[Q^*(s,a,b)\big]-\EE_{a\sim\tilde{\mu}(\cdot\given s),b\sim \vartheta}\big[Q^{\tilde{\mu},*}(s,a,b)\big]\bigg|\leq \gamma \|V^*-V^{\tilde \mu,*}\|_\infty,
\$
and \eqref{equ:coro_transition_trash_1}  is due to the fact that 
\$
-\min_{\vartheta\in\Delta(\cB)}\EE_{a\sim{\mu}^*(\cdot\given s),b\sim \vartheta}\big[\hat{Q}^{\hat{\mu},\hat{\nu}}(s,a,b)\big]+
\min_{\vartheta\in\Delta(\cB)}\EE_{a\sim\tilde{\mu}(\cdot\given s),b\sim \vartheta}\big[\hat{Q}^{\hat{\mu},\hat{\nu}}(s,a,b)\big]\geq 0,
\$
by definition of $\tilde \mu$. Hence,   \eqref{equ:coro_transition_trash_2}, together with Theorem \ref{thm:main_res}, implies that 
\#\label{equ:coro_transition_trash_3}
V^*-V^{\tilde{\mu},*}\leq \frac{2\big\|Q^*-\hat{Q}^{\hat{\mu},\hat{\nu}}\big\|_\infty}{1-\gamma}=\tilde\epsilon. 
\#
By similar arguments, we have 
\#\label{equ:coro_transition_trash_4}
V^{*,\tilde{\nu}}-V^*\leq \frac{2\big\|Q^*-\hat{Q}^{\hat{\mu},\hat{\nu}}\big\|_\infty}{1-\gamma}=\tilde\epsilon. 
\#
Combining \eqref{equ:coro_transition_trash_3} and \eqref{equ:coro_transition_trash_4} yields
\$
V^{\tilde{\mu},\tilde{\nu}}-V^{\tilde{\mu},*}\leq V^{*,\tilde{\nu}}-V^{\tilde{\mu},*}\leq 2\tilde \epsilon,\qquad 
V^{*,\tilde{\nu}}-V^{\tilde{\mu},\tilde{\nu}}\leq V^{*,\tilde{\nu}}-V^{\tilde{\mu},*}\leq 2\tilde \epsilon,
\$
which completes the proof. 
\hfill\QED

\subsection{Proof of Theorem \ref{thm:main_res_2}}\label{sec:proof_main_2}

We now prove the second main result, Theorem \ref{thm:main_res_2}. 
First, following the proof of Corollary \ref{coro:main_res_coro}, it suffices to prove that  $V^*-V^{\hat{\mu},*} \leq \tilde \epsilon$, $~V^{*,\hat{\nu}}-V^* \leq \tilde \epsilon$, since they together imply  
that $(\hat{\mu},\hat{\nu})$ is a  $2\tilde \epsilon$-Nash equilibrium. The following analysis is devoted to proving this argument.

The idea is similar to that presented in \S\ref{sec:proof_main}, i.e., we use the component-wise error decompositions in Lemma \ref{lemma:comp_wise_bnd},  but use \eqref{equ:error_decomp_5}-\eqref{equ:error_decomp_6} instead.    In particular, letting $\mu=\hat{\mu}$ and $\nu=\hat{\nu}$,  we have 
\#
V^{\hat{\mu},*}- V^*&\geq-\|Q^{\hat{\mu},*}-\hat{Q}^{\hat{\mu},*}\|_{\infty}-\epsilon_{opt}-\|\hat{Q}^{\mu^*,*}-Q^*\|_\infty\label{equ:main_res_2_comp_bnd_1}\\
	V^{*,\hat{\nu}}-V^*&\leq \|Q^{*,\hat{\nu}}-\hat{Q}^{*,\hat{\nu}}\|_{\infty}+\epsilon_{opt}+\|\hat{Q}^{*,\nu^*}-Q^*\|_\infty\label{equ:main_res_2_comp_bnd_2}. 
\#
Note that the bounds for $\|\hat{Q}^{\mu^*,*}-Q^*\|_\infty$  and $\|\hat{Q}^{*,\nu^*}-Q^*\|_\infty$ have already been established  in Lemma \ref{lemma:proof_main_res_Q_deviation} (without dependence on $\epsilon_{opt}$ and the \textbf{Planning Oracle}). It now suffices to bound $\|Q^{\hat{\mu},*}-\hat{Q}^{\hat{\mu},*}\|_{\infty}$ and $\|Q^{*,\hat{\nu}}-\hat{Q}^{*,\hat{\nu}}\|_{\infty}$. 
For the former term,  by Lemma \ref{lemma:comp_wise_bnd}, we  first have
	\$
	\underbrace{\gamma (I-\gamma \hat{P}^{\hat{\mu},\nu(\hat{\mu})})^{-1}(P-\hat {P})V^{\hat{\mu},\nu(\hat{\mu})}}_{Q^{\hat{\mu},*}-\hat{Q}^{\hat{\mu},\nu(\hat{\mu})}}\leq Q^{\hat{\mu},*}-\hat{Q}^{\hat{\mu},*}\leq \underbrace{\gamma (I-\gamma P^{\hat{\mu},\hat{\nu(\hat{\mu})}})^{-1}(P-\hat {P})\hat{V}^{\hat{\mu},\hat{\nu(\hat{\mu})}}}_{Q^{\hat{\mu},\hat{\nu(\hat{\mu})}}-\hat{Q}^{\hat{\mu},*}}.
	\$
	Thus, we know that 
	\#
	&\big\|Q^{\hat{\mu},*}-\hat{Q}^{\hat{\mu},*}\big\|_\infty \label{equ:lemma:proof_main_res_Q_deviation2_2_trash_1}\\
	&\leq\max\Big\{\gamma \big\|(I-\gamma P^{\hat{\mu},\hat{\nu(\hat{\mu})}})^{-1}(P-\hat {P})\hat{V}^{\hat{\mu},\hat{\nu(\hat{\mu})}}\big\|_\infty,~~\gamma \big\|(I-\gamma \hat{P}^{\hat{\mu},\nu(\hat{\mu})})^{-1}(P-\hat {P})V^{\hat{\mu},\nu(\hat{\mu})}\big\|_\infty\Big\}. \notag
	\#
	The first term in the $\max$ operator, where the policies in the pair $(\hat{\mu},\hat{\nu(\hat{\mu})})$ are both obtained from the empirical model $\hat{\cG}$, can be bounded similarly as that for $\|Q^{\hat{\mu},\hat{\nu}}-\hat{Q}^{\hat{\mu},\hat{\nu}}\|_\infty$ in Lemma \ref{lemma:proof_main_res_Q_deviation}. Specifically, following \eqref{equ:lemma:proof_main_res_Q_deviation_trash_1}-\eqref{equ:lemma:proof_main_res_Q_deviation_trash_4}, we have
	\#
	&\gamma\big\|(I-\gamma P^{\hat{\mu},\hat{\nu(\hat{\mu})}})^{-1}(P-\hat {P})\hat{V}^{\hat{\mu},*}\big\|_\infty\notag\\
	&\quad\leq \gamma\big\|(I-\gamma P^{\hat{\mu},\hat{\nu(\hat{\mu})}})^{-1}(P-\hat{P})\hat{V}^{*}\big\|_\infty+\gamma\big\|(I-\gamma P^{\hat{\mu},\hat{\nu(\hat{\mu})}})^{-1}(P-\hat{P}){(\hat{V}^{\hat{\mu},*}-\hat{V}^{*})}\big\|_\infty\label{equ:lemma:proof_main_res_Q_deviation_2_trash_2}\\
	&\quad\leq \gamma\big\|(I-\gamma P^{\hat{\mu},\hat{\nu(\hat{\mu})}})^{-1}\big|(P-\hat{P})\hat{V}^{*}\big|\big\|_\infty+\frac{2\gamma{\epsilon_{opt}}}{1-\gamma},\label{equ:lemma:proof_main_res_Q_deviation_2_trash_2_4}
	\#
	where \eqref{equ:lemma:proof_main_res_Q_deviation_2_trash_2} uses triangle inequality, and \eqref{equ:lemma:proof_main_res_Q_deviation_2_trash_2_4} is due to the optimization error of $\hat{\mu}$. 
	Then, to bound $\gamma\big\|(I-\gamma P^{\hat{\mu},\hat{\nu(\hat{\mu})}})^{-1}\big|(P-\hat{P})\hat{V}^{*}\big|\big\|_\infty$, the rest of the proof is analogous to the derivations in \eqref{equ:lemma:proof_main_res_Q_deviation_trash_2_2}-\eqref{equ:lemma:proof_main_res_Q_deviation_trash_5_2}, by replacing $\hat{\nu}$ therein by $\hat{\nu(\hat{\mu})}$, and bound $\|\hat{V}^{\hat{\mu},*}-\hat{V}^*\|_\infty$ by $\epsilon_{opt}$. 
	Solving for $\|Q^{\hat{\mu},\hat{\nu(\hat{\mu})}}-\hat{Q}^{\hat{\mu},*}\|_\infty$  yields the desired bound for the first term in the $\max$ in \eqref{equ:lemma:proof_main_res_Q_deviation2_2_trash_1}, namely, there exists some constant $c$ such that with probability greater than $1-\delta$,  
\#\label{equ:thm_main_res_2_final_1}	
&\big\|Q^{\hat{\mu},\hat{\nu(\hat{\mu})}}-\hat{Q}^{\hat{\mu},*}\big\|_\infty \notag\\
&\leq \frac{\gamma}{1-\alpha'_{\delta,N}}\Bigg(\sqrt{\frac{c\log(c(C+1)|\cS||\cA||\cB|/[(1-\gamma)^4\delta])}{(1-\gamma)^3 N}}+\frac{c\log(c(C+1)|\cS||\cA||\cB|/[(1-\gamma)^4\delta])}{(1-\gamma)^2N}\Bigg)\notag\\
	&\quad +\frac{1}{1-\alpha'_{\delta,N}}\cdot \frac{\gamma\epsilon_{opt}}{(1-\gamma)}\Bigg(1+\sqrt{\frac{\log(c(C+1)|\cS||\cA||\cB|/[(1-\gamma)^4\delta])}{N}}\Bigg),
\#
where $\alpha'_{\delta,N}$ is defined as
\$
	\alpha'_{\delta,N}=\frac{\gamma}{1-\gamma}\sqrt{\frac{2\log(8(C+1)|\cS||\cA||\cB|/[(1-\gamma)^4\delta])}{N}}. 
\$

For the second term in the $\max$ in \eqref{equ:lemma:proof_main_res_Q_deviation2_2_trash_1}, note that $\hat{\mu}$ is obtained from $\hat{\cG}$, while $\nu(\hat{\mu})$ is obtained from the true model $\cG$. \kzedit{Note that this mismatch is one key difference from the single-agent setting \citep{yang2019optimality} and the above proof for the first term.} 
By Lemma \ref{lemma:V_hat_mu_nu_error},   it holds that 
\small
\#
&\gamma\big\|(I-\gamma \hat{P}^{\hat{\mu},{\nu(\hat{\mu})}})^{-1}\big|(P-\hat{P}){V}^{\hat{\mu},*}\big|\big\|_\infty \notag\\
&\leq \gamma\sqrt{\frac{2\log\big(8(C+1)|\cS||\cA||\cB|/[(1-\gamma)^4\delta]\big)}{N}}\Big\|(I-\gamma \hat{P}^{\hat{\mu},{\nu(\hat{\mu})}})^{-1}\sqrt{\Var_{P}({V}^{\hat{\mu},*})}\Big\|_\infty+\frac{\gamma\Delta_{\delta,N}'}{1-\gamma}\notag\\
	&\leq \gamma\sqrt{\frac{2\log\big(8(C+1)|\cS||\cA||\cB|/[(1-\gamma)^4\delta]\big)}{N}}\cdot\bigg[\Big\|(I-\gamma \hat{P}^{\hat{\mu},{\nu(\hat{\mu})}})^{-1} \Big(\sqrt{\Var_{\hat{P}}(\hat{V}^{\hat{\mu},\nu(\hat{\mu})})}\label{equ:lemma:proof_main_res_Q_deviation_final_trash_2_2}\\
	&\quad +\sqrt{\Var_{\hat{P}}({V}^{\hat{\mu},*}-\hat{V}^{\hat{\mu},\nu(\hat{\mu})})}\Big)\Big\|_\infty+\Big\|(I-\gamma \hat{P}^{\hat{\mu},{\nu(\hat{\mu})}})^{-1} \Big|\sqrt{\Var_{P}({V}^{\hat{\mu},*})}-\sqrt{\Var_{\hat{P}}({V}^{\hat{\mu},*})}\Big| \Big\|_\infty\bigg]+\frac{\gamma\Delta_{\delta,N}'}{1-\gamma}\notag\\
	&\leq \gamma\sqrt{\frac{2\log\big(8(C+1)|\cS||\cA||\cB|/[(1-\gamma)^4\delta]\big)}{N}}\Bigg(\sqrt{\frac{2}{(1-\gamma)^3}}+\frac{{\|Q^{\hat{\mu},*}-\hat{Q}^{\hat{\mu},\nu(\hat{\mu})}\|_\infty}}{1-\gamma}\Bigg)+\frac{\gamma\Delta_{\delta,N}'}{1-\gamma}\label{equ:lemma:proof_main_res_Q_deviation_final_trash_4_2}\\
	&\quad +\frac{\gamma}{1-\gamma}\sqrt{\frac{2\log\big(8(C+1)|\cS||\cA||\cB|/[(1-\gamma)^4\delta]\big)}{N}}\Big\|
	\Big|\sqrt{\Var_{P}({V}^{\hat{\mu},*})}-\sqrt{\Var_{\hat{P}}({V}^{\hat{\mu},*})}\Big|
	\Big\|_\infty,\notag
\#	
\normalsize
where \eqref{equ:lemma:proof_main_res_Q_deviation_final_trash_2_2} uses the norm-like triangle-inequality property of $\sqrt{\Var_P(V)}$ and triangle inequality,   \eqref{equ:lemma:proof_main_res_Q_deviation_final_trash_4_2} is due to Lemma \ref{lemma:bellman_variance}, and the facts that $\sqrt{\Var_{P}(X)}\leq \|X\|_\infty$, 
 $\|V^{\hat{\mu},*}-\hat{V}^{\hat{\mu},\nu(\hat{\mu})}\|_\infty\leq \|Q^{\hat{\mu},*}-\hat{Q}^{\hat{\mu},\nu(\hat{\mu})}\|_\infty$, and Lemma \ref{lemma:inf_norm_error}.  Moreover, notice that
\small
\#
&\Big\|\Big|\sqrt{\Var_{P}({V}^{\hat{\mu},*})}-\sqrt{\Var_{\hat{P}}({V}^{\hat{\mu},*})}\Big|\Big\|_\infty \notag\\
&\quad\leq \Big\|\Big|\sqrt{\Var_{P}({V}^{\hat{\mu},*})}-\sqrt{\Var_{{P}}({V}^{*})}\Big|\Big\|_\infty+\Big\|\Big|\sqrt{\Var_{\hat{P}}({V}^{\hat{\mu},*})}-\sqrt{\Var_{\hat{P}}({V}^{*})}\Big|\Big\|_\infty\notag\\
&\qquad\quad +\Big\|\Big|\sqrt{\Var_{{P}}({V}^{*})}-\sqrt{\Var_{\hat{P}}({V}^{*})}\Big|\Big\|_\infty\label{equ:lemma:proof_main_res_Q_deviation_final_trash_4_2.5}\\
&\quad \leq  \Big\|\sqrt{\Var_{P}({V}^{\hat{\mu},*}-{V}^{*})}\Big\|_\infty+\Big\|\sqrt{\Var_{\hat{P}}({V}^{\hat{\mu},*}-{V}^{*})}\Big\|_\infty+\Big\|\sqrt{\Big|\Var_{{P}}({V}^{*})-\Var_{\hat{P}}({V}^{*})\Big|}\Big\|_\infty\label{equ:lemma:proof_main_res_Q_deviation_final_trash_4_3}\\
&\quad \leq  2\big\|{V}^{\hat{\mu},*}-{V}^{*}\big\|_\infty+\sqrt{\Big\|\Var_{{P}}({V}^{*})-\Var_{\hat{P}}({V}^{*})\Big\|_\infty},\label{equ:lemma:proof_main_res_Q_deviation_final_trash_4_4}
\#
\normalsize
where \eqref{equ:lemma:proof_main_res_Q_deviation_final_trash_4_2.5} uses triangle inequality, \eqref{equ:lemma:proof_main_res_Q_deviation_final_trash_4_3} uses the norm-like triangle inequality of $\sqrt{\Var_P(V)}$ and $\sqrt{\Var_{\hat{P}}(V)}$, and the fact $|\sqrt{X}-\sqrt{Y}|\leq \sqrt{|X-Y|}$ for $X,Y\geq 0$, and \eqref{equ:lemma:proof_main_res_Q_deviation_final_trash_4_4} uses $\sqrt{\Var_{P}(X)}\leq \|X\|_\infty$ and the definition of $\|\cdot\|_\infty$. In addition, we know that with probability at least $1-\delta$, 
\#\label{equ:lemma:proof_main_res_Q_deviation_final_trash_4_5}
&\Big\|\Var_{{P}}({V}^{*})-\Var_{\hat{P}}({V}^{*})\Big\|_\infty=\Big\|(P-\hat{P})(V^*)^2-\Big((PV^*)^2-(\hat{P}V^*)^2\Big)\Big\|_\infty\notag\\
&\quad \leq \Big\|(P-\hat{P})(V^*)^2\Big\|_\infty+\Big\|(PV^*)^2-(\hat{P}V^*)^2\Big\|_\infty\notag\\
&\quad \leq \frac{1}{(1-\gamma)^2}\sqrt{\frac{2\log(2|\cS||\cA||\cB|/\delta)}{N}}+\frac{2}{1-\gamma}\big\|(P-\hat{P})V^*\big\|_\infty\leq \frac{3}{(1-\gamma)^2}\sqrt{\frac{2\log(2|\cS||\cA||\cB|/\delta)}{N}},
\#
due to Hoeffding bound and $\|V^*\|_\infty\leq 1/(1-\gamma)$.  
Combining \eqref{equ:lemma:proof_main_res_Q_deviation_final_trash_4_2}, \eqref{equ:lemma:proof_main_res_Q_deviation_final_trash_4_4}, and \eqref{equ:lemma:proof_main_res_Q_deviation_final_trash_4_5} yields
\small
\$
&\Big\|Q^{\hat{\mu},*}-\hat{Q}^{\hat{\mu},\nu(\hat{\mu})}\Big\|_\infty \\
&\leq \gamma\sqrt{\frac{2\log\big(8(C+1)|\cS||\cA||\cB|/[(1-\gamma)^4\delta]\big)}{N}}\Bigg(\sqrt{\frac{2}{(1-\gamma)^3}}+\frac{{\|Q^{\hat{\mu},*}-\hat{Q}^{\hat{\mu},\nu(\hat{\mu})}\|_\infty}}{1-\gamma}\Bigg)+\frac{\gamma\Delta_{\delta,N}'}{1-\gamma}\\
	&\qquad+\frac{\gamma}{1-\gamma}\sqrt{\frac{2\log\big(8(C+1)|\cS||\cA||\cB|/[(1-\gamma)^4\delta]\big)}{N}}\Big(2\big\|{V}^{\hat{\mu},*}-{V}^{*}\big\|_\infty\\
	&\qquad+\sqrt{\frac{3}{(1-\gamma)^2}\sqrt{\frac{2\log(2|\cS||\cA||\cB|/\delta)}{N}}}
	\Big).
\$
\normalsize
Solving for $\|Q^{\hat{\mu},*}-\hat{Q}^{\hat{\mu},\nu(\hat{\mu})}\|_\infty$ further leads to 
\#\label{equ:thm_main_res_2_final_2}
&\big\|Q^{\hat{\mu},*}-\hat{Q}^{\hat{\mu},\nu(\hat{\mu})}\big\|_\infty \notag\\
&\leq \frac{\gamma}{1-\alpha'_{\delta,N}}\Bigg(\sqrt{\frac{c\log(c(C+1)|\cS||\cA||\cB|/[(1-\gamma)^4\delta])}{(1-\gamma)^3 N}}+\frac{c\log(c(C+1)|\cS||\cA||\cB|/[(1-\gamma)^4\delta])}{(1-\gamma)^2N}\Bigg)\notag\\
	&\quad +\frac{1}{1-\alpha'_{\delta,N}}\cdot \frac{\gamma}{1-\gamma}\sqrt{\frac{2\log\big(8(C+1)|\cS||\cA||\cB|/[(1-\gamma)^4\delta]\big)}{N}}\Big(2\big\|{V}^{\hat{\mu},*}-{V}^{*}\big\|_\infty\notag\\
	&\quad+\frac{1}{1-\gamma}\sqrt[4]{\frac{c\log(c(C+1)|\cS||\cA||\cB|/\delta)}{N}}\Big), 
\#
for some absolute constant $c$. 

Now we substitute \eqref{equ:thm_main_res_2_final_1} and \eqref{equ:thm_main_res_2_final_2} into \eqref{equ:lemma:proof_main_res_Q_deviation2_2_trash_1}, to complete the bound in \eqref{equ:main_res_2_comp_bnd_1}. If the first term in the $\max$ in \eqref{equ:lemma:proof_main_res_Q_deviation2_2_trash_1} is larger, and noticing  that the choice of $N$ in the theorem can make $\alpha'_{\delta,N}<1/5$,   
\eqref{equ:main_res_2_comp_bnd_1}, \eqref{equ:lemma:proof_main_res_Q_deviation2_2_trash_1}, 
\eqref{equ:thm_main_res_2_final_1},   and  Lemma \ref{lemma:proof_main_res_Q_deviation} together lead to 
\small
\#\label{equ:thm_main_res_2_final_3}
V^*-V^{\hat{\mu},*} &\leq \frac{5\gamma}{2}\Bigg(\sqrt{\frac{c\log(c(C+1)|\cS||\cA||\cB|/[(1-\gamma)^4\delta])}{(1-\gamma)^3 N}}+\frac{c\log(c(C+1)|\cS||\cA||\cB|/[(1-\gamma)^4\delta])}{(1-\gamma)^2N}\Bigg) \notag\\
&\qquad+ \frac{5\gamma\epsilon_{opt}}{2(1-\gamma)}+\epsilon_{opt},
\#
\normalsize 
with some absolute constant $c$, where we have replaced the term $\log(1/(1-\gamma)^2)$ in the bounds for $\|Q^{*}-\hat{Q}^{\mu^*,*}\|_\infty$ and $\|Q^{*}-\hat{Q}^{*,\nu^*}\|_\infty$ in Lemma \ref{lemma:proof_main_res_Q_deviation} (including that in the definition of $\alpha_{\delta,N}$) by $\log((C+1)/(1-\gamma)^4)$, a larger number.  If the second term in the $\max$ in \eqref{equ:lemma:proof_main_res_Q_deviation2_2_trash_1} is larger, \eqref{equ:main_res_2_comp_bnd_1}, \eqref{equ:lemma:proof_main_res_Q_deviation2_2_trash_1}, 
\eqref{equ:thm_main_res_2_final_2},   and  Lemma \ref{lemma:proof_main_res_Q_deviation} together yield
\small
\$
 V^*-V^{\hat{\mu},*} &\leq \frac{5\gamma}{2}\Bigg(\sqrt{\frac{c\log(c(C+1)|\cS||\cA||\cB|/[(1-\gamma)^4\delta])}{(1-\gamma)^3 N}}+\frac{c\log(c(C+1)|\cS||\cA||\cB|/[(1-\gamma)^4\delta])}{(1-\gamma)^2N}\Bigg) \\
&\qquad +\frac{5}{4}\cdot\frac{\gamma}{1-\gamma}\sqrt{\frac{2\log\big(8(C+1)|\cS||\cA||\cB|/[(1-\gamma)^4\delta]\big)}{N}}\Big(2\big\|{V}^{\hat{\mu},*}-{V}^{*}\big\|_\infty\\
&\qquad+\frac{1}{1-\gamma}\sqrt[4]{\frac{c\log(c(C+1)|\cS||\cA||\cB|/\delta)}{N}}\Big)+\epsilon_{opt},
\$
\normalsize
where we have used the fact that $\alpha'_{\delta,N}<1/5$. Taking infinity norm on both sides and solving for $\|{V}^{\hat{\mu},*}-{V}^{*}\|_\infty$, we have
\#\label{equ:thm_main_res_2_final_4}
 &V^*-V^{\hat{\mu},*}\leq\big\|{V}^{\hat{\mu},*}-{V}^{*}\big\|_\infty \leq {5\gamma}\Bigg(\sqrt{\frac{c\log(c(C+1)|\cS||\cA||\cB|/[(1-\gamma)^4\delta])}{(1-\gamma)^3 N}}+ \\
&\quad \frac{c\log(c(C+1)|\cS||\cA||\cB|/[(1-\gamma)^4\delta])}{(1-\gamma)^2N}\Bigg)+\frac{5\gamma}{2(1-\gamma)^2}\Big(\frac{c\log(c(C+1)|\cS||\cA||\cB|/\delta)}{N}\Big)^{3/4}+2\epsilon_{opt},  \notag
\#
with some absolute constant $c$ (which can be different from that in \eqref{equ:thm_main_res_2_final_3}). 
Using the choice of $N$ in the theorem, and combining  \eqref{equ:thm_main_res_2_final_3} and \eqref{equ:thm_main_res_2_final_4}, 
we finally have $V^*-V^{\hat{\mu},*}\leq \epsilon+4\epsilon_{opt}/(1-\gamma)$. 
Note that on the right-hand side of  \eqref{equ:thm_main_res_2_final_4}, the $N$ that makes the 
third term   to be $\cO(\epsilon)$ is $\tilde\cO(1/[(1-\gamma)^{8/3}\epsilon^{4/3}])$, which is   dominated by $\tilde\cO(1/[(1-\gamma)^3\epsilon^2])$ when $\epsilon\in(0,1/(1-\gamma)^{1/2}]$. In addition, to make $\alpha'_{\delta,N}<1/5$, $N$ should be larger than $\tilde\cO(1/(1-\gamma)^2)$, which  is consistent with both the first and third terms on the right-hand side of   \eqref{equ:thm_main_res_2_final_4} to be $\tilde\cO(1/(1-\gamma)^{1/2})$, determining the allowed range of $\epsilon$ to be $(0,1/(1-\gamma)^{1/2}]$.  This proves the first bound in the theorem.

The proof for completing the bound in \eqref{equ:main_res_2_comp_bnd_2} is analogous: using Lemmas  \ref{lemma:V_hat_mu_nu_error}  and \ref{lemma:comp_wise_bnd} to bound  $\|Q^{*,\hat{\nu}}-\hat{Q}^{*,\hat{\nu}}\|_{\infty}$, which is then substituted into \eqref{equ:main_res_2_comp_bnd_2}. 
This completes the proof. 
\hfill\QED

\section{Concluding Remarks}\label{sec:conclusion}

In this paper, we have established the first (near-)minimax optimal sample complexity for model-based MARL, when a generative model is available. Our setting was focused on the basic model in MARL --- infinite-horizon discounted two-player zero-sum   Markov games \citep{littman1994markov}. By noticing that reward is not used in the sampling process of this model-based approach, we have separated the reward-aware and reward-agnostic cases, and established sample complexity lower bounds correspondingly, a unique separation in the multi-agent context. We have then shown  that this simple model-based approach is near-minimax optimal in the reward-aware case, with only a gap in the dependence on $|\cA|,|\cB|$; and is indeed minimax-optimal in the reward-agnostic case. This separation and the (near-)optimal results have not only justified the sample-efficiency of this simple approach, but also reflected both its power (easily handling multiple reward functions known in hindsight), and its limitation (less adaptive and can hardly achieve the optimal $\tilde\cO(|\cA|+|\cB|)$).  
We believe that our results may shed light  on the choice of model-free and model-based approaches in various MARL scenarios in practice.  

Our results  naturally open up the following interesting future directions. First, besides the turn-based setting in \cite{sidford2019solving} and the episodic setting in the concurrent work  \cite{bai2020near}, the minimax-optimal sample complexity in \emph{all}  parameters for \emph{model-free}  algorithms  is still open. As discussed in \S\ref{sec:res}, in the reward-aware case, the $\tilde\Omega(|\cA|+|\cB|)$ lower bound may only be attainable by model-free ones. It would be interesting to compare the results  with our model-based ones, in both reward-aware and reward-agnostic cases, to better understand their pros and cons in various MARL settings.  It would also  be interesting to explore the (near-)optimal sample complexity or regret of model-based approaches in other MARL scenarios, such as when no generative model is available, episodic and average-reward settings, general-sum Markov games, and the setting with function approximation.

\acks{The research of K.Z. and T.B. was supported in part by the US
  Army Research Laboratory (ARL) Cooperative Agreement
  W911NF-17-2-0196, and in part by the Office of Naval Research (ONR)
  MURI Grant N00014-16-1-2710. The research of S.K. was supported by the funding from the ONR award N00014-18-1-2247, and NSF Awards CCF-1703574 and CCF-1740551. We would also like to thank all the anonymous reviewers for their valuable feedback that helped  improve our paper. 
}   
 

\vskip 0.2in
 
\bibliographystyle{ims} 
\bibliography{main}


\appendix 

 
\section{Lower Bounds}\label{sec:append_proof_lb}

Now we discuss lower bounds of the sample complexity given in \S\ref{sec:lower_bound}. 

\subsection{Reward-Aware Case}\label{sec:append_reward_aware_proof}

\paragraph{Proof of Lemma \ref{thm:lb_informal}.} The proof follows by recalling the hard cases of MDPs considered in \cite{azar2013minimax} or  \cite{feng2019does}, and replacing each action $a$ therein by a joint-action $(a,b)$. Without loss of generality, suppose $|\cA|\geq |\cB|$. Then, we design a Markov game such that agent $2$ has no effect on the reward or the transition. Thus, finding an NE is now the same as agent $1$ finding the optimal value/policy.  
By the arguments in \cite{azar2013minimax,feng2019does}, the sample complexity is at least $\tilde\Omega\big(|\cS|\cdot\max\{|\cA|,|\cB|\}\cdot(1-\gamma)^{-3}\epsilon^{-2}\big)$, where $\tilde\Omega$ suppresses some log factors  of $|\cS|,|\cA|,|\cB|$ and $1/\delta$. Noticing  that $\max\{|\cA|,|\cB|\}=(|\cA|+|\cB|+\big||\cA|-|\cB|\big|)/2$, we obtain the lower bound. 
\hfil\QED 

\paragraph{Challenge in Obtaining $\tilde\Omega(|\cA||\cB|)$.} Note that the proof of a $\tilde\Omega\big(|\cS|(|\cA|+|\cB|)\cdot(1-\gamma)^{-3}\epsilon^{-2}\big)$ lower bound is a straightforward adaptation from the single-agent result.  The lower bound can also be obtained in several other  ways (via a treatment of turn-based Markov games, or the attempts to be introduced next).  
Nevertheless, these attempts can hardly lead to a lower bound of $\tilde\Omega(|\cA||\cB|)$, in this reward-aware case.  
We  highlight the challenges as follows. 

{
}

The core proof idea of \cite{azar2013minimax,feng2019does} for the single-agent setting lower bound is to create a class of $\cO(|\cS||\cA|)$ number of MDPs, which are hard to distinguish from each other. 
When the reward function is given (i.e.,  in the reward-aware setting), 
one can only change the transition model to obtain different hard MDPs.
Hence, in \cite{azar2013minimax}, 
their approach is to first create a null hypothesis, in which the optimal $Q$-value and $\epsilon$-optimal actions at every state are fixed. 
Then in each of the $\cO(|\cS||\cA|)$ alternative hypothesis, they change 
the transition probability of a distinct state-action pair $(s,a)$ in the null case to make the Q-value slightly differ from the null-setting and such that $a$ is an $\epsilon$-optimal action at state $s$. 
They construct the hard instance cleverly such that if an algorithm correctly outputs the optimal Q-value (or optimal policy) in an alternative hypothesis with high probability, then it must have sampled  $\tilde\Omega((1-\gamma)^{-3}\epsilon^{-2})$ samples at the corresponding $(s,a)$ pair in the null hypothesis. As this holds for all $\cO(|\cS||\cA|)$ alternative hypotheses, 
we obtain 
an $\tilde\Omega(|\cS||\cA|(1-\gamma)^{-3}\epsilon^{-2})$ sample lower bound.

In the game setting, however, the above idea requires to change the Nash equilibrium (say, a unique pure strategy) to a different state-action-action tuple at any state while only make changes to the probability transition of the corresponding state-joint-action tuple.
Nevertheless, this is challenging to achieve in general, as the NE value of zero-sum matrix games is \emph{not sensitive} to the small number of element changes in the payoff matrices.  This can be evidenced  either by the stability of the NE in this case against the payoff perturbation \citep{jansen1981regularity}, or by the  sensitivity analysis of the equivalent linear program of the game  \citep{luce1989games} against the problem data \citep{dantzig1998linear}. Indeed, one can verify that only changing $\cO(1)$ elements in the transition probability matrix, and thus changing $\cO(1)$ elements in the Q-value table at each state, by a small amount,  can hardly change the NE value/policy too much. Some order of $\cO(|\cA|)$ (or $\cO(|\cB|)$) number of changes may suffice, but will eventually yields $\cO(|\cB|)$ (or $\cO(|\cA|)$) hard alternative cases, leading to the same $\tilde\Omega(|\cA|+|\cB|)$ result as Lemma \ref{thm:lb_informal}. In other words, one can hardly obtain the sufficient number of required hard cases ($\tilde\Omega(|\cA||\cB|)$ in total) by changing only $\cO(1)$ elements in the transition probability matrix of each alternative hypothesis case.

 On the other hand, interestingly, we note that there are some results on the \emph{payoff query complexity}, i.e., the number of queries for the elements in the payoff matrix,  for finding the NE \citep{JMLR:v16:fearnley15a,fearnley2016finding}. It is possible to use $\cO(k\log(k)/\epsilon^2)$ queries to find the $\epsilon$-NE in zero-sum matrix games when $|\cA|=|\cB|$, where $k=|\cA|=|\cB|$ \citep{fearnley2016finding}. Note that the lower bound given in \cite{fearnley2016finding}, though  being  $\Omega(k^2)$, requires the accuracy $\epsilon\leq 1/k$ to be small, which cannot be used in our previous analysis with a dimension-free choice of $\epsilon$. From a different angle, these results imply that it may indeed be unnecessary  to  accurately estimate \emph{all} elements in the matrix, in order to obtain an approximate Nash equilibrium.  
 
 In light of these observations, we have conjectured that with reward knowledge,    the lower bound of $\tilde\Omega(|\cA|+|\cB|)$ is indeed unimprovable, which might be matched by some other (possibly \emph{model-free}) MARL algorithms, as general model-based approaches inherently require  $\tilde\Omega(|\cA||\cB|)$ for transition model estimation. Such a $\tilde\Omega(|\cA|+|\cB|)$ lower bound on {regret} has been provided recently in \cite{bai2020provable}, though in a different setting. More interestingly,  though not entirely comparable to us, in the concurrent work \cite{bai2020near}, the $\tilde \cO(|\cA|+|\cB|)$ complexity is indeed shown to be attainable by a model-free Nash-V learning algorithm {in the episodic setting}, with the reward information guiding the online update.



\subsection{Reward-Agnostic Case}\label{sec:append_reward_agnostic_proof}

Now we establish the lower bound for the reward-agnostic case, i.e., the proof of  Theorem \ref{thm:lb_reward_agnostic}. 
The idea to construct hard cases is similar to that discussed in \S\ref{sec:append_reward_aware_proof}, which is motivated by \cite{azar2013minimax,feng2019does},  but with additional flexibility to design the reward function that is unknown in the sampling stage.   
Our hard cases apply to both finding the $\epsilon$-NE policy pair and finding the $\epsilon$-approximate NE value. 
For the sake of presentation, we focus on proving the lower bound for the $\epsilon$-NE policy. 
Let us first formally define the notion of a correct algorithm in terms of learning an $\epsilon$-NE policy in this reward-agnostic case.

\begin{definition}(($\epsilon, \delta$)-correct reward-agnostic algorithm)\label{def:epsilonalgorithm}
	 We say that an RL algorithm $\mathfrak{A}$ is $(\epsilon, \delta)$-correct in the reward-agnostic case, if for any unknown MG $\cG=(\cS, \cA, \cB, r, P, \gamma)$, $\mathfrak{A}$ first calls a generative model on $(\cS, \cA, \cB, P, \gamma)$, and is then fed with the reward $r$, and  outputs an $\epsilon$-NE policy  $(\mu, \nu)$ with probability at least $1-\delta$. 
\end{definition}  
 
Note that $r$ is only revealed to   $
\mathfrak{A}$ after the sampling, and such an $
\mathfrak{A}$ should be able to output an ($\epsilon, \delta$)-correct NE policy for any single $r$ in the underlying model. Thus, for  $M$ reward functions defined over the same $(\cS, \cA, \cB, P, \gamma)$, using a union bound argument, the $\epsilon$-NE policy corresponding to all $M$ reward functions can be obtained simultaneously with probability greater than $1-M\delta$ (of course with a small enough $\delta$).  To prove the theorem, we will construct a class of Markov game models. We show that if algorithm $
\mathfrak{A}$ only draws samples much fewer than the lower bound, there exists an MG $\cG$ such that $\mathfrak{A}$ cannot be an $(\epsilon, \delta)$-correct reward-agnostic algorithm for. \kzedit{Compared to the reward-aware case, we now allow more freedom to construct hard instances, by not only perturbing the transition matrix, but also choosing the reward function judiciously. This would eventually allow us to obtain $\Theta(|\cA||\cB|)$ hard cases, combating the insensitivity of NE to the perturbation of the payoff matrices (c.f. discussion in \S\ref{sec:append_reward_aware_proof}).}

\begin{figure*}[!t]
	\centering
	\begin{tabular}{c}
		\includegraphics[width=0.98\textwidth]{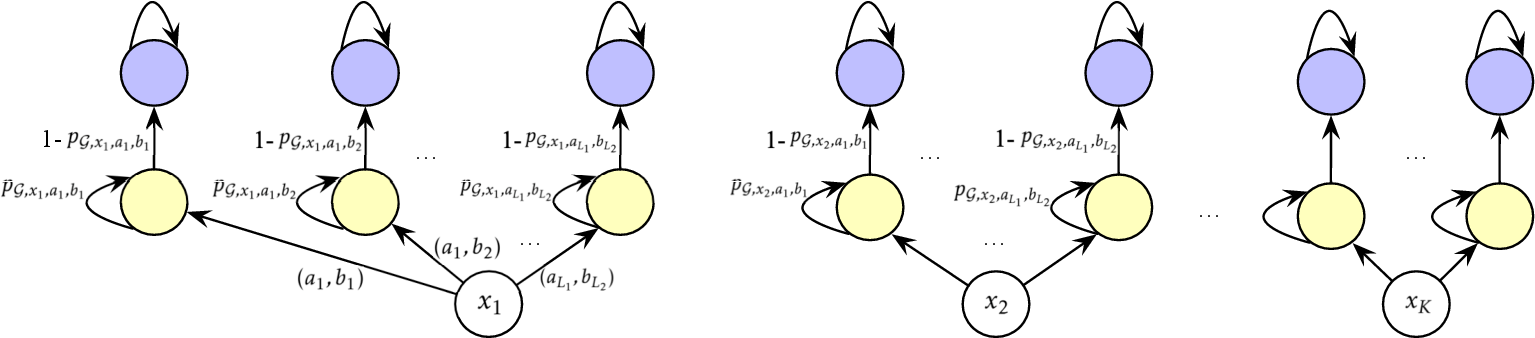} 
	\end{tabular}
	\caption{The class of zero-sum Markov games $\mathbb{G}$ considered in the proof of  Theorem \ref{thm:lb_reward_agnostic}. Circles denote the states and arrows denote the transitions. White, yellow, and blue circles denote the three disjoint subsets of states $\cX$ , $\cY_1$, and $\cY_2$, respectively. } 
	\label{fig:hard_case}
\end{figure*}
 
\paragraph{Construction of the Hard Case.}
We define a family of MGs $\mathbb{G}$. {See illustrations in Figure \ref{fig:hard_case}}.  The state space $\cS$ consists of three disjoint subsets $\cX$ , $\cY_1$, and $\cY_2$. The set $\cX$ includes $K$ states $\{x_1, x_2, \dots, x_K\}$ and each of them has $L_1>1$ available max-player actions $\{a_1,a_2,\dots,a_{L_1}\}=:\cA$, and $L_2>1$ min-player actions $\{b_1,b_2,\dots,b_{L_2}\}=:\cB$. Each  state in $\cY_1:=\{y_{1,x,a,b}: ~ \forall x\in \cX, a\in \cA, b\in \cB\}$ and $\cY_2:=\{y_{2,x,a,b}: ~ \forall x\in \cX, a\in \cA, b\in \cB\}$ only has a single joint-action pair to choose.   
In total, there are $N:=3KL_1L_2$ state-joint-action pairs. For state $x\in\cX$, by taking a joint-action   $(a,b)$ for $a\in\cA, b\in \cB$, it transitions to a state $y_{1,x,a,b}\in\cY_1$ with probability 1. For state $y_{1,x,a,b}\in\cY_1$, there is only a single joint-action for both players to choose from, which is the $(a,b)$ pair that leads to this $y_{1,x,a,b}$. It  then transitions to itself with probability $p_{\cG,x,a,b}\in(1/2,1)$ and to a corresponding state $y_{2,x,a,b} \in \cY_2$ with probability $1-p_{\cG,x,a,b}$. Note that $p_{\cG,x,a,b}$ can be different for different state-joint-action tuples. All states in $\cY_2$ are absorbing. The reward function is: for any $y_{1,x,a,b}\in\cY_1$, 
$R(y_{1,x,a,b})=\iota_{\cG,x,a,b}$ for some $\iota_{\cG,x,a,b}\in[0,1]$ (to be specified later); and for all other states,  $R(s)=0$. 
And the Q-function of the MGs can be computed as 
\begin{align} 
Q_{\cG}(x,a,b) = \frac{\gamma \iota_{\cG,x,a,b}}{1-\gamma p_{\cG,x,a,b}}, \quad \forall~ (x,a,b)\in\cX\times\cA\times \cB,
\end{align}
which is fully characterized by $p_{\cG,x,a,b}$ and $\iota_{\cG,x,a,b}$. 

\paragraph{Transition Model Hypotheses of $\cG$.} 
We restrict $\gamma\in(1/2,1)$.
Let $p_0=\gamma$ and $\alpha_1,\alpha_2 \in (0,1)$. We consider $M+1$ possibilities of the transition models of  $\cG$, where   $M:=K[L_1(L_2-1)]$ --- the null hypothesis is:
\begin{align}
\cG_{1}:& \begin{cases}p_{\cG_1,x_k,a_1, b_1} = p_0-\alpha_1,\quad\forall~ k\in[K],\\
p_{\cG_1,x_k,a_l, b_1} = p_0-2\alpha_1,\quad\forall~ k\in[K],l\in [L_1]\backslash\{1\},\\ p_{\cG_1,x_k,a_{l_1}, b_{l_2}}=p_0,\quad \forall~k\in[K], l_1\in [L_1], l_2\neq1;\end{cases}
\end{align}
and for all $k\in[K]$,  $l_1\in[L_1]$, and $l_2\in[L_2]\backslash\{1\}$ the $M$ alternative hypotheses are:
\begin{align}\label{equ:construct_G_alt}
\cG_{k,l_1, l_2}:& \begin{cases}p_{\cG_{k,l_1, l_2}, x_{k},a_{l_1}, b_{l_2}} = p_{\cG_{1}, x_{k},a_{l_1}, b_{l_2}} - \alpha_2=p_0-\alpha_2,\\
p_{\cG_{k,l_1, l_2}, x_{k},a_{l_1'}, b_{l_2}} = p_{\cG_{1}, x_{k},a_{l_1'}, b_{l_2}}=p_0
 ,\quad \forall l_1'\neq l_1,\\
p_{\cG_{k,l_1, l_2}, x_{k'},a_{l_1'}, b_{l_2'}} = p_{\cG_{1}, x_{k'},a_{l_1'}, b_{l_2'}},\quad \forall (k',l_1', l_2')\neq(k,l_1,l_2),\end{cases}
\end{align} 
where $\alpha_1=c'(1-\gamma p_0)^2\epsilon/\gamma$, $\alpha_2=c(1-\gamma p_0)^2\epsilon/\gamma$ for some $\epsilon\in(0,1)$ and absolute constants $c', c>0$ to be determined later. Note that each alternative hypothesis only has one element in the transition model  different from the null one. 

\paragraph{Reward Functions.}  
We define a class of $M+1$ reward functions $\mathfrak{R}=\{r_1\}\bigcup\{r_{k,l_1,l_2}: ~ \forall k\in [K], l_1\in[L_1], l_2\in[L_2]\backslash\{1\}\}$, which is unknown to $\mathfrak{A}$ during sampling, and is defined as follows (recall that other than the value specified by $\iota_{\cG,x,a, b}$, rewards are all zero):  
\begin{align*}
r_{1}&:\begin{cases}\iota_{\cG,x_k,a_1, b_1}=1,\quad\forall~ k\in[K],\\
\iota_{\cG,x_k,a_l, b_1}=1,\quad\forall~ k\in[K],l\in [L_1]\backslash\{1\},\\ \iota_{\cG,x_k,a_{l_1}, b_{l_2}}=1,\quad \forall~k\in[K], l_1\in [L_1], l_2\neq1;\end{cases}\\
r_{k,l_1, l_2}&:\begin{cases}\iota_{\cG, x_{k},a_{l_1}, b_{l_2}}=1,\\
\iota_{\cG, x_{k},a_{l_1'}, b_{l_2}}=\frac{1-\gamma p_{\cG_{1}, x_{k},a_{l_1'}, b_{l_2}}}{1-\gamma \big(p_{\cG_{1}, x_{k},a_{l_1'}, b_{l_2}}-2\alpha_2\big)},\quad \forall l_1'\neq l_1,\\
\iota_{\cG, x_{k'},a_{l_1'}, b_{l_2'}}=1,\quad \forall (k',l_1', l_2')\neq(k,l_1,l_2),\end{cases}
\end{align*}
for all $k\in[K]$,  $l_1\in[L_1]$, and $l_2\in[L_2]\backslash\{1\}$.

By the construction above, if the reward function $r_m\in\mathfrak{R}$ is assigned to the corresponding transition model in $\cG_{m}$, for either $m=1$ or any  $m=(k,l_1, l_2)$, then the corresponding Q-values become 
\small
\begin{align}
\text{for}~~\cG_{1}:& \begin{cases}Q_{\cG_1}(x_k,a_1, b_1) = \frac{\gamma}{1-\gamma (p_0-\alpha_1)},\quad\forall~ k\in[K],\\
Q_{\cG_1}(x_k,a_l, b_1) = \frac{\gamma}{1-\gamma (p_0-2\alpha_1)},\quad\forall~ k\in[K],l\in [L_1]\backslash\{1\},\\ 
Q_{\cG_1}(x_k,a_{l_1}, b_{l_2})=\frac{\gamma}{1-\gamma p_0},\quad \forall~k\in[K], l_1\in [L_1], l_2\neq1;\end{cases}\label{equ:nomimal_assign_1}
\end{align}
and $\forall k\in[K],~l_1\in[L_1],~l_2\in[L_2]\backslash\{1\}$,  
\begin{align}
\text{for}~~\cG_{k,l_1, l_2}:& \begin{cases}Q_{\cG_{k,l_1, l_2}}(x_{k},a_{l_1}, b_{l_2}) = \frac{\gamma}{1-\gamma \big(p_{\cG_{1}, x_{k},a_{l_1}, b_{l_2}}-\alpha_2\big)},\\
Q_{\cG_{k,l_1, l_2}}(x_{k},a_{l_1'}, b_{l_2}) = \frac{\gamma}{1-\gamma \big(p_{\cG_{1}, x_{k},a_{l_1'}, b_{l_2}}-2\alpha_2\big)},\quad \forall l_1'\neq l_1,\\
Q_{\cG_{k,l_1, l_2}}(x_{k'},a_{l_1'},  b_{l_2'}) = \frac{\gamma}{1-\gamma p_{\cG_{1}, x_{k'},a_{l_1'}, b_{l_2'}}},\quad \forall (k',l_1', l_2')\neq(k,l_1,l_2).\end{cases}\label{equ:nomimal_assign_2} 
\end{align}
\normalsize
We then select $\alpha_1$ such that for $(l_1,l_2)\neq(1,1)$, 
\begin{align}\label{equ:def_Q_G_1_distance}
\big|Q_{\cG_{1}}&(x_k,a_1, b_1)-Q_{\cG_1}(x_k,a_{l_1},b_{l_2})\big|\nonumber\\
&\ge\min\Big(\frac{\gamma}{1-\gamma p_0} - \frac{\gamma}{1-\gamma(p_0-\alpha_1)},  \frac{\gamma}{1-\gamma(p_0-\alpha_1)}-\frac{\gamma}{1-\gamma (p_0-2\alpha_1)} \Big)\geq 20\epsilon
\end{align}
and $\alpha_2$ is selected such that $\alpha_2\geq 2\alpha_1$ and 
\begin{align}\label{equ:choose_alpha_2}
48\epsilon\geq |Q_{\cG_{k,l_1,l_2}}(x_{k},a_{l_1}, b_{l_2})-Q_{\cG_{k,l_1, l_2}}(x_{k},a_{l_1'}, b_{l_2'})| 
\geq 20\epsilon
\end{align}
for all $(l_1', l_2') \neq (l_1, l_2)$. 
Moreover, we require that
$p_0\in (1/2+2\alpha_1+2\alpha_2, 1)$,
$\alpha_2/(1-p_0)\in (0, 1/2)$ and $\alpha_2/(p_0 -2\alpha_1- 2\alpha_2)\in (0,1/2)$. Hence, $\epsilon \le \cO(1/(1-\gamma))$.
 
In the sequel, we denote $\EE_{1}$ and $\PP_1$ to measure the expectation and probability of an event under the transition model hypothesis $\cG_1$. Similarly, we denote $\EE_{k, l_1, l_2}$ and $\PP_{k, l_1, l_2}$ to measure the expectation and probability of an event under hypothesis $\cG_{k, l_1, l_2}$.  
It is not hard to verify that in the above case (with $r_m$ being assigned to $\cG_m$ correspondingly, for $m=1$ or $m=(k,l_1,l_2)$), there is a unique NE policy pair under hypothesis $\cG_1$: for $x\in \cX$, 
$\mu_1^*(x) = a_1$ and $\nu_1^*(x)=b_1$; and that there is a unique NE policy under hypothesis $\cG_{k,l_1, l_2}$ for all $l_1\in[L_1]$ and $l_2\in[L_2]\backslash\{1\}$: for $k'\neq k$, 
$\mu_{k,l_1,l_2}^*(x_{k'}) = a_1$, 
$\nu_{k,l_1,l_2}^*(x_{k'}) = b_1$, and 
$\mu_{k,l_1,l_2}^*(x_k) = a_{l_1}$, $\nu_{k,l_1,l_2}^*(x_k) = b_{l_2}$.

Moreover, one can verify that if any reward $r_{k,l_1,l_2}$ with $k\in[K]$,  $l_1\in[L_1]$, and $l_2\in[L_2]\backslash\{1\}$ (instead of $r_1$ as in \eqref{equ:nomimal_assign_1}) is assigned to the transition model of $\cG_1$, then the NE policy at $x_k$ can never be the pure strategy $\mu_{1}^*(x_k) = a_{l_1}$, $\nu_{1}^*(x_k) = b_{l_2}$ (it can be some mixed NE policy). As a consequence, for algorithm $\mathfrak{A}$, after estimating the transition model of $\cG_1$, if $r_{k,l_1,l_2}$ is revealed, then it will output some $\epsilon$-NE policy with probability greater than $1-\delta$; this $\epsilon$-NE policy pair, which can be  mixed strategies at $x_k$, should output the joint-action $(a_{l_1},b_{l_2})$ with a small probability, which is smaller than 
\small
\#\label{equ:def_beta_1}
\beta:=\frac{\frac{\gamma}{1-\gamma(p_0-\alpha_1)}-\frac{\gamma}{1-\gamma(p_0-2\alpha_2)}+\epsilon}{\frac{\gamma}{1-\gamma p_0}-\frac{\gamma}{1-\gamma (p_0-2\alpha_2)}}=1-\frac{\frac{\gamma}{1-\gamma p_0}-\frac{\gamma}{1-\gamma (p_0-\alpha_1)}-\epsilon}{\frac{\gamma}{1-\gamma p_0}-\frac{\gamma}{1-\gamma (p_0-2\alpha_2)}}\leq 1-\frac{19}{96}\leq 1-\frac{19\epsilon(1-\gamma p_0)}{\gamma},
\# 
\normalsize
(implying that $\epsilon\leq \gamma/[96(1-\gamma p_0)]$), where the first inequality is due to \eqref{equ:def_Q_G_1_distance}-\eqref{equ:choose_alpha_2}, and the last one follows by upper-bounding  ${\gamma}/{(1-\gamma p_0)}-{\gamma}/{[1-\gamma (p_0-2\alpha_2)]}$ simply by ${\gamma}/{(1-\gamma p_0)}$. This is because  otherwise, the value of the $\epsilon$-NE policy at $x_k$, denoted by $V^*_{\cG_1}(x_k)$ satisfies 
\#\label{equ:V_G_1_x_k}
V^*_{\cG_1}(x_k)\geq \beta\cdot\frac{\gamma}{1-\gamma p_{0}}+(1-\beta)\cdot\frac{\gamma}{1-\gamma (p_{0}-2\alpha_2)}= \frac{\gamma}{1-\gamma (p_{0}-\alpha_1)}+\epsilon,
\#
where the first inequality is because with reward $r_{k,l_1,l_2}$ being assigned to  model $\cG_1$, at state $x_k$ and with the joint-action $(a_{l_1},b_{l_2})$, the Q-value is ${\gamma}/{(1-\gamma p_{0})}$, while the smallest Q-value at state $x_k$ is ${\gamma}/{[1-\gamma (p_{0}-2\alpha_2)]}$; the last equation  is due to the definition of $\beta$ in \eqref{equ:def_beta_1}. However,  one can verify that the NE-value in this case lies in the range $[{\gamma}/{(1-\gamma (p_{0}-2\alpha_1))},{\gamma}/{(1-\gamma (p_{0}-\alpha_1))}]$, by finding the minimax and maximin elements in the payoff matrix, i.e., the Q-value table at $x_k$ (using {Lemma} \ref{lemma:NE_value_range}). Thus, \eqref{equ:V_G_1_x_k} contradicts the fact that this policy is an $\epsilon$-NE policy (thus making $V^*_{\cG_1}(x_k)$  $\epsilon$-close to the NE-value).  If we define the following  events  for every $k\in[K]$, $l_1\in[L_1]$, and $l_2\in[L_2]\backslash\{1\}$: 
\small 
\begin{align}\label{equ:B_event_def}
B_{k,l_1, l_2} &= \Big\{\text{when fed with }r_{k,l_1,l_2}\in\mathfrak{R},~\mathfrak{A} \text{ outputs } (\mu,\nu) \text{ s.t.  at } x_{k},~\big(a_{l_1}, b_{l_2}\big) \text{ is generated w.p.} \leq \beta\Big\},
\end{align}
\normalsize 
then the above argument can be written as $\PP_1\big(B_{k,l_1, l_2}\big)\geq 1-\delta$. 


Now, we fix $\epsilon\in(0, \epsilon_0)$ and $\delta\in(0, \delta_0)$, where $\epsilon_0$ and $\delta_0$ will be determined later. 
Let 
$$t^* = \frac{c_1}{(1-\gamma)^3\epsilon^2}\log\Big(\frac{1}{4\delta}\Big),$$
where $c_1>0$ is an absolute constant to be determined later. We also define $T_{k,l_1, l_2}$ to be the number of samples that algorithm $\mathfrak{A}$ calls from the generative model with input state $y_{1,x_k,a_{l_1}, b_{l_2}}$ till $\mathfrak{A}$ stops (these sample calls are not necessarily consecutive). Note that no reward information is used/revealed to the agent in this sampling process of $\mathfrak{A}$. For every $k\in[K]$, $l_1\in[L_1]$, and $l_2\in[L_2]\backslash\{1\}$, we define the following two events:
\begin{align}
A_{k,l_1, l_2} &= \{T_{k,l_1, l_2}\leq 4t^*\},\\ 
C_{k,l_1, l_2} &= \Big\{S_{k,l_1, l_2}-p_{\cG_1, x_k, a_{l_1}, b_{l_2}}T_{k,l_1, l_2}\leq\sqrt{2p_{\cG_1, x_k, a_{l_1}, b_{l_2}}(1-p_{\cG_1, x_k, a_{l_1}, b_{l_2}})T_{k,l_1, l_2}\log(1/4\delta)}\Big\} ,
\end{align}
where $S_{k,l_1, l_2}$ is the number of transitions to itself in the $T_{k,l_1, l_2}$ calls to the  generative model with input state $y_{1, x_k,a_{l_1}, b_{l_2}}$. For these events, we have the following lemmas.
 
\begin{lemma}\label{lemma:lb_1}
	For any $k\in[K]$, $l_1\in[L_1]$, and $l_2\in[L_2]\backslash\{1\}$, if ~$\mathbb{E}_1[T_{k,l_1, l_2}]\leq t^*$, $\mathbb{P}_1(A_{k,l_1, l_2})> 3/4$.
\end{lemma}
	\begin{proof}
	Notice that 
		$$t^*\geq \mathbb{E}_1[T_{k,l_1, l_2}]> 4t^*\mathbb{P}_1(T_{k,l_1, l_2}>4t^*)=4t^*(1-\mathbb{P}_1(T_{k,l_1, l_2}\leq 4t^*)).$$
		Thus, $\mathbb{P}_1(A_{k,l_1, l_2})> 3/4$.
	\end{proof}

\begin{lemma}\label{lemma:lb_2}
	For any $k\in[K]$, $l_1\in[L_1]$, and $l_2\in[L_2]\backslash\{1\}$, if ~$\delta<1/16$, $\mathbb{P}_1(C_{k,l_1, l_2})\geq 3/4$. 
\end{lemma}	
	\begin{proof}
		We denote \emph{outcome} to be $1$ if the transition from $y_{1,x_k,a_{l_1}, b_{l_2}}$ ends up on itself; otherwise 0. 
		By definition, the outcomes from state $y_{1,x_k,a_{l_1}, b_{l_2}}$ are i.i.d. Bernoulli-$p_{\cG_{1},x_k,a_{l_1}, b_{l_2}}$ random variables. Let  $\epsilon:=\sqrt{2p_{\cG_1,x_k,a_{l_1}, b_{l_2}}(1-p_{\cG_1,x_k,a_{l_1}, b_{l_2}})T_{k,l_1,l_2}\log(1/4\delta)}$. 
		By Chernoff-Hoeffding bound and $p_{\cG_1, x_k, a_{l_1}, b_{l_2}}\ge p_0 - 2\alpha_1>1/2$, we have that
		\begin{align}
		&\mathbb{P}_1\bigg(S_{k,l}-p_{\cG_1,x_k,a_{l_1}, b_{l_2}}T_{k,l_1,l_2} \leq \epsilon\bigg)\notag\\
		&\geq ~1-\exp\bigg(-\text{KL}\bigg(p_{\cG_1,x_k,a_{l_1}, b_{l_2}}+\frac{\epsilon}{T_{k,l_1,l_2}}~\Big|\Big|~p_{\cG_1,x_k,a_{l_1}, b_{l_2}}\bigg)\cdot T_{k,l_1,l_2}\bigg)\nonumber\\
		&\geq 1-\exp\bigg(-\frac{\epsilon^2}{2p_{\cG_1,x_k,a_{l_1}, b_{l_2}}(1-p_{\cG_1,x_k,a_{l_1}, b_{l_2}})T_{k,l_1,l_2}}\bigg)\ge 1-4\delta.
		\end{align}
		\normalsize 
		Additional application of 
		$\delta<1/16$ 
		 proves the lemma.
	\end{proof}

Let $\delta_0=1/16$ and $\epsilon_0={\gamma}/{[96(1-\gamma p_0)]}$. Then, for $\delta\in(0,\delta_0)$ and $\epsilon\in(0,\epsilon_0)$, and with the  transition model of $\cG_1$ being input, by the argument after \eqref{equ:B_event_def}, we have $\mathbb{P}_1(B_{k,l_1,l_2})\geq 1-\delta\geq 1-1/16\geq 3/4$,  for all $k\in[K]$, $l_1\in[L_1]$, and $l_2\in[L_2]\backslash\{1\}$. Define the event $\cE_{k,l_1, l_2}:=A_{k,l_1, l_2}\cap B_{k,l_1,l_2} \cap C_{k,l_1, l_2}$. Combining Lemmas \ref{lemma:lb_1} and \ref{lemma:lb_2} and $\mathbb{P}_1(B_{k,l_1,l_2})\geq 3/4$, we have that
\begin{align}
\mathbb{P}_1(\cE_{k,l_1, l_2})>(3/4)^3>1/4, \quad \forall~ k\in[K],~l_1\in[L_1],~l_2\in[L_2]\backslash\{1\},
\end{align}
if $\mathbb{E}_1[T_{k,l_1, l_2}]\leq t^*$, $\delta\in(0,\delta_0)$ and $\epsilon\in(0,\epsilon_0)$.  
Next, we show that if the expectation of the number of samples in $\mathfrak{A}$ on any $y_{1, x_k,a_{l_1}, b_{l_2}}$ is no greater than $t^*$ under the hypothesis $\cG_{1}$, then $B_{k,l_1,l_2}$ occurs with probability greater than $\delta$ under the hypothesis $\cG_{k,l_1, l_2}$.

\begin{lemma}\label{lemma:newversion}
	Let $\epsilon_0 =\min\Big\{\frac{\gamma}{96(1-\gamma p_0)},~c''\min\Big\{ \frac{\gamma}{(1-\gamma p_0)^2}, \frac{1}{1-\gamma}\Big\}\Big\}$ for some constant $c''>0$. For any $k\in[K]$,  $l_1\in[L_1]$, and $l_2\in[L_2]\backslash\{1\}$, when $\epsilon\in(0, \epsilon_0)$, if ~{$\mathbb{E}_1[T_{k,l_1, l_2}]\leq t^*$}, then {$\mathbb{P}_{k,l_1, l_2}(B_{k,l_1,l_2})\geq \delta$}.
\end{lemma}
	\begin{proof} 
		Let $W$ be the length-$T_{k,l_1, l_2}$ random sequence of the next states by calling the generative model $T_{k,l_1, l_2}$ times  with the input state $y_{1, x_k,a_{l_1}, b_{l_2}}$. 
		To simplify notation, we represent $W$ as a binary sequence where $1$ represents the next state from  $y_{1, x_k,a_{l_1}, b_{l_2}}$ to itself and $0$ otherwise.
		If $(l_1, l_2)\neq (1,1)$ and $\cG=\cG_1$, $W$ forms an i.i.d. Bernoulli-$p_{\cG_{1}, x_k, a_{l_1}, b_{l_2}}$ sequence; if $\cG=\cG_{k,l_1, l_2}$, this is an i.i.d Bernoulli-$p_{\cG_{k, l_1, l_2}, x_k, a_{l_1}, b_{l_2}}$ sequence. We define the likelihood function $\cL_{k,l_1, l_2}$ as 
		$$
		\forall w\in \{0,1\}^{T_{k,l_1, l_2}}:\quad \cL_{k,l_1, l_2}(w) = \mathbb{P}_{k,l_1, l_2}[W=w]\quad\text{and}\quad  \cL_{1}(w) = \mathbb{P}_{1}[W=w].$$ Recall that the notation $S_{k,l_1, l_2}$  denotes the total number of $1$'s in $W$. 
		For convenience, let us denote
		\[
		p_1 = p_{\cG_1, x_k, a_{l_1}, b_{l_2}}, \quad\text{and} \quad 
		p_2 = p_{\cG_{k, l_1, l_2}, x_k, a_{l_1}, b_{l_2}}.
		\]
		Note that
		\[
		p_1-p_2 = \alpha_2.
		\]
		To additionally simplify the notation, we define $T=T_{k, l_1, l_2}$ and $S=S_{k, l_1, l_2}$.
		With these new notations, we compute  $\cL_{k,l_1, l_2}(W)/\cL_1(W)$ as follows
		\begin{align}
		\frac{\cL_{k,l_1, l_2}(W)}{\cL_1(W)}&=\frac{(p_2)^{S}(1-p_2)^{T-S}}{(p_1)^{S}(1-p_1)^{T-S}}= \left(1+\frac{\alpha_2}{p_1}\right)^{S}\left(1-\frac{\alpha_2}{1-p_1}\right)^{T-S}\notag\\
		&= \left(1+ \frac{\alpha_2}{p_1}\right)^{S}\left(1-\frac{\alpha_2}{1-p_1}\right)^{S\frac{1-p_1}{p_1}}\left(1-\frac{\alpha_2}{1-p_1}\right)^{T-S/p_1}.
		\end{align}
		Note that $p_0-2\alpha_1\le p_1\le p_0$.
		By our choice of $p_0$, $\alpha_1$, $\alpha_2$, and $\epsilon$, it holds that $\alpha_2/(1-p_1)\in(0, 1/2)$  
		 and $\alpha_2/p_1\in(0,1/2)$. 
		 With the fact that $\log (1-u) \geq-u-u^{2}$ for $u\in[0,1/2]$ and $\exp (-u) \geq 1-u$ for $u\in[0,1]$, we have that
		\begin{align}
		\bigg(1-\frac{\alpha_2}{1-p_1}\bigg)^{\frac{1-p_1}{p_1}} \geq \exp \left(\frac{1-p_1}{p_1}\left(-\frac{\alpha_2}{1-p_1}-\big(\frac{\alpha_2}{1-p_1}\big)^{2}\right)\right)  \geq\left(1-\frac{\alpha_2}{p_1}\right)\left(1-\frac{\alpha_2^{2}}{p_1(1-p_1)}\right).
		\end{align}
		Thus
		\begin{align} \frac{\cL_{k,l_1, l_2}(W)}{\cL_{1}(W)} & \geq\left(1-\frac{\alpha_2^2}{p_1^2}\right)^{S}\left(1-\frac{\alpha_2^{2}}{p_1(1-p_1)}\right)^{S}\bigg(1-\frac{\alpha_2}{1-p_1}\bigg)^{T-S/p_1} \\ & \geq\left(1-\frac{\alpha_2^2}{p_1^2}\right)^{T}\left(1-\frac{\alpha_2^{2}}{p_1(1-p_1)}\right)^{T}\bigg(1-\frac{\alpha_2}{1-p_1}\bigg)^{T-S/p_1}
		\end{align}
		due to $S\leq T$. Next, we proceed on the event $\cE_{k,l_1, l_2}$. By definition, if $\cE_{k,l_1, l_2}$ occurs, event $A_{k,l_1, l_2}$ has occurred. Using $\log(1-u)\geq -2u$ for $u\in[0,1/2]$, it follows that
		$$\left(1-\frac{\alpha_2^2}{p_1^2}\right)^{T} \geq \left(1-\frac{\alpha_2^2}{p_1^2}\right)^{4t^*} \geq \exp \left(-8t^* \frac{\alpha_2^2}{p_1^2}\right) \geq\left(4\delta\right)^{128c^2c_1},$$
		where we use the fact that
		\begin{align*}
		t^* \cdot \frac{\alpha_2^2}{p_1^2}
		&= \frac{c_1}{(1-\gamma)^3\epsilon^2}\log\Big(\frac{1}{4\delta}\Big) \cdot \frac{c^2(1-\gamma p_0)^4\epsilon^2}{\gamma^2p_1^2}\le \frac{c_1}{(1-\gamma)^3}\log\Big(\frac{1}{4\delta}\Big) \cdot \frac{c^2(1-\gamma p_0)^4}{\gamma^2p_1^2}\\
		&\le 16c_1c^2(1-\gamma)\cdot\log(1/4\delta). 
		\end{align*}
		Using $\log(1-u)\geq -2u$ for $u\in[0,1/2]$, we also have that
		$$\left(1-\frac{\alpha_2^{2}}{p_1(1-p_1)}\right)^{T} \geq \left(1-\frac{\alpha_2^{2}}{p_1(1-p_1)}\right)^{4t^*}\geq \exp \left(-8t^* \frac{\alpha_2^{2}}{p_1(1-p_1)}\right) \geq\left(4\delta\right)^{64c^2c_1},$$
		where we use
		\begin{align*}
		t^* \cdot \frac{\alpha_2^2}{p_1(1-p_1)}
		&= \frac{c_1}{(1-\gamma)^3\epsilon^2}\log\Big(\frac{1}{4\delta}\Big) \cdot \frac{c^2(1-\gamma p_0)^4\epsilon^2}{\gamma^2p_1(1-p_1)}\\
		&\le \frac{c_1}{(1-\gamma)^3}\log\Big(\frac{1}{4\delta}\Big) \cdot \frac{c^2(1-\gamma p_0)^4}{\gamma^2(p_1)(1-p_0)}\le 8c_1c^2\cdot\log(1/4\delta).
		\end{align*} 
		Further, we have that when $\cE_{k,l_1, l_2}$ occurs, $C_{k,l_1, l_2}$ also occurs. Therefore,
		\begin{align*}
		\left(1-\frac{\alpha_2}{1-p_1}\right)^{T-S/{p_1}} & \geq\left(1-\frac{\alpha_2}{1-p_1}\right)^{\sqrt{\frac{1-p_1}{p_1}T\log(1/4\delta)}}
		\geq\left(1-\frac{\alpha_2}{1-p_1}\right)^{\sqrt{\frac{1-p_1}{p_1}4t^*\log(1/4\delta)}} \\
		& \geq \exp \left(-\sqrt{16\frac{\alpha_2^2}{p_1(1-p_1)}t^*\log(1/4\delta)}\right)  \geq\left(4\delta \right)^{\sqrt{16c_1c^2}}.
		\end{align*}
		By taking $c_1$ small enough, e.g.,  $c_1=10^{-5} c^{-2}$, we have 
		${\cL_{k,l_1,l_2}(W)}/{\cL_{1}(W)} \geq 4\delta$. Note that by \eqref{equ:construct_G_alt},  the probability measure of the whole sample sequence  under the two hypotheses $\cG_1$ and $\cG_{k,l_1, l_2}$ only differ at $(k,l_1,l_2)$.  
		By a change of measure, we deduce that 
		\begin{equation}\label{eq:p2}
		\mathbb{P}_{k,l_1, l_2}(B_{k,l_1,l_2})\geq \mathbb{P}_{k,l_1, l_2}(\cE_{k,l_1, l_2})=\mathbb{E}_{k,l_1, l_2}[\mathbf{1}_{\cE_{k,l_1, l_2}}]=\mathbb{E}_1\left[\frac{\cL_{k,l_1, l_2}(W)}{\cL_1(W)}\mathbf{1}_{\cE_{k,l_1, l_2}}\right]\geq 4\delta\cdot 1/4=\delta,
		\end{equation}
		which completes the proof. 
	\end{proof}

If $\mathfrak{A}$ is an $(\epsilon, \delta)$-correct reward-agnostic algorithm, then  under transition model hypothesis $\cG_{k,l_1, l_2}$, when fed with $r_{k,l_1, l_2}$, it produces an $\epsilon$-NE policy pair $(\mu, \nu)$ with probability at least $1-\delta$.   At state $x_k$, this $\epsilon$-NE policy should generate the joint-action $(a_{l_1},b_{l_2})$ with a high  probability. To see this, note that now $(a_{l_1},b_{l_2})$ is the unique NE strategy at state $x_k$, which is a pure strategy. By Lemma \ref{lemma:eps_NE_inter}, $(\mu(\cdot\given x_k),\mathbbm{1}_{b=b_{l_2}})$ is an $2\epsilon$-NE strategy at $x_k$.  Let $\zeta\in[0,1]$ denote the probability of choosing $a_{l_1}$, i.e., $\zeta=\mu(a_{l_1}\given x_k)$. Then, the value at $x_k$ under $(\mu(\cdot\given x_k),\mathbbm{1}_{b=b_{l_2}})$, denoted by $V_{\mu,b_{l_2}}(x_k)$, is  
\$
V_{\mu,b_{l_2}}(x_k)=\zeta\cdot \frac{\gamma}{1-\gamma (p_0-\alpha_2)}+(1-\zeta)\cdot \frac{\gamma}{1-\gamma (p_0-2\alpha_2)}\leq \frac{\gamma}{1-\gamma (p_0-\alpha_2)},
\$
which, by the $2\epsilon$-NE property, should satisfy 
\$
V_{\mu,b_{l_2}}(x_k)\geq \frac{\gamma}{1-\gamma (p_0-\alpha_2)}-2\epsilon \Longrightarrow \zeta\geq 1-\frac{2\epsilon}{\frac{\gamma}{1-\gamma (p_0-\alpha_2)}-\frac{\gamma}{1-\gamma (p_0-2\alpha_2)}}\geq 1-\frac{2\epsilon}{20\epsilon}=\frac{9}{10}.
\$
Similarly, let $\xi=\nu(b_{l_2}\given x_k)$,  we have
\$
V_{a_{l_1},\nu}(x_k)&\geq \xi\cdot \frac{\gamma}{1-\gamma (p_0-\alpha_2)}+(1-\xi)\cdot \Big[\frac{\gamma}{1-\gamma (p_0-\alpha_2)}+20\epsilon\Big]\\
&=\frac{\gamma}{1-\gamma (p_0-\alpha_2)}+20(1-\xi)\epsilon,
\$
where the inequality is due to \eqref{equ:choose_alpha_2}. 
As $(\mathbbm{1}_{a=a_{l_1}},\nu(\cdot\given x_k))$ is an $2\epsilon$-NE at $x_k$, we have $V_{a_{l_1},\nu}(x_k)\leq {\gamma}/{[1-\gamma (p_0-\alpha_2)]}+2\epsilon$, leading to $\xi\geq 9/10$.  Thus, for the $\epsilon$-NE $(\mu,\nu)$, the probability of generating $(a_{l_1},b_{l_2})$ is at least $\zeta\cdot\xi\geq 81/100$.  
 
Hence, recalling the definition in \eqref{equ:B_event_def} and the fact that $\beta\leq 1-19/96<81/100$, we have 
$\mathbb{P}_{k,l_1, l_2}\big(B_{k,l_1, l_2}\big)<\delta$ for all $k\in[K]$, $l_1\in[L_1]$, and $l_2\in[L_2]\backslash\{1\}$. From Lemma~\ref{lemma:newversion},  this  does not happen unless $\mathbb{E}_1[T_{k,l_1,l_2}]>t^*$ for all $k\in[K]$, $l_1\in[L_1]$, and $l_2\in[L_2]\backslash\{1\}$.   
By linearity of expectation,  
the expected number of samples required by $\mathfrak{A}$ under hypothesis $\cG_1$ is at least $K[L_1(L_2-1)]t^*= \Omega\left(\frac{N}{(1-\gamma)^3\epsilon^2}\log(1/\delta)\right)$, which proves 
the lower bound for finding the $\epsilon$-NE policy. 

On the lower bound for finding $\epsilon$-approximate NE value, the hard cases  above  can also be used. In fact, suppose some algorithm $\mathfrak{A}$ returns some $\hat{Q}$ such that $\|\hat{Q}-Q^*\|_{\infty}\leq \epsilon/4$ with probability at least $1-\delta$, then it can identify the pure NE strategy as described in the paragraph  before \eqref{equ:B_event_def} for the  Q-values given in \eqref{equ:nomimal_assign_1}-\eqref{equ:nomimal_assign_2} (when reward $r_m$ is assigned to transition model of $\cG_m$ correspondingly), under our choices of the parameters.  This can be done by solving for the NE of the corresponding $\hat{Q}$. Moreover, when $r_{k,l_1,l_2}$ is assigned to $\cG_1$ (instead of $r_1$ as in \eqref{equ:nomimal_assign_1}), this procedure of solving the NE policy for $\hat{Q}$ will also output some policy that makes $B_{k,l_1, l_2}$ in \eqref{equ:B_event_def} hold with $\PP_1(B_{k,l_1, l_2})\geq 1-\delta$, following similar arguments around \eqref{equ:V_G_1_x_k}. Indeed, otherwise, if this procedure  outputs $\mu(x_k) = a_{l_1}, \nu(x_k)=b_{l_2}$ with a high probability, then the NE value under payoff matrix $\hat{Q}(x_k,\cdot,\cdot)$ will be at least $\epsilon/2$ away from the NE value under payoff matrix ${Q}^*(x_k,\cdot,\cdot)$, due to our choice of $\alpha_1$. However, as one-step of the $\max\min$ operation onto $\hat{Q}(x_k,\cdot,\cdot)$ (and $Q^*(x_k,\cdot,\cdot)$) is non-expansive, the NE values under these two payoff matrices should differ  no greater than $\epsilon/4$, as $\|\hat{Q}-Q^*\|_{\infty}\leq \epsilon/4$.  This shows $\PP_1(B_{k,l_1, l_2})\geq 1-\delta$. 
Then, using almost identical arguments as above, we    
obtain a lower bound of the same order, and thus prove  Theorem \ref{thm:lb_reward_agnostic}. 

\section{Auxiliary Results}

\subsection{A Smooth \textbf{Planning Oracle}}\label{append:planning}

We now show that solving the regularized matrix game induced by $\hat{Q}^*$, see \eqref{equ:solve_reg_minimax}, leads to a smooth \textbf{Planning Oracle} \kzedit{with certain smoothness coefficient $C$ (see Definition \ref{def:smooth_oracle})}. 


{\color{black}
\begin{lemma}\label{lemma:smooth_reg}
	Suppose that the nonnegative  regularizers $\Omega_i$ for $i=1,2$ in \eqref{equ:solve_reg_minimax} are 
	 twice continuously differentiable,  strongly convex, and bounded over the simplex. Suppose that for each $s\in\cS$,  the solution policy pair $(\hat{\mu}(\cdot\given s),\hat{\nu}(\cdot\given s))$ of  \eqref{equ:solve_reg_minimax} with $\tau_1=\tau_2=(1-\gamma)^2\epsilon>0$ lies in the relative interior of the simplexes $\Delta(\cA)$ and $\Delta(\cB)$, respectively.    
	 Then,  $(\hat{\mu},\hat{\nu})$ is smooth with respect to $\hat{Q}^*$, namely, this \textbf{Planning Oracle} follows Definition \ref{def:smooth_oracle}, with some constant $C={\texttt{poly}}(|\cA|,|\cB|,|\cS|,1/\epsilon,1/(1-\gamma))$, and meanwhile $\|\hat{V}^{\hat{\mu},*}-\hat{V}^*\|_\infty\leq \cO((1-\gamma)\epsilon),~\|\hat{V}^{*,\hat{\nu}}-\hat{V}^*\|_\infty\leq \cO((1-\gamma)\epsilon)$, namely, $\epsilon_{opt}$ in Theorem \ref{thm:main_res_2} satisfies $\epsilon_{opt}\leq \cO((1-\gamma)\epsilon)$.     
\end{lemma}  
\begin{proof}
Let $Q_s:=\hat{Q}^*(s,\cdot,\cdot)\in\RR^{|\cA|\times|\cB|}$ denote the  payoff matrix of the game at state $s$. Note that $Q_s\in [0,(1-\gamma)^{-1}]^{|\cA|\times|\cB|}$, $u\in [0,1]^{|\cA|}$ and $\vartheta\in [0,1]^{|\cB|}$. Also, note that by the simplex constraints on $u,\vartheta$, there are $|\cA|-1$ and $|\cB|-1$ free variables, and the last dimension can be represented as $1-\sum_{i=1}^{|\cA|-1}u(a_i)$, where we use $a_i$ to denote the $i$-th element in $\cA=\{a_1,a_2,\cdots,a_{|\cA|}\}$. Thus, we introduce new vectors $\tilde u=[u(a_1),u(a_2),\cdots,u(a_{|\cA|-1})]^\top$ and $\tilde \vartheta=[\vartheta(a_1),\vartheta(a_2),\cdots,\vartheta(a_{|\cA|-1})]^\top$ of dimensions $\RR^{|\cA|-1}$ and $\RR^{|\cB|-1}$, respectively. As the solution to  \eqref{equ:solve_reg_minimax} 
	lies in the relative interior of the simplex, we know that $1-\sum_{i=1}^{|\cA|-1}u(a_i)>0$ and $1-\sum_{i=1}^{|\cB|-1}\vartheta(b_i)>0$. We can then redefine the objective in \eqref{equ:solve_reg_minimax} as
	\#\label{equ:minimax_reform}
	f(\tilde u,\tilde \vartheta):=\Lambda_1(\tilde u)^\top Q_s \Lambda_2(\tilde \vartheta) -\tau_1\Omega_1(\Lambda_1(\tilde u))+\tau_2\Omega_2(\Lambda_2(\tilde \vartheta))
	\#
	where $u=\Lambda_1(\tilde u)=\left[\begin{matrix}
		I \\
 		-\bm{1}^\top   
	\end{matrix}\right]\tilde u+\be_{|\cA|}$
	and $ \vartheta=\Lambda_2(\tilde \vartheta)=\left[\begin{matrix}
		I \\
 		-\bm{1}^\top   
	\end{matrix}\right]\tilde\vartheta+\be_{|\cB|}$,  
 $\bm{1}$ denotes the all-one vector of proper dimension, and $\be_{i}$ denotes the vector of proper dimension whose $i$-th element is one and all other elements are zero. 
 
 Since the  solution  lies in the relative interior of $\Delta(\cA)$ and $\Delta(\cB)$, by first-order optimality,  
	we have that for each $s\in\cS$
\#
	\nabla_{\tilde u}f(\tilde u,\tilde \vartheta)&=-\tau_1\nabla_{\tilde u} \Omega_1(\Lambda_1(\tilde u)) +\left[\begin{matrix}
		I & -\bm{1}
	\end{matrix}\right]Q_s \Lambda_2(\tilde\vartheta) =0,\label{equ:1st_order_opt}\\
	\nabla_{\tilde \vartheta}f(\tilde u,\tilde \vartheta)&=\tau_2\nabla_{\tilde\vartheta} \Omega_2(\Lambda_2(\tilde\vartheta)) +\left[\begin{matrix}
		I & -\bm{1}
	\end{matrix}\right]Q_s^\top \Lambda_1(\tilde u)=0,\label{equ:1st_order_opt_2}
	\#
	whose solution is unique since \eqref{equ:minimax_reform} is still a strongly-convex-strongly-concave minimax problem. In particular, note that by the chain rule, the Hessians of $f$ are  $\nabla_{\tilde u}^2 f(\tilde u,\tilde \vartheta)=\left[\begin{matrix}
		I & -\bm{1}
	\end{matrix}\right]\nabla_{u}^2 g(u,\vartheta)\left[\begin{matrix}
		I \\ -\bm{1}^\top 
	\end{matrix}\right]$ and $\nabla_{\tilde \vartheta}^2 f(\tilde u,\tilde \vartheta)=\left[\begin{matrix}
		I & -\bm{1}
	\end{matrix}\right]\nabla_{\vartheta}^2 g(u,\vartheta)\left[\begin{matrix}
		I \\ -\bm{1}^\top 
	\end{matrix}\right]$, where 
	\$
	g(u,\vartheta):=u^\top Q_s  \vartheta-\tau_1\Omega_{1}(u)+\tau_2\Omega_{2}(\vartheta)
	\$
	is the original objective used in \eqref{equ:solve_reg_minimax}. Let $\eta_i>0$ be the strong-convexity coefficient for $\Omega_i$, then, we have $\nabla_{\tilde u}^2 f(\tilde u,\tilde \vartheta)\preceq -\tau_1\eta_1 \cdot I $ and $\nabla_{\tilde \vartheta}^2 f(\tilde u,\tilde \vartheta)\succeq \tau_2\eta_2  \cdot I$, since for any  vector $x\in\RR^{|\cA|-1}$ (or $x\in\RR^{|\cB|-1}$) that is not $\bm{0}$, $\left[\begin{matrix}
		I \\ -\bm{1}^\top 
	\end{matrix}\right]x$ is not $\bm{0}$.

	Define a function $F:\RR^{|\cA|-1}\times \RR^{|\cB|-1}\times \RR^{|\cA||\cB|}\to \RR^{|\cA|+|\cB|-2}$ as follows,  such that \eqref{equ:1st_order_opt}-\eqref{equ:1st_order_opt_2} is equivalent to
	\$
	F\big(\tilde u,\tilde \vartheta,\vect(Q_s)\big):=\left[\begin{matrix}
		\tau_1\nabla_{\tilde u} \Omega_1(\Lambda_1(\tilde u)) -\left[\begin{matrix}
		I & -\bm{1}
	\end{matrix}\right]Q_s \Lambda_2(\tilde\vartheta) \\ \tau_2\nabla_{\tilde\vartheta} \Omega_2(\Lambda_2(\tilde\vartheta)) +\left[\begin{matrix}
		I & -\bm{1}
	\end{matrix}\right]Q_s^\top \Lambda_1(\tilde u)
	\end{matrix}\right]=0. 
	\$
	As the solution to \eqref{equ:1st_order_opt}-\eqref{equ:1st_order_opt_2} lies in the relative interior of the simplexes, for any choice of $Q_s\in\RR^{|\cA|\times|\cB|}$ (not just $[0,(1-\gamma)^{-1}]^{|\cA|\times |\cB|}$), the domain of $F$ can be specified  as $\Delta^o(\cA)\times \Delta^o(\cB)\times \cQ^o$, where $\Delta^o(\cA)$ and $\Delta^o(\cB)$ denote the sets of $\tilde u$ and $\tilde \vartheta$ whose corresponding  $u$ and $\vartheta$ lie in the 	interiors of $\Delta(\cA)$ and $\Delta(\cB)$, respectively, and $\cQ^o\subset\RR^{|\cA||\cB|}$ denotes some open set that contains $[0,(1-\gamma)^{-1}]^{|\cA||\cB|}$. 
	
	Notice that the Jacobian of $F$ with respect to $[\tilde u^\top~~\tilde\vartheta^\top]^\top$ is
	\small
	\#\label{equ:jacobian_F}
	M\big(\tilde u,\tilde \vartheta,\vect(Q_s)\big):=\left[\begin{matrix}
	\frac{\partial F}{\partial \tilde u} &
	\frac{\partial F}{\partial \tilde\vartheta}
	\end{matrix}\right]=\left[\begin{matrix}
	\tau_1\left[\begin{matrix}
		I & -\bm{1}
	\end{matrix}\right]\nabla^2_u \Omega_1(\Lambda_1(\tilde u))\left[\begin{matrix}
		I \\ -\bm{1}^\top 
	\end{matrix}\right] & -\left[\begin{matrix}
		I & -\bm{1}
	\end{matrix}\right]Q_s\left[\begin{matrix}
		I \\ -\bm{1}^\top 
	\end{matrix}\right] \\
	\left[\begin{matrix}
		I & -\bm{1}
	\end{matrix}\right]Q_s^\top \left[\begin{matrix}
		I \\ -\bm{1}^\top 
	\end{matrix}\right]& \tau_2\left[\begin{matrix}
		I & -\bm{1}
	\end{matrix}\right]\nabla^2_\vartheta \Omega_2(\Lambda_2(\tilde\vartheta))\left[\begin{matrix}
		I \\ -\bm{1}^\top 
	\end{matrix}\right]
	\end{matrix}\right],
	\#
	\normalsize
	which is always invertible for any point in $\Delta^o(\cA)\times \Delta^o(\cB)\times \cQ^o$. This is because $\Omega_i$ are strongly convex, and thus the real parts of the eigenvalues of the matrix, which are  the eigenvalues of $(M+M^\top)/2$, are always positive and uniformly lower bounded. Specifically, we have 
	\$ 
	&\min_{i} \lambda_i\big(M(\tilde u,\tilde\vartheta,\vect(Q_s))+M^\top(\tilde u,\tilde\vartheta,\vect(Q_s))\big)\\
	&\qquad\geq 2\min\{\tau_1\eta_1,~\tau_2\eta_2\}=2(1-\gamma)^2\min\{\eta_1,~\eta_2\}\cdot\epsilon >0, 
	\$ 
	with $\lambda_i(\cdot)$ being the $i$-th largest eigenvalues of the corresponding matrix. This further implies that for any $(\tilde u,\tilde \vartheta,\vect(Q_s))\in \Delta^o(\cA)\times \Delta^o(\cB)\times \cQ^o$, 
	\#\label{equ:ub_M_inv}
	&\big\|M(\tilde u,\tilde \vartheta,\vect(Q_s))^{-1}\big\|_2=\frac{1}{\min_{i}~\sigma_i(M(\tilde u,\tilde \vartheta,\vect(Q_s)))}\\
	&\quad\leq \frac{2}{\min_{i}~\lambda_i(M(\tilde u,\tilde \vartheta,\vect(Q_s))+M^\top(\tilde u,\tilde \vartheta,\vect(Q_s)))}\leq \frac{1}{\min\{\eta_1,~\eta_2\}(1-\gamma)^2\cdot\epsilon},  
	\#
	where $\sigma_i(\cdot)$ is the $i$-th largest  singular value of the corresponding matrix. 
	
	By the implicit function theorem \citep{krantz2012implicit}, for any point that solves $F(\tilde u,\tilde \vartheta,\vect(Q_s))=0$, since $M(\tilde u,\tilde \vartheta,\vect(Q_s))$ is invertible, there exists a neighborhood $U\subseteq \Delta^o(\cA)$, $V\subseteq \Delta^o(\cB)$, and $W\subseteq \cQ^o$ around it, such that $[\tilde u^\top  ~~\tilde \vartheta^\top]^\top\in U\times V$ is a unique function of $\vect(Q_s)$ for all $\vect(Q_s)\in W$, and 
	\small
	\$
	\frac{\partial [\tilde u^\top~\tilde \vartheta^\top]^\top}{\partial \vect(Q_s)}=-\left[\begin{matrix}
	\frac{\partial F}{\partial \tilde u} &
	\frac{\partial F}{\partial \tilde\vartheta}
	\end{matrix}\right]^{-1}\cdot \frac{\partial F}{\partial \vect(Q_s)}=-M(\tilde u,\tilde \vartheta,\vect(Q_s))^{-1}\cdot \left[\begin{matrix}
		-\Lambda_2(\tilde\vartheta)^\top \otimes \left[\begin{matrix}
 	I & -\bm{1}
 \end{matrix}\right]\\
\left[\begin{matrix}
 	I & -\bm{1}
 \end{matrix}\right] \otimes \Lambda_1(\tilde u)^\top  
		\end{matrix}\right],
	\$
	\normalsize
	where $\otimes$ denotes the Kronecker product.  
	Thus, we have
	\$
	\bigg\|\frac{\partial [\tilde u^\top~\tilde \vartheta^\top]^\top}{\partial \vect(Q_s)}\bigg\|_2&\leq \big\|M(\tilde u,\tilde \vartheta,\vect(Q_s))^{-1}\big\|_2\cdot \bigg\|\frac{\partial F}{\partial \vect(Q_s)}\bigg\|_2\\
	&\leq \frac{\sqrt{(|\cA|+|\cB|-2)|\cA||\cB|}}{\min\{\eta_1,~\eta_2\}(1-\gamma)^2\cdot\epsilon}\cdot \bigg\|\frac{\partial F}{\partial \vect(Q_s)}\bigg\|_{\infty}\leq  \frac{2{(|\cA|\sqrt{|\cB|}+|\cB|\sqrt{|\cA|})}}{\min\{\eta_1,~\eta_2\}(1-\gamma)^2\cdot\epsilon},
	\$
	where we have used \eqref{equ:ub_M_inv}, the fact that for matrix $A\in\RR^{m\times n}$, $\|A\|_2\leq \sqrt{mn} \|A\|_\infty$, and the fact that 
	\$
	\bigg\|\frac{\partial F}{\partial \vect(Q_s)}\bigg\|_\infty=2 \cdot \max\left\{\|\Lambda_2(\tilde\vartheta)\|_1,~\|\Lambda_1(\tilde u)\|_1\right\}=2.
	\$ 
	Notice that this is a uniform bound on the gradient of the implicit function, at any point in $\Delta^o(\cA)\times \Delta^o(\cB)\times \cQ^o$, which together with the  mean-value theorem leads to 
	\$
	\big\|[\tilde u^\top_1~\tilde \vartheta^\top_1]-[\tilde u^\top_2~\tilde \vartheta^\top_2]\big\|_2\leq \frac{2{(|\cA|\sqrt{|\cB|}+|\cB|\sqrt{|\cA|})}}{\min\{\eta_1,~\eta_2\}(1-\gamma)^2\cdot\epsilon} \cdot\big\|\vect(Q_{s,1})-\vect(Q_{s,2})\big\|_2,
	\$
	where the pair $(\tilde u_i,\tilde\vartheta_i)$ is the unique solution of $F=0$ corresponding to $Q_{s,i}$. 
	By the equivalence of norms and considering all $s\in\cS$, we can find some constant $C$ (which  depends on $|\cA|$, $|\cB|$, $|\cS|$, as well as $1/\epsilon$  and $1/(1-\gamma)$ polynomially) as the smooth coefficient, and this completes the first argument of the result. 
	
	Now, it suffices to prove that the obtained solution $(\hat{\mu},\hat{\nu})$ with parameter $\tau_1=\tau_2=(1-\gamma)^2\epsilon$ also leads to small $\epsilon_{opt}$. Let $D_i>0$ denotes the upper bound of the regularizer $\Omega_i$ over the simplex. Then, we have that for any $s\in\cS$ 
\#
	0&\leq \hat{V}^*(s)-\hat{V}^{\hat{\mu},*}(s)=\min_{\vartheta\in\Delta(\cB)}\EE_{a\sim{\hat{\mu}}^*(\cdot\given s),b\sim \vartheta}\big[\hat{Q}^*(s,a,b)\big]-\min_{\vartheta\in\Delta(\cB)}\EE_{a\sim\hat{\mu}(\cdot\given s),b\sim \vartheta}\big[\hat{Q}^{\hat{\mu},*}(s,a,b)\big]\notag\\
\quad &=\min_{\vartheta\in\Delta(\cB)}\EE_{a\sim{\hat{\mu}}^*(\cdot\given s),b\sim \vartheta}\big[\hat{Q}^*(s,a,b)\big]-\min_{\vartheta\in\Delta(\cB)}\EE_{a\sim\hat{\mu}(\cdot\given s),b\sim \vartheta}\big[\hat{Q}^*(s,a,b)\big]\notag\\
&\quad\qquad +\min_{\vartheta\in\Delta(\cB)}\EE_{a\sim\hat{\mu}(\cdot\given s),b\sim \vartheta}\big[\hat{Q}^*(s,a,b)\big]-\min_{\vartheta\in\Delta(\cB)}\EE_{a\sim\hat{\mu}(\cdot\given s),b\sim \vartheta}\big[\hat{Q}^{\hat{\mu},*}(s,a,b)\big]\notag\\
\quad &\leq \min_{\vartheta\in\Delta(\cB)}\EE_{a\sim{\hat{\mu}}^*(\cdot\given s),b\sim \vartheta}\big[\hat{Q}^*(s,a,b)\big]-\min_{\vartheta\in\Delta(\cB)}\EE_{a\sim\hat{\mu}(\cdot\given s),b\sim \vartheta}\big[\hat{Q}^*(s,a,b)\big]+\gamma \|\hat{V}^*-\hat{V}^{\hat{\mu},*}\|_\infty \label{equ:small_eps_opt_1}\\
\quad &\leq  \min_{\vartheta\in\Delta(\cB)}\EE_{a\sim{\hat{\mu}}^*(\cdot\given s),b\sim \vartheta}\big[\hat{Q}^*(s,a,b)\big]-\Bigg(\min_{\vartheta\in\Delta(\cB)}\EE_{a\sim\hat{\mu}(\cdot\given s),b\sim \vartheta}\big[\hat{Q}^*(s,a,b)\big]-\tau_1\Omega_1(\hat{\mu}(\cdot\given s))\notag\\
&\qquad\qquad+\tau_2\Omega_2(\vartheta)\Bigg)
+\tau_2 D_2
+\gamma \|\hat{V}^*-\hat{V}^{\hat{\mu},*}\|_\infty\label{equ:small_eps_opt_2}\\
\quad &\leq \min_{\vartheta\in\Delta(\cB)}\EE_{a\sim{\hat{\mu}}^*(\cdot\given s),b\sim \vartheta}\big[\hat{Q}^*(s,a,b)\big]-\Bigg(\min_{\vartheta\in\Delta(\cB)}\EE_{a\sim\hat{\mu}^*(\cdot\given s),b\sim \vartheta}\big[\hat{Q}^*(s,a,b)\big]-\tau_1\Omega_1(\hat{\mu}^*(\cdot\given s))\notag\\
&\qquad\qquad+\tau_2\Omega_2(\vartheta)\Bigg)
+\tau_2 D_2
+\gamma \|\hat{V}^*-\hat{V}^{\hat{\mu},*}\|_\infty\label{equ:small_eps_opt_3}\\
\quad &\leq \EE_{a\sim{\hat{\mu}}^*(\cdot\given s),b\sim \tilde\vartheta}\big[\hat{Q}^*(s,a,b)\big]-\Bigg(\min_{\vartheta\in\Delta(\cB)}\EE_{a\sim\hat{\mu}^*(\cdot\given s),b\sim \vartheta}\big[\hat{Q}^*(s,a,b)\big]+\tau_2\Omega_2(\vartheta)\Bigg)\notag\\
&\qquad\qquad
+\tau_1D_1+\tau_2 D_2
+\gamma \|\hat{V}^*-\hat{V}^{\hat{\mu},*}\|_\infty\label{equ:small_eps_opt_4}\\
\quad &\leq \tau_1D_1+\tau_2 D_2
+\gamma \|\hat{V}^*-\hat{V}^{\hat{\mu},*}\|_\infty \label{equ:small_eps_opt_5} 
\#
where $(\hat{\mu}^*,\hat{\nu}^*)$ denotes a Nash equilibrium policy in the empirical model $\hat{G}$, with $\hat{V}^*=\hat{V}^{\hat{\mu}^*,\hat{\nu}^*}$, \eqref{equ:small_eps_opt_1} uses Bellman equation to relate $Q$-function and $V$-function, \eqref{equ:small_eps_opt_2} uses the boundedness of $\Omega_2$, \eqref{equ:small_eps_opt_3} uses \eqref{equ:solve_reg_minimax}, and in \eqref{equ:small_eps_opt_4} $\tilde\vartheta$ denotes the $\argmin$ of $\vartheta\in\Delta(\cB)$ in the second term in \eqref{equ:small_eps_opt_3}. By the choices of $\tau_1=\tau_2=(1-\gamma)^2\epsilon$ and \eqref{equ:small_eps_opt_5}, we have that $\|\hat{V}^*-\hat{V}^{\hat{\mu},*}\|_\infty\leq \cO(\max\{\tau_1,\tau_2\}/(1-\gamma))=\cO((1-\gamma)\epsilon)$. The proof for $\|\hat{V}^{*,\hat{\nu}}-\hat{V}^*\|_\infty\leq \cO((1-\gamma)\epsilon)$ is symmetric and analogous. This completes the proof.
\end{proof} 
}

To ensure that the solution $(\hat{\mu}(\cdot\given s),\hat{\nu}(\cdot\given s))$ of  \eqref{equ:solve_reg_minimax} lies  in the relative interior of the simplexes, the  common choice of \emph{steep} regularizers will suffice \citep{mertikopoulos2016learning}. The steep regularizer means that for any  $u$ (resp. $\vartheta$) on the boundary of the simplex $\Delta(\cA)$ (resp. $\Delta(\cB)$), and for every interior sequence $u_n\to u$ (resp. $\vartheta_n\to \vartheta$) that approaches it, it holds that 	
	$\big\|\frac{d\Omega_1(u)}{du}\big|_{u=u_n}\big\|_2\to\infty$ (resp. $\big\|\frac{d\Omega_2(\vartheta)}{d\vartheta}\big|_{\vartheta=\vartheta_n}\big\|_2\to\infty$). This way, the optimizer is not on the boundary of the simplexes. 
Examples of steep regularizers in Lemma \ref{lemma:smooth_reg} include the commonly used negative entropy, Tsallis entropy and R\'enyi entropy with certain parameters; see \cite{mertikopoulos2016learning} for more discussions.  \kzedit{Also note that they are bounded over simplex for standard choices of the parameters, and thus satisfy the conditions in our Lemma \ref{lemma:smooth_reg}.}

\subsection{Properties of ($\epsilon$-)NE in Zero-Sum Matrix Games}

Now we establish several properties of the ($\epsilon$-)NE strategies in zero-sum matrix games, which have been used in the proof in \S\ref{sec:append_reward_agnostic_proof}. 

\begin{lemma}[NE Value Range]\label{lemma:NE_value_range}
	Consider a two-player zero-sum matrix game $\cM$ with the action spaces $\cA$ and $\cB$, and the payoff matrix $M\in\RR^{|\cA|\times |\cB|}$ with $=\{m_{ij}\}_{i\in[|\cA|],j\in[|\cB|]}$ for the maximizer (agent-$1$). Then, the NE value of the game $V^*$ is bounded between the \emph{maximin}  and \emph{minimax} elements  in $M$, i.e., 
	\$
	\max_i\min_j~m_{ij}\leq V^*\leq \min_j\max_i~m_{ij}.
	\$
\end{lemma}
\begin{proof}
	Note that 
	\$
	\max_{u\in\Delta(\cA)}\min_{j}~u^\top M \bm{e}_j=\max_{u\in\Delta(\cA)}\min_{\vartheta\in\Delta(\cB)}~u^\top M \vartheta=V^*=\min_{\vartheta\in\Delta(\cB)}\max_{u\in\Delta(\cA)}~u^\top M \vartheta=\min_{\vartheta\in\Delta(\cB)}\max_{i}~\bm{e}_i^\top M \vartheta,
	\$
	where $\bm{e}_i$ denote the all-zero vector except a single $1$ at element $i$, with proper dimensions. Also, notice that
	\$
	\min_{\vartheta\in\Delta(\cB)}\max_{i}~\bm{e}_i^\top M \vartheta\leq \min_{j}\max_{i}~\bm{e}_i^\top M \bm{e}_j=\min_j\max_i~m_{ij},
	\$
	where the inequality is due to that $\bm{e}_i\in\Delta(\cB)$ and the $\min$ on the right is taken over a smaller set, thus has a larger value. This proves the right-hand side of the inequality. Proof for the other side is analogous. 
\end{proof}

\begin{lemma}[$\epsilon$-NE Strategy Interchangeability]\label{lemma:eps_NE_inter}
	Consider the game as above in Lemma \ref{lemma:NE_value_range}. Let $u_1,u_2\in\Delta(\cA)$ and $\vartheta_1,\vartheta_2\in\Delta(\cB)$ be strategies such that $(u_1,\vartheta_1)$ is a Nash equilibrium strategy, and $(u_2,\vartheta_2)$ is an $\epsilon$-NE strategy. Then, both $(u_1,\vartheta_2)$ and $(u_2,\vartheta_1)$ are $2\epsilon$-NE strategy pairs. 
\end{lemma}
\begin{proof}
Let $V(u,\vartheta):=u^\top M \vartheta$ denote the value under any strategy pair  $(u,\vartheta)$.   
	By definition, we have that for any $u\in\Delta(\cA)$ and $\vartheta\in\Delta(\cB)$
	\$
	V(u,\vartheta_1)\leq V(u_1,\vartheta_1)\leq V(u_1,\vartheta),\qquad\qquad V(u,\vartheta_2)-\epsilon \leq V(u_2,\vartheta_2)\leq V(u_2,\vartheta)+\epsilon.
	\$
	Then, we have 
	\$
	V(u_1,\vartheta_1)\geq V(u_2,\vartheta_1) \geq V(u_2,\vartheta_2)-\epsilon,\qquad\qquad V(u_1,\vartheta_1)\leq V(u_1,\vartheta_2)\leq V(u_2,\vartheta_2)+\epsilon.
	\$
	Combining the two, we have
	\$
	V(u,\vartheta_2)-2\epsilon\leq V(u_2,\vartheta_2)-\epsilon\leq V(u_1,\vartheta_1)&\leq V(u_1,\vartheta_2)\\
	&\leq V(u_2,\vartheta_2)+\epsilon\leq V(u_1,\vartheta_1)+2\epsilon\leq V(u_1,\vartheta)+2\epsilon
	\$
	for any $u\in\Delta(\cA)$ and $\vartheta\in\Delta(\cB)$, showing that $(u_1,\vartheta_2)$ is an $2\epsilon$-NE. The proof for the pair $(u_2,\vartheta_1)$ is analogous. 
\end{proof}

\end{document}